\documentclass{article}
\PassOptionsToPackage{numbers}{natbib}
\usepackage[final]{nips_2017} % produce camera-ready copy
\usepackage{nicefrac}
\usepackage{amssymb, amsmath, amsthm, graphicx, bbm, appendix, booktabs, array}
\usepackage{microtype}
\usepackage{thmtools,thm-restate}
\usepackage[inline]{enumitem}
\usepackage[colorlinks,bookmarks=true]{hyperref}
\usepackage[capitalize,nameinlink]{cleveref}
\crefname{thm}{\text{Thm}}{\text{Thms}}
\crefname{defn}{\text{Defn}}{\text{Defns}}
\crefname{assm}{\text{Axiom}}{\text{Axioms}}
\crefname{myquote}{\text{Hypothesis}}{\text{Hypotheses}}
\crefformat{footnote}{#2\footnotemark[#1]#3}
\usepackage{mymacros} % This needs to go after cleveref so that theorem labeling works
\theoremstyle{theorem}

\newtheorem{assm}[thm]{Axiom}
 
\usepackage[normalem]{ulem}
\usepackage{subcaption}
\usepackage[font=small,labelfont=bf]{caption}

\renewcommand{\poly}{\mathsf{poly}}
\newcommand{\nxt}{\overline}
\newcommand{\prv}{\underline}
\newcommand{\Gaus}{\mathcal{N}}
\newcommand{\Vt}{\mathrm{V}}
\newcommand{\Wt}{\mathrm{W}}

\newcommand{\sech}{\operatorname{sech}}

\newcommand{\eps}{\epsilon}
\newcommand{\ee}{\mathbf{e}}
\newcommand{\JJ}{\mathbb{J}}
\newcommand{\cV}{\mathsf{c}}
\newcommand{\oalpha}{{\f 1 {1-\alpha}}}
\newcommand{\aalpha}{{\f \alpha {1-\alpha}}}
\newcommand{\bessel}{{\mathcal{K}_0}}

\usepackage{letltxmacro}
\LetLtxMacro{\oldchi}{\chi}
\renewcommand{\chi}{\boldsymbol{\oldchi}}
\LetLtxMacro{\oldgamma}{\gamma}
\renewcommand{\gamma}{\boldsymbol{\oldgamma}}
\LetLtxMacro{\oldlambda}{\lambda}
\renewcommand{\lambda}{\boldsymbol{\oldlambda}}
\LetLtxMacro{\olddaleth}{\daleth}

\renewcommand{\daleth}{\chi}
\newcommand{\qq}{\mathbf{q}}
\newcommand{\pp}{\mathbf{p}}
\renewcommand{\cc}{\mathbf{c}}
\renewcommand{\beth}{\beta}
\newcommand{\mfs}{\mathbf{s}}

\usepackage{endnotes}

\let\footnote=\endnote

\title{Mean Field Residual Networks: On the Edge of Chaos}
\author{
	Greg Yang\thanks{Work done while at Harvard University}\\
 	{\small Microsoft Research AI}\\
     \texttt{\small gregyang@microsoft.com}\\
     \And
     Samuel S. Schoenholz\\
     {\small Google Brain}\\
     \texttt{\small schsam@google.com}\\
}
\begin{document}

\maketitle
\begin{abstract}
We study randomly initialized residual networks using mean field theory and the theory of difference equations.
Classical feedforward neural networks, such as those with tanh activations, exhibit exponential behavior on the average when propagating inputs forward or gradients backward.
The exponential forward dynamics causes rapid collapsing of the input space geometry, while the exponential backward dynamics causes drastic vanishing or exploding gradients.
We show, in contrast, that by adding skip connections, the network will, depending on the nonlinearity, adopt subexponential forward and backward dynamics, and in many cases in fact polynomial.
The exponents of these polynomials are obtained through analytic methods and proved and verified empirically to be correct.
In terms of the ``edge of chaos'' hypothesis, these subexponential and polynomial laws allow residual networks to ``hover over the boundary between stability and chaos,'' thus preserving the geometry of the input space and the gradient information flow.
In our experiments, for each activation function we study here, we initialize residual networks with different hyperparameters and train them on MNIST.
Remarkably, our {\it initialization time} theory can accurately predict {\it test time} performance of these networks, by tracking either the expected amount of gradient explosion or the expected squared distance between the images of two input vectors.
Importantly, we show, theoretically as well as empirically, that common initializations such as the Xavier or the He schemes are not optimal for residual networks, because {\it the optimal initialization variances depend on the depth}.
Finally, we have made mathematical contributions by deriving several new identities for the kernels of powers of ReLU functions by relating them to the zeroth Bessel function of the second kind.
\end{abstract}

\section{Introduction}

Previous works \cite{poole_exponential_2016,daniely_toward_2016,schoenholz_deep_2017} have shown that randomly initialized neural networks exhibit a spectrum of behavior with depth, from stable to chaotic, which depends on the variance of the initializations: the cosine distance of two input vectors converges exponentially fast with depth to a fixed point in [0, 1]; if this fixed point is 1, then the behavior is stable; if this fixed point is 0, then the behavior is chaotic.
It has been argued in many prior works \cite{bertschinger_real-time_2004,poole_exponential_2016} that effective computation can only be supported by a dynamical behavior that is on {\bf the edge of chaos}.
Too much stability prevents the neural network from telling apart two different inputs.
While some chaotic behavior can increase the expressivity of a network, too much chaos makes the neural network think two similar inputs are very different.
At the same time, the same initialization variances also control how far gradient information can be propagated through the network; the networks with chaotic forward dynamics will tend to suffer from exploding gradients, while networks with stable forward dynamics will tend to suffer from vanishing gradients.

These works have focused on vanilla (fully connected) feedforward networks.
Here we consider residual networks \cite{he_deep_2016,he_identity_2016} (with fully-connected layers and without batchnorm), which are a family of recently proposed neural network architectures that has achieved state-of-the-art performance on image recognition tasks, beating all other approaches by a large margin.
The main innovation of this family of architectures is the addition of a passthrough (identity) connection from the previous layer to the next, such that the usual nonlinearity computes the ``residual'' between the next-layer activation and the previous-layer activation.

In this work, we seek to characterize randomly initialized residual networks.
One of our main results is that random residual networks for many nonlinearities such as $\tanh$ {\bf live on the edge of chaos}, in that the cosine distance of two input vectors will converge to a fixed point at a polynomial rate, rather than an exponential rate, as with vanilla tanh networks.
Thus a typical residual network will slowly cross the stable-chaotic boundary with depth, hovering around this boundary for many layers.
In addition, for most of the nonlinearities considered here, the mean field estimate of the gradient grows subexponentially with depth.
In fact, for $\alpha$-ReLU, the $\alpha$th-power of ReLU, for $\alpha < 1$, the gradient grows only polynomially.
These theoretical results provide some theoretical justification for why residual networks work so well in practice.
In our experiments, we are also able to predict surprisingly well the relative performances of {\it trained} residual networks based only on their initialization hyperparameters, in a variety of settings.
In particular, we find that the quality of initialization for tanh resnets is determined by {\it trainability} (how much gradient explosion on average) while that for ($\alpha$-)ReLU resnets is determined by expressivity (how far can two different input vectors be pulled apart) (see \cref{sec:experiments}).
To the best of our knowledge, this is the first time that a quantity other than gradient explosion/vanishing has been found to control the quality of initialization.
We establish theoretically and empirically that the best initialization variances for residual networks depend on the depth of the network (contrary to the feedforward case \cite{schoenholz_deep_2017}), so that \textbf{common initialization schemes like Xavier \cite{glorot_understanding_2010} or He \cite{he_delving_2015} cannot be optimal}.
In fact, even the rationale of He initialization is incorrect for ReLU residual networks because it tries to control gradient dynamics rather than expressivity.
However we want to emphasize that we study a simplified model of residual networks in this work, with no batchnorm or convolutional layers, so that these results are not necessarily indicative of the MSRA residual network used in practice \cite{he_deep_2016}.

In the body of this paper, we give account of general intuition and/or proof strategy when appropriate for our theoretical results, but we relegate all formal statements and proofs to the appendix.

\section{Background}
\label{sec:background}
% We review the mean field theory developed in \cite{poole_exponential_2016}.
Consider a vanilla feedforward neural network of $L$ layers, with each layer $l$ having $\p N l$ neurons; here layer 0 is the input layer.
For the ease of presentation we assume all hidden layer widths are the same $\p N l = N$ for all $l > 0$.
Let $\p x 0 = (\p x 0_1, \ldots, \p x 0_{\p N 0})$ be the input vector to the network, and let $\p x l$ for $l > 0$ be the activation of layer $l$.
Then a neural network is given by the equations
	\begin{align*}
%		\p x l &= \phi(\p h l), & \p h l &= \p w l \p x {l-1} + \p b l\\
		\p x l _i &= \phi(\p h l _i), & \p h l _i &= \sum_{j=1}^N \p w l _{ij} \p x {l-1}_j + \p b l_i
	\end{align*}
	where
	\begin{enumerate*}[label=(\roman*)]
	    \item $\p h l$ is the pre-activation at layer $l$,
	    \item $\p w l$ is the weight matrix,
	    \item $\p b l$ is the bias vector, and
	    \item $\phi$ is a nonlinearity, for example $\tanh$ or ReLU, which is applied coordinatewise to its input.
	\end{enumerate*}
	
To lighten up notation, we suppress the explicit layer numbers $l$ and write
	\begin{align*}
%		x &= \phi(h), & h &= w \prv x + b\\
		x_i &= \phi(h_i), & h_i &= \sum_j w_{ij} \prv x_j + b_i
	\end{align*}
where $\bullet$ implicitly denotes $\p \bullet l$, and $\prv \bullet$ denotes $\p \bullet {l-1}$ (and analogously, $\nxt \bullet$ denotes $\p \bullet {l+1}$).

A series of papers \cite{poole_exponential_2016,raghu_expressive_2016,schoenholz_deep_2017} investigated the ``average behavior'' of random neural networks sampled via $\p w l_{ij} \sim \Gaus(0, \sigma_w^2/N), \p b l_i \sim \Gaus(0, \sigma_b^2)$, for fixed parameters $\sigma_w$ and $\sigma_b$, independent of $l$.
Consider the expectation of $\f 1 N \sum_{i=1}^N x^2_i$, the normalized squared length of $x$, over the sampling of $w$ and $b$.
\citet{poole_exponential_2016} showed that this quantity converges to a fixed point exponentially fast for sigmoid nonlinearities.
Now suppose we propagate two different vectors $\p x 0$ and $(\p x 0)'$ through the network.
\citet{poole_exponential_2016} also showed that the expectation of the normalized dot product $\f 1 N \sum_{i=1}^N x_i x'_i$ converges exponentially fast to a fixed point.
The ratio between the normalized squared length and the normalized dot product is the cosine distance between $x$ and $x'$.
Thus these two exponential convergence results show that the cosine distance converges exponentially fast to a fixed point as well.
Intuitively, this means that a vanilla feedforward network ``forgets'' the geometry of the input space ``very quickly,'' after only a few layers.

In addition, \citet{schoenholz_deep_2017}, under certain independence assumptions, showed that the expected normalized squared norm of the gradient also vanishes or explodes in an exponential fashion with depth, with the ''half-life'' controlled by $\sigma_w$ and $\sigma_b$.
They verified that this theoretical ''half-life'' correlates in practice with the maximal number of layers that are admissible to good performance.

At the same time, \citet{daniely_toward_2016} published work of similar nature, but phrased in the language of reproducing kernel Hilbert spaces, and provided high probability estimates that are meaningful for the case when the width $N$ is finite and the depth is logarithmic in $N$.
However, they essentially fixed the variance parameters $\sigma_\bullet$, and furthermore, their framework (for example the notion of a ``skeleton'') does not immediately generalize to the residual network case.

In this work, we show that residual networks have very different dynamics from vanilla feedforward networks.
In most cases, the cosine distance convergence rate and the gradient growth rate are subexponential in a residual network, and in most cases, these rates may be polynomial.

\section{Preliminaries}

Residual networks were first introduced by \cite{he_deep_2016} and later refined by \cite{he_identity_2016}, and they are now commonplace among deployed neural systems.
The key innovation there is the addition of a shortcut connection from the previous layer to the next.
We define the following idealized architectures for ease of analysis.
Note that we only consider fully-connected affine layers instead of convolutional layers.
A {\bf reduced residual network (RRN)} has the recurrence
	\begin{align*}
%		x &= \phi(h) + \prv x, & h &= w \prv x + b\\
		x_i &= \phi(h_i) + \prv x, & h_i &= \sum_j w_{ij} \prv x_j + b_i.
	\end{align*}
A {\bf (full) residual network (FRN)} in addition has an affine connection given by weights $v$ and biases $a$ from the nonlinearity $\phi(h)$ to the next layer:
	\begin{align*}
%	x &= v \phi( h) + \prv x + a, &  h &= w \prv x + b\\
	x_i &= \sum_j v_{ij} \phi( h_j) + \prv x_i + a_i, &  h_i &= \sum_j w_{ij} \prv x_j + b_i
	\end{align*}

We are interested in the ``average behavior'' of these network when the weights and biases, $\p w l_{ij}, \p b l_i, \p v l_{ij}$, and $\p a l_i$ are sampled i.i.d. from Gaussian distributions resp. with standard deviations $\sigma_w, \sigma_b, \sigma_v,$ and $\sigma_a$, independent from $l$.
Here we take the variance of $\p w l_{ij}$ to be $\sigma_w^2/N$ so that the variance of each $h_i$ is  $\sigma_w^2$, assuming each $\prv x_j$ is fixed (similarity for $\p v l_{ij}$).
Such an initialization scheme is standard in practice.

We make several key ``physical assumptions'' to make theoretical computations tractable:
\begin{assm}[Symmetry of activations and gradients]\label{ass:symAct}
	(a) We assume $\la (\p h l _i)^2 \ra = \la (\p h l _j)^2 \ra$ and $\la (\p x 0 _i)^2 \ra = \la (\p x 0 _j)^2 \ra$ for any $i, j, l$.
	(b) We also assume that the gradient $\pd E/\pd \p x l _i$ with respect to the loss function $E$ satisfies $\la (\pd E/\pd \p x l _i)^2 \ra = \la (\pd E/ \pd \p x l _j)^2 \ra$ for any $i, j, l$.
\end{assm}
One can see that \cref{ass:symAct}(a) is satisfied if the input $\p x 0 \in \{\pm 1\}^N$ and \cref{ass:symAct}(b) is satisfied if \cref{ass:gradInd} below is true and the gradient at the last layer $\pd E/\pd x L \in \{\pm 1 \}^N$.
But in general it is justified both empirically and theoretically as an approximation, because $(\p h l _i)^2 - (\p h l _j)^2$ stays about constant with $l$, but $(\p h l _i)^2$ and $(\p h l _j)^2$ grow rather quickly at the same pace with $l$ (as will be seen later in calculations), so that their additive difference becomes negligible; similarly for $(\p x l _i)^2$ and $(\pd E/\pd \p h l _i)^2$.
\begin{assm}[Gradient independence]\label{ass:gradInd}
	(a) We assume the we use a different set of weights for backpropagation than those used to compute the network outputs, but sampled i.i.d. from the same distributions.
	(b) For any loss function $E$, we assume that the gradient at layer $l$, $\pd E/\pd \p x l _i$, is independent from all activations $\p h {l} _j$ and $\p x {l-1} _j$ from the previous layer.
\end{assm}
\cref{ass:gradInd}(a) was first made in \cite{schoenholz_deep_2017} for computing the mean field theory of gradients for feedforward tanh networks.
This is similar to the practice of feedback alignment \cite{lillicrap_random_2016}.
Even though we are the first to explicitly formulate \cref{ass:gradInd}(b), in fact it was already applied implicitly in the gradient calculations of \cite{schoenholz_deep_2017}.
Note that a priori \cref{ass:gradInd}(b) is not true, as $\pd E/\pd \p x l _i$ depends on $\dot \phi(\p h {l+1} _k)$ for every $k$, which depend on $\p h {l} _j$ for each $j$, and which depends on $\p x {l-1} _k$ for every $k$.
Nevertheless, in practice both subassumptions hold very well.

Now we define the central quantities studied in this paper.
Inevitably, our paper involves a large amount of notation that may be confusing for the first-time reader.
We have included a glossary of symbols (\cref{tab:glossary}) to ameliorate notation confusion.
\begin{defn}\label{defn:length}
Fix an input $\p x 0$.
Define the {\bf length quantities} $\p \qq l := \la (\p h l_1)^2 \ra$ and $\p \pp l := \la (\p x l_1)^2 \ra$ for $l > 0$ and $\p \pp 0 = \|\p x 0\|^2/N$.
Here the expectations $\la \bullet \ra$ are taken over all random initialization of weights and biases for all layers $l$, as $N \to \infty$ (large width limit).
\end{defn}
Note that in our definition, the index $1$ does not matter by \cref{ass:symAct}.
\begin{defn}\label{defn:corr}
Fix two inputs $\p x 0$ and $\p x 0{}'$.
We write $\bullet'$ to denote a quantity $\bullet$ with respect to the input ${\p x 0}'$.
Then define {\bf the correlation quantities} $\p \gamma l:= \la \p h l_1 \p h l_1{}' \ra$ and $\p \lambda l:= \la \p x l_1 \p x l_1{}'\ra$ for $l > 0$ and $\p \gamma 0 = \p x 0 \cdot \p x 0 {}' / N$, where the expectations $\la \bullet \ra$ are taken over all random initialization of weights and biases for all layers $l$, as $N \to \infty$ (large width limit).
Again, here the index $1$ does not matter by \cref{ass:symAct}.
By {\bf metric expressivity}, we mean $\p \mfs l := \f 1 {2N} \la \|\p x l - \p x l {}'\|^2\ra = \f 1 {2N} (\la \|\p x l\|^2\ra + \la \|\p x l {}'\|^2 \ra - 2 \la \p x l \cdot \p x l {}'\ra) = \f 1 2 (\p \pp l + \p \pp l {}') - \p \gamma l$.
Additionally, define {\bf the cosine distance quantities} $\p \ee l := \p \gamma l / \sqrt{\p \pp l \p \pp l {}'}$ and $\p \cc l := \p \lambda l / \sqrt{\p \qq l \p \qq l{}'}$, and we will also call $\p \ee l$ {\bf angular expressivity}.
\end{defn}
In this paper, for the ease of presentation, we assume $\p \pp 0 = \p \pp 0 {}'$.
Then, as we will see, $\p \pp l = \p \pp l{}', \p \qq l = \p \qq l{}'$ for all $l$, and as a result, $\p \ee l = \p \gamma l / \p \pp l$ and $\p \mfs l = \p \pp l - \p \gamma l = (1 - \p \ee l) \p \pp l$.

\begin{defn}\label{defn:grad}
Fix an input $\p x 0$ and a gradient vector $(\pd E/ \pd{\p x L_i})_i$ of some loss function $E$ with respect to the last layer $\p x L$.
% For any quantity $\bullet$, write $\pd^E \bullet$ for $\pdf E{\bullet}$.
Then define {\bf the gradient quantities} $\p \daleth l:= \la (\pd E/\pd \p x l _1)^2 \ra, \p \chi l _\bullet := \la (\pd E/\pd \p \bullet l _1)^2 \ra$ for $\bullet = a, b$, and $\p \chi l _\bullet := \la (\pd E/\pd \p \bullet l _{11})^2 \ra$ for $\bullet = w, v$.
Here the expectations are taken with \cref{ass:gradInd} in mind, over both random initialization of forward and backward weights and biases, as $N \to \infty$ (large width limit).
Again, the index $1$ or $11$ does not matter by \cref{ass:symAct}.
\end{defn}

\newcommand{\TTheta}{{\check\Theta}}
\paragraph{Asymptotic notations.}
The expressions $f = O(g) \iff g = \Omega(f)$ have their typical meanings, and $f = \Theta(g)$ iff $f = O(g), g = O(f)$.
We take $f(x) = \tilde O(g(x)) \iff g(x) = \tilde \Omega(f(x))$ to mean $f(x) = O(g\log^k x)$ for some $k \in \Z$ (this is slightly different from the standard usage of $\tilde O$), and $f = \tilde\Theta(g) \iff f = \tilde O(g) \And g = \tilde O(f).$
We introduce a new notation: $f = \TTheta(g)$ if $f(x) = O(g(x) \cdot x^\eps)$ and $f(x) = \Omega(g(x) \cdot x^{-\eps})$, as $x \to \infty$, for any $\eps > 0$.
All asymptotic notations are sign-less, i.e. can indicate either positive or negative quantities, unless stated otherwise.
\section{Overview}

The primary reason we may say anything about the average behavior of any of the above quantities is the central limit theorem: every time the activations of the previous layer pass through an affine layer whose weights are sampled i.i.d., the output is a sum of a large number of random variables, and thus follows approximately Gaussian distributions.
The mean and variance of these distributions can be computed by keeping track of the mean and variances of the activations in the previous layer.

In what follows, we use this technique to derive recurrence equations governing $\pp, \qq, \gamma, \lambda, \daleth$ for different architectures and different activation functions.
We use these equations to investigate the dynamics of $\ee$ and $\mfs$, the key quantities in the forward pass, and the dynamics of $\daleth$, the key quantity in the backward pass.

The cosine distance $\ee$ in some sense measures the angular geometry of two vectors.
If $\ee = 1$, then the vectors are parallel; if $\ee = 0$, then they are orthogonal.
Just as in \cite{poole_exponential_2016} and \cite{schoenholz_deep_2017}, we will show that in all of the architectures and activations we consider in this paper, $\p \ee l$ converges to a fixed point $\ee^*$ as $l \to \infty$ \footnote{Under simplified conditions, \citet{daniely_toward_2016} showed that there exists a fixed point for any ``well-behaved'' activation function in a feedforward net.
	However, this result does not apply to architectures with residual connections.}.
Thus, on the average, as vectors propagate through network, the geometry of the original input space, for example, linear separability, is ``forgotten'' by residual networks as well as by vanilla networks.
But we will prove and verify experimentally that, while \citet{poole_exponential_2016} and \cite{schoenholz_deep_2017} showed that the convergence rate to $\ee^*$ is exponential in a vanilla network, the convergence rate is rather only polynomial in residual networks, for tanh and $\alpha$-ReLU (\cref{defn:alphaReLU}) nonlinearities; see \cref{thm:dalethRecReduced}, \cref{thm:eDynamicsFullResTanh}, \cref{thm:ReLUSquaredConvergence}, and \cref{thm:alphaReLUeConvergence}.
%In fact, the convergence in $\alpha$-ReLU residual networks is so slow that it appears even logarithmic in simulation (\cref{thm:pseudoLogarithmConvergence}).
This slow convergence preserves geometric information in the input space, and allows a typical residual network to ``hover over the edge of chaos'': Even when the cosine distance $\p\ee l$ converges to 0, corresponding to ``chaos'', (resp. 1, corresponding to ``stability''), for the number of layers usually seen in practice, $\p \ee l$ will reside well away from 0 (resp. 1).

Similarly, the quantity $\mfs$ measures the metric geometry of two vectors.
The evolution of $\p \mfs l$ with $l$ tells us the ability of the average network to separate two input points in terms of Euclidean distance.
Again, for tanh and $\alpha$-ReLU ($\alpha < 1$) nonlinearities, $\mfs$ varies only polynomially with $l$.

On the other hand, $\p \daleth l$ measures the size of gradient at layer $l$, and through it we track the dynamics of gradient backpropagation, be it explosion or vanishing.
In contrast to vanilla tanh networks, which can experience both of these two phenomenon depending on the initialization variances, typical residual networks cannot have vanishing gradient, in the sense of vanishing $\p \daleth l$ as $l \to 1$; see \cref{thm:dalethRecReduced} and \cref{thm:dalethRecFull}.
Furthermore, while vanilla tanh networks exhibit exponentially vanishing or exploding gradients, all of the activation/architecture pairings considered here, except the full residual network with ReLU, have subexponential gradient dynamics.
While tanh residual networks (reduced or full) has $\p \daleth 0 \approx \exp(\Theta(\sqrt l)) \p \daleth l$ (\cref{thm:dalethExpSqrtTanhFullRes}), $\alpha$-ReLU residual networks for $\alpha < 1$ have $\p \daleth 0 \approx \poly(l) \p \daleth l$ (\cref{thm:dalethDynamicsAlphaReLU}).
Instead of $\pd E/\pd x_i$, we may also consider the size of gradients of actual trainable parameters.
For tanh and $\alpha$-ReLU with $\alpha < 1$, they are still subexponential and polynomial (\cref{thm:alphaReLUAllGradients}).
On the other hand, while $\p \daleth 0 = \exp(\Theta(l))\p \daleth l$ for a ReLU resnet, its weight gradients have size independent of layer, within $O(1)$ (\cref{thm:alphaReLUAllGradients})!
This is the only instance in this paper of gradient norm being completely preserved across layers.

The above overviews the theoretical portion of this paper.
Through experiments, we discover that we can very accurately predict whether one random initialization leads to better performance than another on the test set, after training, by leveraging this theory we build.
Residual networks of different nonlinearities have different {\it controlling quantities}: for resnets with tanh, the optimal initialization is obtained by controlling the gradient explosion $\p \daleth 0 / \p \daleth L$; whereas for ReLU and $\alpha$-ReLU, the optimal initialization is obtained by maximizing $\mfs$ without running into numerical issues (with floating point computation).
%Intuitively, $\p\pp l - \p\gamma l$ indicates how much separation the network has imposed on the two inputs $\p x 0$ and $\p x 0 {}'$, with respect to a final linear layer: If $\p x l$ is completely decorrelated from $\p x l {}'$ then $\p \gamma l \approx 0$ and $\p \pp l - \p \gamma l \approx \p \pp l$; if $\p x l$ is too close to $\p x l {}'$, then $\p \pp l - \p \gamma l \approx 0$.
See \cref{sec:experiments} for details.

Over the course of our investigation of $\alpha$-ReLU, we derived several new identities involving the associated kernel functions, first defined in \cite{cho_kernel_2009}, which relate them to the zeroth Bessel functions (\cref{lemma:JalphaBessel,lemma:LAlphaRec,lemma:JalphaRec,lemma:JalphaGrad}).

\section{Theoretical Results}

In what follows in the main text, we assume $\sigma_\bullet > 0$ for all $\bullet = w, v, b, a$; in the appendix, the formal statement of each main theorem will contain results for other cases.
We are interested in the two major categories of nonlinearities used today: tanh-like and rectified units.
We make the following formal definitions as a foundation for further consideration.
\begin{defn}
We say a function $\phi$ is {\bf tanh-like} if $\phi$ is antisymmetric ($\phi(-x) = -\phi(x)$), $|\phi(x)| \le 1$ for all $x$, $\phi(x) \ge 0, \forall x \ge 0$, and $\phi(x)$ monotonically increases to 1 as $x \to \infty$.
\end{defn}
\begin{defn}\label{defn:alphaReLU}
Define the $\alpha$-ReLU $\psi_\alpha(x) = x^\alpha$ if $x > 0$ and 0 otherwise.~\footnote{
	Note that in practice, to avoid the diverging gradient $\dot \psi_\alpha(x) \to \infty$ as $x \to 0$, we can use a tempered version $\Psi_\alpha(x)$ of $\alpha$-ReLU, defined by $\Psi_\alpha(x) = (x + \eps)^\alpha - \eps^\alpha$ on $x > 0$ and 0 otherwise, for some small $\eps > 0$.
	The conclusions of this paper on $\psi_\alpha$ should hold similarly for $\Psi_\alpha$ as well.}
%= \begin{cases} x^\alpha & \text{if $x > 0$}\\ 0 & \text{otherwise.}\end{cases}$
\end{defn}

\begin{table}[t]
  \caption{Main Recurrences}
  \label{tab:recurrences}
  \centering
  \begin{tabular}{lll}
    \toprule
    Antisymmetric/RRN& Any/FRN \\
    \midrule
    \parbox{3cm}{
	\begin{align*}
		\qq &= \sigma_w^2 \prv \pp + \sigma_b^2 &
		\pp &= \Vt \phi( \qq) + \prv \pp\\
		\lambda &= \sigma_w^2 \prv \gamma + \sigma_b^2 &
		\gamma &= \Wt \phi(\qq, \lambda) + \prv \gamma\\
		&& \prv \daleth &= (\sigma_w^2  \Vt \dot \phi( \qq) + 1)\daleth
	\end{align*}
	}
	&
	\parbox{3cm}{
	\begin{align*}
	\qq &= \sigma_w^2 \prv \pp + \sigma_b^2 &
	\pp &= \sigma_v^2 \Vt \phi( \qq) + \sigma_a^2 + \prv \pp\\
	\lambda &= \sigma_w^2 \prv \gamma + \sigma_b^2 &
	\gamma &= \sigma_v^2 \Wt \phi(\qq, \lambda) + \sigma_a^2 + \prv \gamma\\
	&& \prv \daleth &= (\sigma_v^2\sigma_w^2  \Vt \dot \phi( \qq) + 1)\daleth
	\end{align*}
	}\\
	\midrule
	 Theorems~\ref{thm:p_q_linear},
	 \ref{thm:lambda_gamma_recurrence},
	 \ref{thm:dalethRecReduced} &
	 Theorems~\ref{thm:fullResPQRec},
	 \ref{thm:full_res_l_g_recurr},
	 \ref{thm:dalethRecFull}\\
    \bottomrule
  \end{tabular}
\end{table}
By applying the central limit theorem as described in the last section, we derive a set of recurrences for different activation/architecture pairs, shown in \cref{tab:recurrences} (see appendix for proofs).
They leverage certain integral transforms \footnote{\citet{daniely_toward_2016} called the version of $\Wt\phi$ with fixed $\rho = 1$ the ``dual function'' of $\phi$.} as in the following
\begin{defn}\label{defn:integralTransform}
Define the transforms $\Vt$ and $\Wt$ by $\Vt \phi(q) := \EV[\phi(z)^2: z \sim \Gaus(0, q)]$ and $\Wt \phi(\rho, \nu) := \EV[\phi(z)\phi(z'): (z, z') \sim \Gaus(0, \begin{pmatrix}\rho & \nu \\ \nu & \rho \end{pmatrix})]$.
\end{defn}

These recurrences are able to track the corresponding quantities in practice very well.
For example, \cref{fig:theory_tracks_pratice} compares theory vs experiments for the tanh/FRN pair.
The agreement is very good for tanh/RRN (not shown, but similar to the case of tanh/FRN with $\sigma_v = 1$ and $\sigma_a = 0$) and $\alpha$-ReLU/FRN as well (see \cref{fig:alphaReLUTheoryVsEmpirics}).

As mentioned in previous sections, we seek to characterize the long term/high depth behavior of all of the quantities defined in \cref{sec:background}.
To do so, we solve for the asymptotics of the recurrences in \cref{tab:recurrences}, where $\phi$ is instantiated with tanh or $\alpha$-ReLU.
Our main dynamics results are summarized in \cref{tab:dynamics}.

\begin{table}[t]
  \caption{Summary of Main Dynamics Results.
  Note that while $\p\daleth l$ is exponential for ReLU/FRN, the gradients with respect to weight parameters have norms ($\chi_w$ and $\chi_v$) constant in $l$ (\cref{thm:alphaReLUAllGradients}).
  Also, the $\p \daleth l$ entry for $\alpha$-ReLU is for $\alpha \in (3/4, 1)$ only}
  \label{tab:dynamics}
  \centering
  \begin{tabular}{llllll}
    \toprule
    		&Tanh/RRN	&Tanh/FRN	&ReLU/FRN	&$\alpha$-ReLU/FRN, $\alpha<1$\\
    \midrule
    $\p \pp l$		&$\Theta(l)$, \hfill \ref{thm:p_q_linear}
    					&$\Theta(l)$, \hfill\ref{thm:pIsLinearTanh}
    								&$\exp(\Theta(l))$, \hfill\ref{thm:pDynamicAlphaReLU}
    									&$\Theta(l^{1/(1-\alpha)})$, \hfill\ref{thm:pDynamicAlphaReLU}\\
    $\p \mfs l$		&$\Theta(l)$, \hfill \ref{thm:edynamics}
    					&$\Theta(l)$, \hfill \ref{thm:eDynamicsFullResTanh}
    								&$\exp(\Theta(l))$, \hfill \ref{thm:ReLUSquaredConvergence}
    											&$\Theta(l^{1/(1-\alpha)})$, \hfill \ref{thm:alphaReLUeConvergence}\\
    $\p \ee l - \ee^*$
    		&$\TTheta(l^{\f2 \pi -1})$, \hfill\ref{thm:edynamics}
						&$\poly(l)$, \hfill\ref{thm:eDynamicsFullResTanh}
									&$\Theta(l^{-2})$, \hfill\ref{thm:ReLUSquaredConvergence}
										&$\poly(l)$, \hfill\ref{thm:alphaReLUeConvergence}\\
	$\p \daleth l$&$\exp(\Theta(\sqrt l))$, \hfill\ref{thm:dalethExpSqrtTanh}
						&$\exp(\Theta(\sqrt l))$, \hfill\ref{thm:dalethRecFull}
									&$\exp(\Theta(l))$, \hfill\ref{thm:dalethDynamicsAlphaReLU}
										&$\Theta(l^{\f{\alpha^2}{(1-\alpha)(2 \alpha - 1)}})$, \hfill\ref{thm:dalethDynamicsAlphaReLU}\\
    \bottomrule
  \end{tabular}
\end{table}

\subsection{Tanh}

\begin{figure}
\centering
\includegraphics[height=.17\textheight]{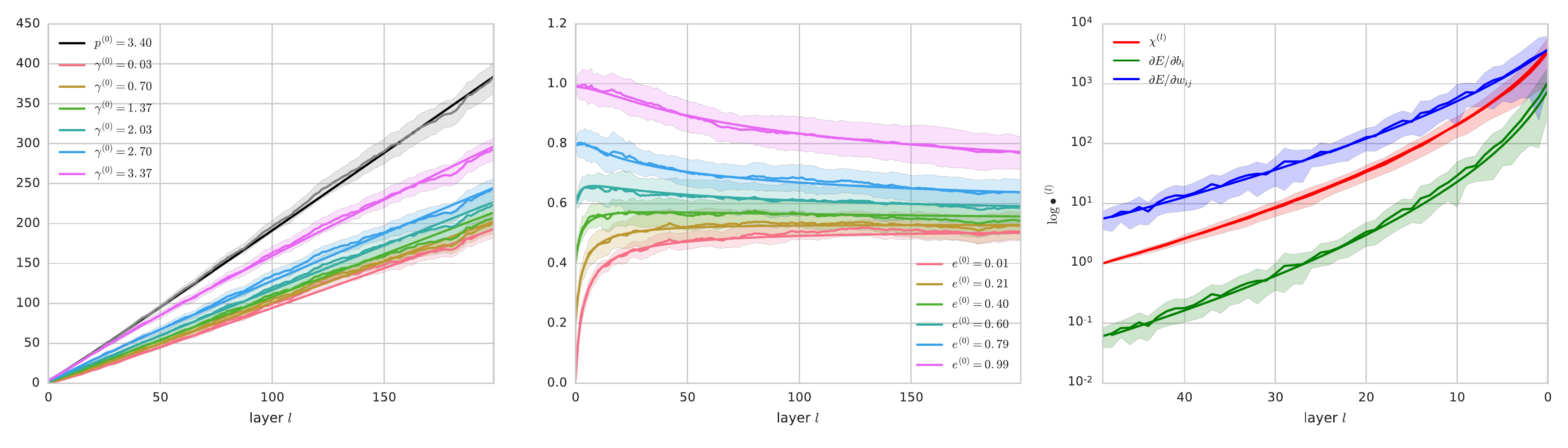}
\caption{Our equations predict the relevant quantities very well in practice.
These plots make the comparison between prediction and measurements for the full resnet with tanh activation, with $\sigma_v^2 = 1.5$, $\sigma_a^2 = .5$, $\sigma_w^2 = 1.69$, $\sigma_b^2 = .49$.
Left-to-right:
{\bf (a)} $\p \pp l$ and $\p \gamma l$ against layer $l$ for 200 layers.
{\bf (b)} $\p \ee l = \p \gamma l /\p \pp l$ against $l$ for 200 layers.
Both (a) and (b) trace out curves for different initial conditions.
{\bf (c)} Different gradient quantities against $l$ for 50 layers.
From left to right the layer number $l$ decreases, following the direction of backpropagation.
Notice that the gradient increases in norm as $l \to 1$.
All three figures exhibit smooth curves, which are theoretical estimates, and irregular curves with shades around them, which indicate empirical means and standard deviations (both of which taken in regular scale, not log scale).
(a) and (b) are made with 20 runs of resnets of width 1000.
(c) is made with 25 runs of resnets of width 250.}
\label{fig:theory_tracks_pratice}
\end{figure}

%Let's start with our results for $\tanh$.

\paragraph{Forward dynamics.}
When $\phi = \tanh$, $\p \pp l$ and $\p \qq l$ increase as $\Theta(l)$ in either RRN or FRN (\cref{thm:p_q_linear}), as one might expect by observing that $\Vt \tanh( \qq) \to 1$ as $\qq \to \infty$ so that, for example in the RRN case, the recurrence $\pp = \Vt \tanh( \qq) + \prv \pp$ becomes $\pp = 1 + \prv \pp$.
This is confirmed graphically by the black lines of the leftmost chart of \cref{fig:theory_tracks_pratice}.
We carefully verify that this intuition is correct in its proof in the appendix, and find that in fact $\p \pp l \sim l$ in the RRN case and $\p \pp l \sim (\sigma_v^2 + \sigma_a^2)l$ in the FRN case.

What about $\p \gamma l$?
The middle chart of \cref{fig:theory_tracks_pratice} shows that over time, $\p \ee l = \p \gamma l / \p \pp l$ contracts toward the center of the interval $[0, 1]$, but from the looks of it, it is not clear whether there is a stable fixed point $\ee^*$ of $\ee$ or not.  
We prove that, in fact, {\bf all trajectories of $\ee$ not starting at 1 do converge to a single fixed point, but only at a polynomial rate}, in both the RRN and FRN cases (\cref{thm:p_q_linear} and \cref{thm:full_res_l_g_recurr}); we can even explicitly compute the fixed point and the rate of convergence:
For FRN, there is a {\bf unique stable fixed point} $\ee^* < 1$ determined by the equation
$$\ee^* = \f 1 {\sigma_v^2 + \sigma_a^2}[\sigma_v^2 \f 2 \pi \arcsin\lp \ee^* \rp + \sigma_a^2],$$
and $|\ee^* - \p \ee l|$ decreases like $l^{-\delta^*}$, where
$$\delta^* := 1 - \f 2 \pi \f 1 {\sqrt{1 - (\ee^*)^2}} \f{\sigma_v^2 }{\sigma_v^2 + \sigma_a^2}.$$
Since $\ee^* < 1$, $\mfs = (1 - \ee) \pp = \Theta(\pp) = \Theta(l).$
The case of RRN can be viewed as a special case of the above, setting $\sigma_v^2 = 1$ and $\sigma_a^2 = 0$, which yields $\ee^* = 0$ and $\delta^* = 1 - \f 2 \pi$.
We observe that both $\ee^*$ and $\delta^*$ only depend on the ratio $\rho := \sigma_a/\sigma_v$, so in \cref{fig:edelta_plot} we graph these two quantities as a function of $\rho$.
$\ee^*$ and $\delta^*$ both increase with $\rho$ and asymptotically approach 1 and $\nicefrac 1 2$ respectively from below.
When $\rho = \sigma_a = 0$, $\ee^* = 0$ and $\delta^* = 1 - \f 2 \pi$.
Thus the rate of convergence at its {\bf slowest} for tanh/FRN is $\delta^* = 1 - \f 2 \pi \approx 0.36338$, where asymptotically the network tends toward a {\bf chaotic regime} $\ee^* = 0$, corresponding to a large weight variance and a small bias variance; it at its {\bf fastest} is $\delta^* = \nicefrac 1 2$, where asymptotically the network tends toward a {\bf stable regime} $\ee^* = 1$, corresponding to a large bias variance and small weight variance.
We verify $\delta^*$ by comparing $\p \ee l - \p \ee {l-1}$ to $l^{-\delta^* - 1}$ in log-log scale.
If $\p \ee l = \Theta(l^{-\delta^*})$, then $\p \ee l - \p \ee {l-1} = \Theta(l^{-\delta^* - 1})$ and should obtain the same slope as $l^{-\delta^* - 1}$ as $l \to \infty$.
The middle figure of \cref{fig:edelta_plot} ascertains that this is indeed the case, starting around layer number 400. 

\begin{figure}
\centering
\includegraphics[height=.16\textheight,width=.3\textwidth]{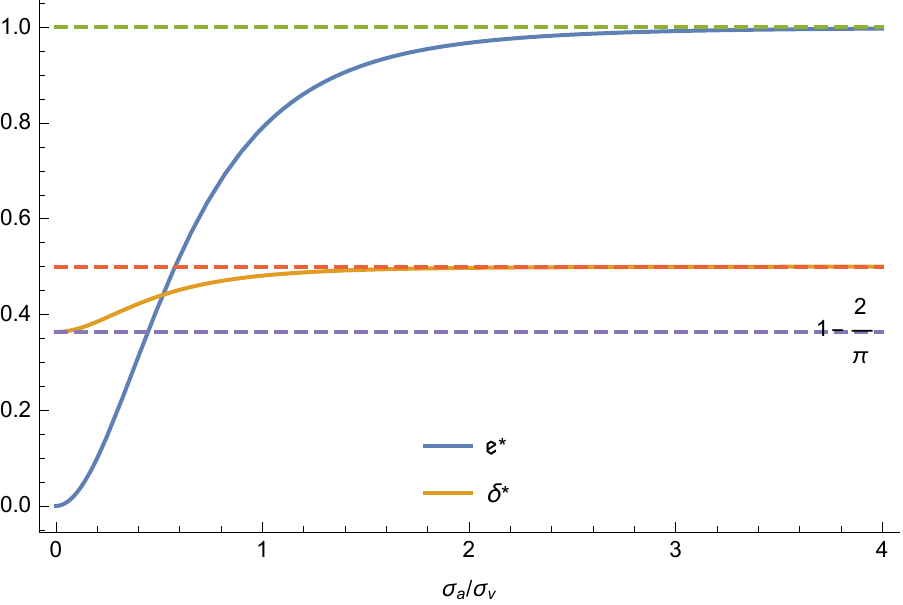}
\includegraphics[height=.16\textheight]{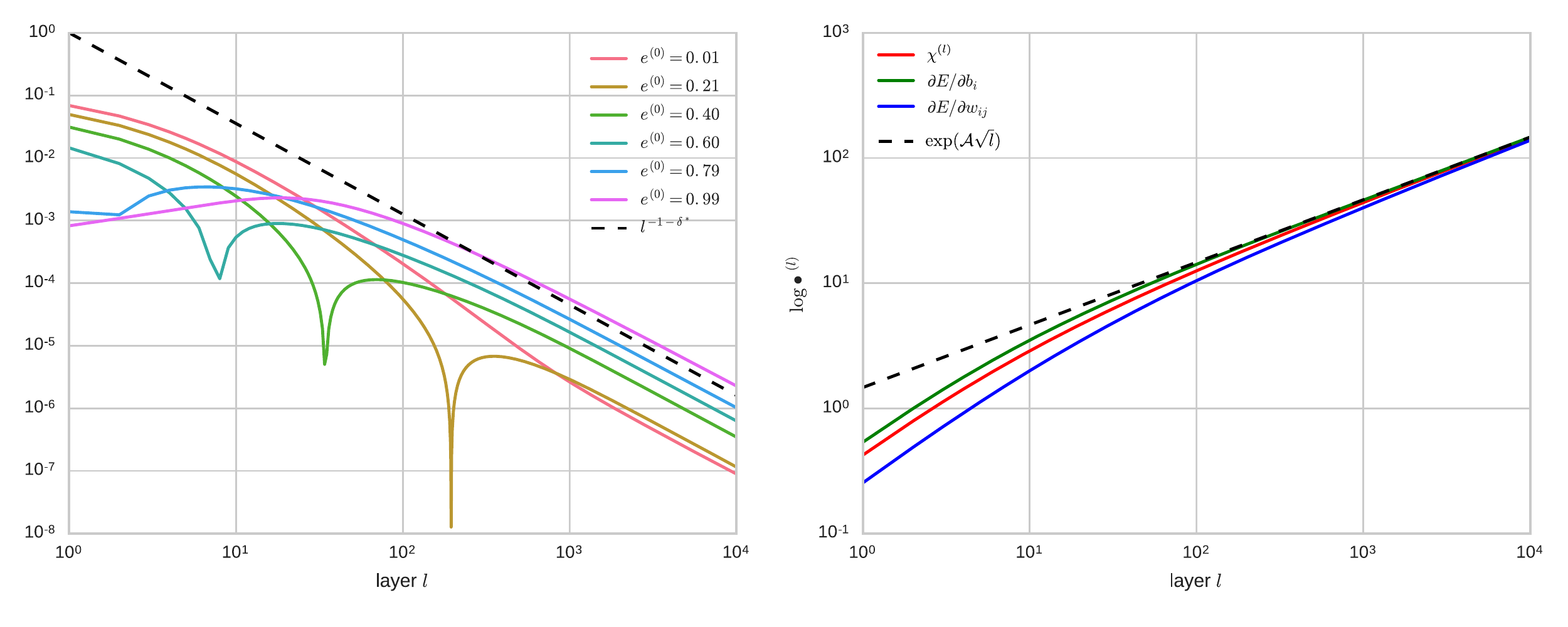}
\caption[caption with footnote]{Left-to-right:
{\bf (a)} Plots of $\ee^*$ and $\delta^*$ against $\sigma_a/\sigma_v$.
{\bf (b)} In log-log scale: the dashed line is $l^{-\delta^* - 1}$, and the colored lines are $\p \ee l - \p \ee {l-1}$ for different initial conditions $\p \ee 0$.
That they become parallel at about $l = 400$ on verifies that $\p \ee l = \Theta(l^{-\delta^*})$.
\footnote{{A more natural visualization is to graph $\p \ee l - \ee^*$ versus $l^{-\delta^*}$, but because of floating point precision, $\p \ee l - \ee^*$ doesn't converge to 0, but a small number close to 0, so that the log-log plot wouldn't look like what is expected.\label{footnote:plotDelta}}}
{\bf (c)} In log-log scale: The dashed line is $\mathcal A \sqrt l$  ($\mathcal A$ given in \cref{thm:dalethExpSqrtTanhFullRes}), and the colored lines are $\log(\p \bullet 1/\p \bullet l)$ for $\bullet = \daleth, \chi_b, \chi_w$.
That they all converge together starting around $l=1000$ indicates that the approximation in \cref{thm:dalethExpSqrtTanhFullRes} is very good for large $l$.}
\label{fig:edelta_plot}
\end{figure}

\paragraph{Backward dynamics.} Finally, we show that the gradient is approximated by 
\begin{align}
\p \daleth {m} &= \exp(\mathcal A(\sqrt{l} - \sqrt{m}) + O(\log l - \log m))\p \daleth {l} \label{eqn:tanhGradEst}\tag{$\star$}
\end{align}
where $\mathcal A = \f 4 3 \sqrt{\f 2 \pi} \sigma_w$ in the RRN case and $
\mathcal A = \f 4 3 \sqrt{\f 2 \pi} \f{\sigma_v^2 \sigma_w}{\sqrt{\sigma_v^2 + \sigma_a^2}}$ in the FRN case (\cref{thm:dalethExpSqrtTanh} and \cref{thm:dalethExpSqrtTanhFullRes}).
The rightmost plot of \cref{fig:edelta_plot} verifies that indeed, for large $l \ge 1000$, this is a very good approximation.
This demonstrates that the mean field assumption of independent backpropagation weights is very practical and convenient even for residual networks.

Note that in the FRN case, the constant $\mathcal A$ can be decomposed into $\mathcal A = \f 4 3 \sqrt{\f 2 \pi} \cdot \sigma_v \cdot \sigma_w \cdot (1 + \sigma_a^2/\sigma_v^2)^{-1/2}$.
Consider the ratio $\rho := \sigma_a/\sigma_v$.
If $\rho \gg 1$, then $\ee^* \approx 1$ (\cref{fig:jjj_vs_id_main}), meaning that the typical network essentially computes a constant function, and thus unexpressive; at the same time, large $\rho$ makes $\mathcal A$ small, and thus ameliorating the gradient explosion problem, making the network more trainable.
On the other hand, if $\rho \ll 1$, then $\ee^* \approx 0$ (\cref{fig:jjj_vs_id_main}), the typical network can tease out the finest differences between any two input vectors, and a final linear layer on top of such a network should be able to express a wide variety of functions \cite{poole_exponential_2016}; at the same time, small $\rho$ increases $\mathcal A$, worsening the gradient explosion problem, making the network less trainable.
This is the same expressivity-trainability tradeoff discussed in \cite{schoenholz_deep_2017}.

\subsection{$\alpha$-ReLU}

\paragraph{Forward dynamics.} As with the tanh case, to deduce the asymptotic behavior of random $\alpha$-ReLU resnets, we need to understand the transforms $\Vt\psi_\alpha$ and $\Wt \psi_\alpha$.
Fortunately, $\Vt\psi_\alpha$ has a closed form, and $\Wt \psi_\alpha$ has been studied before \cite{cho_kernel_2009}.
In particular, if $\alpha > -\f 1 2$, then
$\Vt\psi_\alpha( \qq) = \cV_\alpha \qq^{\alpha}$, where $\cV_\alpha$ is a constant with a closed form given by \cref{lemma:VtPsiAlpha}.
In addition, by \cite{cho_kernel_2009}, we know that $\Wt\psi_\alpha( \qq, \cc \qq) = \Vt\psi_\alpha( \qq) \JJ_\alpha(\cc)$ for $\JJ_\alpha$ given in \cref{sec:AlphaReluForwardProofs}.
\cref{fig:jjj_vs_id_main} shows a comparison of $\JJ_\alpha$ for different $\alpha$s along with the identity function.

Substituting in $\cV_\alpha \qq^\alpha$ for $\Vt \psi_\alpha$, we get a difference equation $\pp - \prv \pp = \sigma_v^2 \cV_\alpha (\sigma_w^2 \prv \pp + \sigma_b^2)^\alpha + \sigma_a^2$ governing the evolution of $\pp$.
This should be reminiscent of the differential equation $\dot P(l) = C P(l)^\alpha$, which has solution $\propto l^{1/(1-\alpha)}$ for $\alpha < 1$, and $\propto \exp(Cl)$ when $\alpha = 1$.
And indeed, the solutions $\p \pp l$ to these difference equations behave asymptotically exactly like so (\cref{thm:pDynamicAlphaReLU}).
Thus {\bf ReLU behaves very explosively compared to $\alpha$-ReLU with $\alpha<1$}.
In fact, in simulations, for $\sigma_w^2 = 1.69$ and $\sigma_v^2 = 1.5$, the ReLU resnets overflows into \texttt{inf}s after around 100 layers, while there's no problem from any other kind of networks we consider.

Regardless, {\bf $\alpha$-ReLU for all $\alpha$ massages $\p \ee l$ toward a fixed point $\ee^*$ that depends on $\alpha$}.
{ When $\phi = \psi_1$, the standard ReLU, $\p \ee l$ converges to 1 asymptotically as $C l^{-2}$ for an explicit constant $C$ depending on $\sigma_v$ and $\sigma_w$ only (\cref{thm:ReLUSquaredConvergence}), so that $\mfs = (1 - \ee)\pp = \Theta(l^{-2}\exp(\Theta(l))) = \exp(\Theta(l)).$
When $\phi = \psi_\alpha$ for $\alpha < 1$, then $\p \ee l$ converges to the nonunit fixed point $\ee^*$ of $\JJ_\alpha$ at a rate of $\TTheta(l^{-\mu})$, where $\mu = (1-\dot \JJ_\alpha(\ee^*))/(1-\alpha)$ is independent of the variances (\cref{thm:alphaReLUeConvergence}), so that $\mfs = \Theta(\pp)$.}
These rates are verified in \cref{fig:alphaReLUVerifyExponents}.

\paragraph{Backward dynamics.} Finally, we have also characterized the rate of gradient growth for any $\alpha \in (\f 3 4, 1]$.
\unskip\footnote{Our derivations actually apply to all $\alpha \in (\f 1 2, 1]$, where at $\alpha = \f 1 2$, the expected norm of the gradient diverges within our mean field formalism.
However, at $\alpha \le \f 3 4$, the variance of the gradient already diverges (\cref{thm:dalethInfVarAlphaReLU}), so we cannot expect the empirical values to agree with our theoretical predictions.
But in fact, empirically our theoretical predictions seem to form an upper bound on the gradient norms (see \cref{fig:alphaReLUTheoryVsEmpirics}).
}
{\bf In the case of $\alpha = 1$, the dynamics of $\daleth$ is exponential}, the same as that of $\pp$, $\p\daleth{l-m} = \p \daleth{l} B^m$ where $B =\f 1 2 \sigma_v^2 \sigma_w^2 + 1$.
{\bf For $\alpha \in (\f 3 4, 1)$, the dynamics is polynomial}, but with different exponent in general from that of the forward pass: $\p\daleth{l-m} = \Theta(1) \p \daleth{l} (l/(l-m))^R$ for $R = \f{\alpha^2}{(1-\alpha)(2 \alpha - 1)}$, where the constants in $\Theta(1)$ do not depend on $l$ or $m$.
This exponent $R$ is minimized on $\alpha \in [\f 3 4, 1)$ at $\alpha = \nicefrac 3 4$, where $R = \nicefrac 9 2$ (but on $\alpha \in (\f 1 2, 1)$ it is minimized at $\alpha = \nicefrac 2 3$, where $R = 4$); see \cref{fig:backprop_exponent_alpha-relu}.
These exponents are verified empirically in \cref{fig:alphaReLUVerifyExponents}.

Looking only at $\daleth$ and the gradients against the biases, it seems that ReLU suffers from a dramatic case of exploding gradients.
But in fact, because $\daleth$ gains a factor of $B$ moving backwards while $\pp$ loses a factor of $B$, the gradient norm $\chi_w^{(l-m)}$ (and similarly for $\chi_v^{(l-m)}$) is independent of how far, $m$, the gradient has been propagated (\cref{thm:alphaReLUAllGradients}) --- this is certainly the best gradient preservation among all of the models considered in this paper.
Thus strangely, random ReLU FRN exhibits both the best (constant for $v$ and $w$) and the worse (exponential for $a$ and $b$) gradient dynamics.
This begs the question, then, is this a better deal than other $\alpha$-ReLU for which for any learnable parameter we have at most a polynomial blowup with depth in its gradient?
%We have started experiments toward this question, but so far the answer remains unclear. 
% TODO answer this question.
Our experiments (discussed below) show that $\alpha$-ReLU is useful to the extent that smaller $\alpha$ avoids numerical issues with exponentiating forward and backward dynamics, but the best performance is given by the largest $\alpha$ that avoids them (\cref{fig:tanhHeatmaps}(c, d)); in fact, the metric expressivity $\mfs$, determines performance, not gradient explosion (see $\alpha$-ReLU experiments).

\section{Experimental Results}
\label{sec:experiments}
\newcommand{\height}{.16\textheight}
\newcommand{\negsp}{\hspace{0em}}
\begin{figure}
	\centering
	\includegraphics[height=\height]{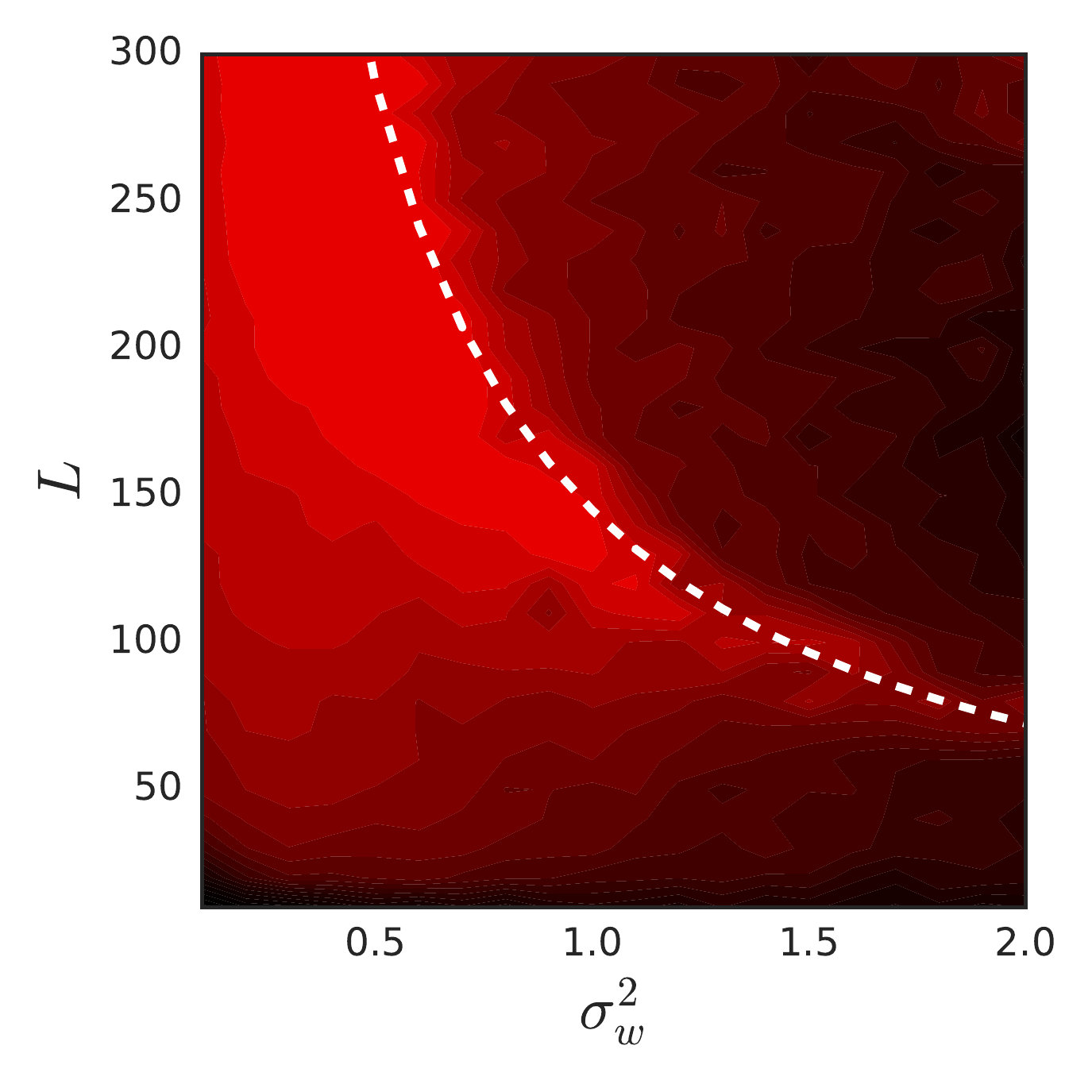}
	\negsp\includegraphics[height=\height]{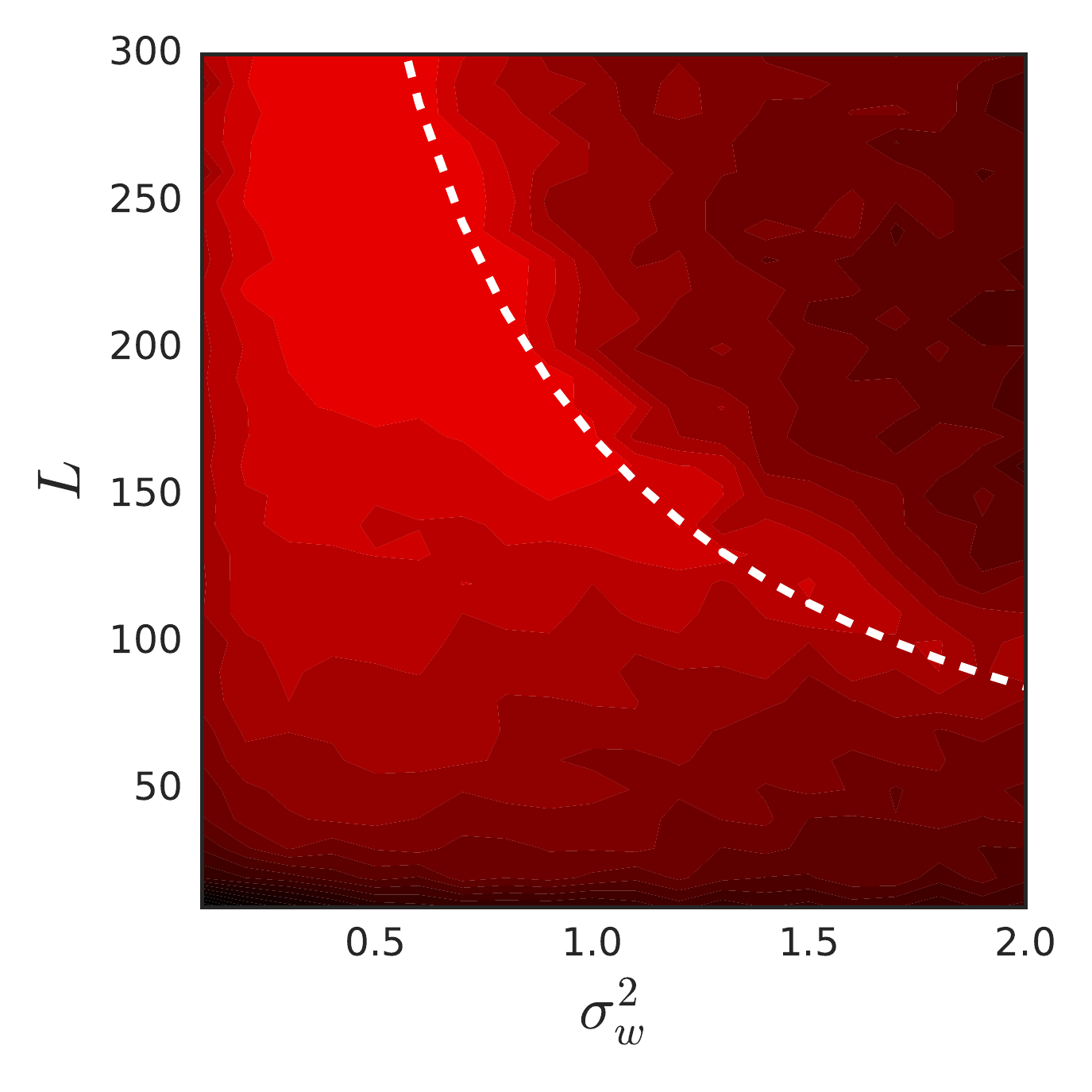}
	\negsp\includegraphics[height=\height]{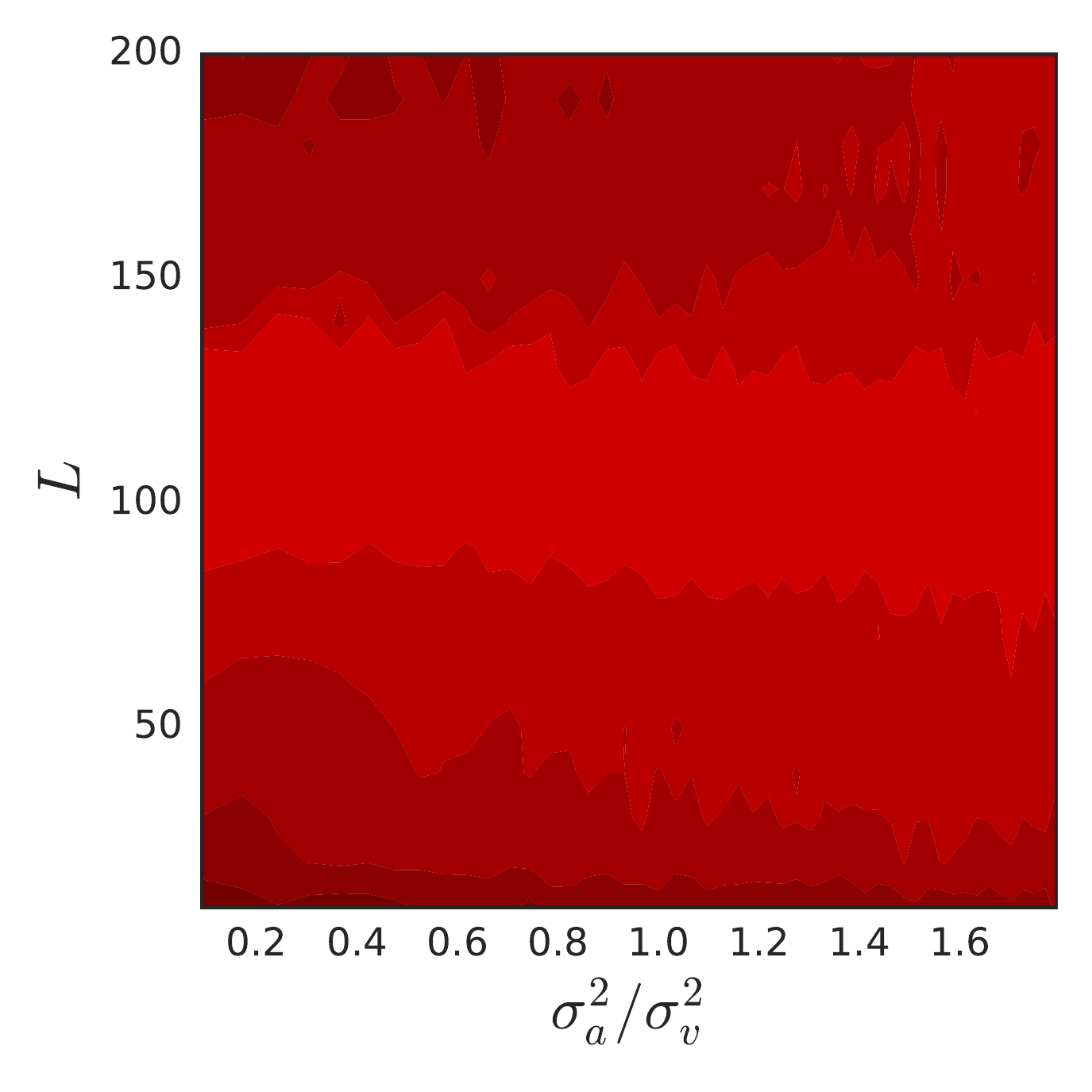}\\
	\negsp\includegraphics[height=\height]{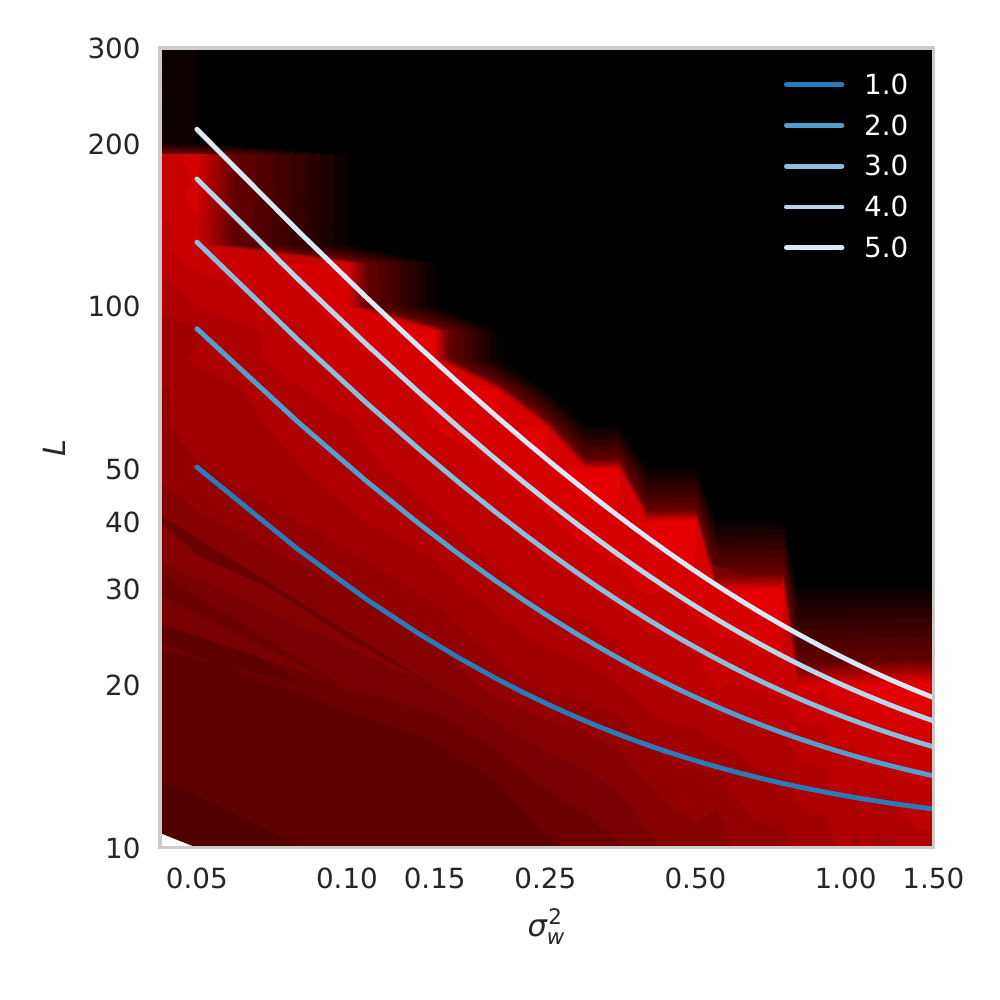}
	\negsp\includegraphics[height=\height]{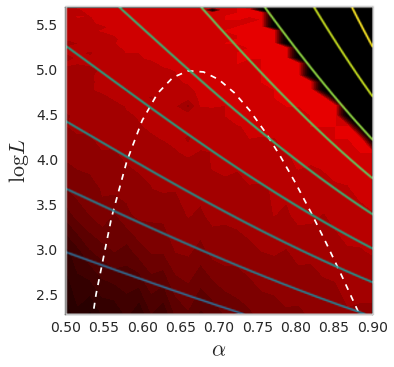}

	\caption{From left to right, top to bottom: \textbf{(a)} and \textbf{(b)}: $\sigma_w^2$, $L$, and test set accuracy of a grid of tanh reduced (left) and full (right) resnets trained on MNIST.
		Color indicates performance, with ligher colors indicating higher accuracy on test set.
		Other than the values on the axes, we have fixed $\sigma_b^2 = \sigma_a^2 = \f 1 2$ and $\sigma_v^2 = 1$.
		The white dotted lines are given by $\sigma_w^2 L = C$, where $C = 170$ on the left and $C = 145$ on the right.
		We see that both dotted lines accurately predict the largest optimal $\sigma_w$ for each depth $L$.
		\textbf{(c)} Varying the ratio $\sigma_a^2/\sigma_v^2$ while fixing $\sigma_v/\sqrt{1 + \sigma_a^2/\sigma_v^2}$, and thus fixing $\mathcal A$, the leading constant of $\log \p \daleth 0 / \p \daleth L$.
		\textbf{(d)} in log-log scale: Heatmap gives the test accuracies of ReLU FRN for varying $\sigma_w^2$ and $L$.
		Curves give level sets for the log ratios $\log \p \mfs L / \p \mfs 0 \approx \log \p \pp L / \p \pp 0 \approx \log \p\daleth 0 / \p \daleth L = L \log(1 + \sigma_v^2 \sigma_w^2/2)$.
		\textbf{(e)} Red heatmap shows the test accuracies of a grid of $\alpha$-ReLU FRN with varying $\alpha$ and $L$ as shown, but with all $\sigma_\bullet$s fixed.
		The white dashed curve gives a typical contour line of $L^R = \text{const}$, where $R = \f {\alpha^2}{(1-\alpha)(2\alpha-1)}.$
		The yellow-to-blue curves form a set of level curves for $\p \mfs l = \p \pp l - \p \gamma l = \text{const}$, with yellow curves corresponding to higher levels.
	}
\label{fig:tanhHeatmaps}
\end{figure}

Our experiments show a dichotomy of what matters in initialization: for tanh resnets, quality of an initialization is determined by how much gradient explosion there is (measured by $\p \daleth 0 / \p \daleth L$); for ($\alpha$-)ReLU resnets, it is determined by how expressive the random network is (measured by the metric expressivity $\p \mfs L$).
We hypothesize this is because in tanh resnets, the gradient dynamics is much more explosive than the expressivity dynamics ($\exp(\Theta(\sqrt l))$ vs $\Theta(l)$), whereas for ReLU it's somewhat the opposite ($\chi_w, \chi_v = \Theta(1)$ vs $\mfs = \exp(\Theta(l))$).

\paragraph{Tanh, vary $\sigma_w$.}
We train a grid of reduced and full tanh resnets on MNIST, varying the variance $\sigma_w^2$ and the number of layers (for FRN we fix $\sigma_v = 1$).
The results are indicated in \cref{fig:tanhHeatmaps}(a, b).
We see that in either model, deeper resnets favor much smaller $\sigma_w$ than shallower ones.
%(with the effect most pronounced at the beginning of training, \cref{fig:expresults}), and more pronounced in the reduced resnet than the full ones.
The white dotted lines in \cref{fig:tanhHeatmaps}(a, b) confirm our theory: according to \cref{eqn:tanhGradEst}, for the same gradient ratio $R = \p \daleth 0 / \p \daleth L$, we want $\log R \approx \sigma_w \sqrt L$.
% were one to limit the ratio of the gradient squared norm between the first and the last layer by $G$, in order to guarantee trainability, then we'd like $\log G \gtrsim \sigma_w \sqrt l$.
%Thus as $L$ increases, for the same level of performance we expect to see $\sigma_w$ dropping as $\log R/\sqrt l$.
Indeed, the white dotted lines in \cref{fig:tanhHeatmaps}(a, b) trace out such a level curve and it remarkably pinpoints the largest $\sigma_w$ that gives the optimal test set accuracy for each depth $L$.
Why isn't the best initialization given by $R = 1 \iff \sigma_w = 0$?
We believe that when $L$ and/or $\sigma_w$ is small, gradient dynamics no longer dominates the initialization quality because it has ``less room to explode,'' and expressivity issues start to dampen the test time performance.
%Similarly, the second row of \cref{fig:expresults} rescales the x-axis to $\sigma_w \sqrt l$, and it is easy to see that every model achieves its best performance with $\sigma_w \sqrt l \approx 4$ for the RRN and around 4.2 for the FRN.
%This suggests the following procedure for finding optimal initialization parameters for tanh-residual networks: train an ensemble of shallow networks quickly with depth $L'$, and find $\sigma_w'$ with the best performance.
%Compute $C := \sigma_w' \sqrt{L'}$.
%When training deep networks with depth $L > L'$, initialize with $\sigma_w' := C/\sqrt{L}.$

\paragraph{Tanh, vary $\sigma_a^2/\sigma_v^2$.}
As suggested in the analysis of \cref{eqn:tanhGradEst}, the ratio $\rho^2 = \sigma_a^2/\sigma_v^2$ determines the fixed point $\ee^*$ and its convergence rate by itself while also contributes to the rate of gradient explosion in tanh FRN.
We seek to isolate its effect on forward dynamics by varying $\sigma_v$ with $\rho$ such that $\sigma_v/\sqrt{1 + \rho^2}$ is kept constant, so that the leading term of the log gradient ratio is kept approximately equal for each $L$ and $\rho$.
\cref{fig:tanhHeatmaps}(c) shows the test accuracies of a grid of tanh FRN initialized with such an ensemble of $\sigma_\bullet$s.
What stands out the most is that performance is maximized essentially around a fixed value of $L$ regardless of $\rho$, which shows that indeed gradient dynamics determines the initialization quality in tanh resnets.
There is also a minor increase in performance with increasing $\rho$ regardless of $L$; this is counterintuitive as increasing $\rho$ means ``decreasing expressivity.''
It is currently not clear what accounts for this effect.

\paragraph{ReLU, vary $\sigma_w$}
We train a grid of ReLU FRN on MNIST, varying $\sigma_w^2 \in [0, 1.5]$ while fixing $\sigma_v^2 = 1, \sigma_a^2 = \sigma_b^2 = \f 1 2$.
The resulting test set accuracies are shown in \cref{fig:tanhHeatmaps}(d).
The dark upper region signifies failure of training caused by numerical issues with exploding activation and gradient norms: This corresponds to the region where $\p \pp L$, which is a measure of the mean magnitude of an neuronal activation in layer $L$, becomes too big.
We see that the best test accuracies are given by depths just below where these numerical issues occur.
However, if we were to predict that the optimal init is the one minimizing $\p \chi 0 / \p \chi L \ge 1$, then we would be wrong --- in fact it is exactly the opposite.
In this case, the dynamics of $\p \mfs l, \p \pp l$, and $\p \chi 0 / \p \chi l$ are approximately the same (all $\exp(\Theta(l))$ with the same hidden constants), and optimal performance corresponds to the highest $\p \mfs L$, $\p \pp L$, and $\p \chi 0 / \p \chi L$ without running into \texttt{inf}s.
%Our theory predicts (the ordering of) accuracies very well under the assumption that networks with similar levels of gradient explosion should perform similarly:
%The red lines give contours of the gradient ratio $\p \daleth 0 / \p \daleth L$ (which is also the activation norm ratio $\p \pp L / \p \pp 0$), and they track the level sets of accuracies.
%It seems that a ratio of $\exp(11.25)$ is optimal.

\paragraph{$\alpha$-ReLU, vary $\alpha$.}
We similarly trained a grid of $\alpha$-ReLU FRN on MNIST, varying only $\alpha$ and the depth, fixing all $\sigma_\bullet$.
\cref{fig:tanhHeatmaps}(e) shows their test accuracies.
We see similar behavior to ReLU, where when the net is too deep, numerical issues doom the training (black upper right corner), but the best performance is given by $L$ just below where this problem occurs.
In this case, if we were to predict optimality based on minimizing gradient explosion, we would be again wrong, and furthermore, the contour plot of $\p \chi 0 / \p \chi L$ (white dashed line) now gives no information at all on the test set accuracy.
In contrast, the contours for $\p \mfs l$ succeeds remarkably well at this prediction (yellow/green lines).%
\footnote{the contour for $\p \pp l$ is similar, but its slopes are slightly off from the heatmap contours.}
By interpolation, this suggests that indeed in the ReLU case, it is expressivity, not trainability, which determines performance at test time.

In all of our experiments, we did not find $\ee$ dynamics to be predictive of neural network performance.
%Our $\alpha$-ReLU experiment suggests that we should replace $\ee = \gamma / \pp$ with $\pp - \gamma$ as the central expressivity quantity.

\section{Conclusion}
In this paper, we have extended the mean field formalism developed by \cite{poole_exponential_2016,raghu_expressive_2016,schoenholz_deep_2017} to residual networks, a class of models closer to practice than classical feedforward neural networks as were investigated earlier.
We proved and verified that in both the forward and backward passes, most of the residual networks discussed here do not collapse their input space geometry or the gradient information exponentially.
We found our theory incredibly predictive of test time performance despite saying nothing about the dynamics of training.
%We in particular found that random tanh residual networks seem to perform best when the product $\sigma_w \sqrt{l}$ hovers around a constant value.
%This suggests that, beyond enriching our understanding of how residual networks work, and how they perform so well on the field, our theory has the potential to be applied in practice to develop better initialization schemes.
In addition, we overwhelmingly find, through theory and experiments, that an optimal initialization scheme must take into account the depth of the residual network.
The reason that Xavier \cite{glorot_understanding_2010} or He \cite{he_delving_2015} scheme are not the best for residual networks is in fact not that their statistical assumptions are fragile --- theirs are similar to our mean field theoretic assumptions, and they hold up in experiments for large width --- but rather that their structural assumptions on the network break very badly on residual nets.

\paragraph{Open Problems.}
Our work thus have shown that optimality of initialization schemes can be very unstable with respect to architecture.
We hope this work will form a foundation toward a mathematically grounded initialization scheme for state-of-the-art architectures like the original He et al. residual network.
To do so, there are still two major components left to study out of the following three:
\begin{enumerate*}
	\item Residual/skip connection
	\item Batchnorm
	\item Convolutional layers.
\end{enumerate*}
Recurrent architectures and attention mechanisms are also still mostly unexplored in terms of mean field theory.
Furthermore, many theoretical questions still yet to be resolved; the most important with regard to mean field theory is:
why can we make \cref{ass:symAct,ass:gradInd} and still be able to make accurate predictions?
We hope to make progress on these problems in the future and encourage readers to take part in this effort.

%%%%%%%%%%%%%%%%%%%%%%%%%%%%%
\newpage

\section*{Acknowledgments}

Thanks to Jeffrey Ling for early exploration experiments and help with the initial draft.
Thanks to Felix Wong for offering his wisdom and experience working in statistical physics.
\bibliographystyle{plainnat}
\bibliography{neural_dynamics}
\theendnotes
%%%%%%%%%%%%%%%%%%%%%%%%%%%%%
\newpage
\appendix
\renewcommand\thefigure{\thesection.\arabic{figure}}
\setcounter{table}{0}
\renewcommand*{\thetable}{\thesection.\arabic{table}}

\setcounter{figure}{0}
\appendixpage

\section{Additional Figures}
In figures appearing in the appendix, $\olddaleth$ means $\chi$ (due to legacy reasons).

\begin{figure}[h!]
\centering
	\begin{tabular}{m{.1\linewidth}m{0.6\linewidth}}
	$\alpha=1$ & \includegraphics[height=.08\textheight]{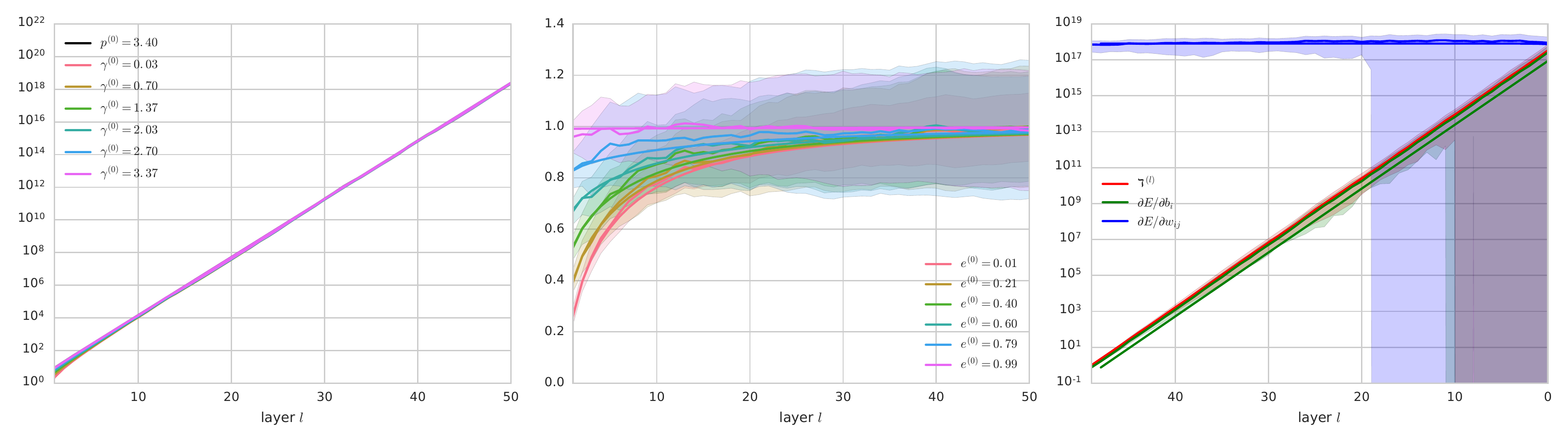}\\

	$\alpha=.9$ & \includegraphics[height=.08\textheight]{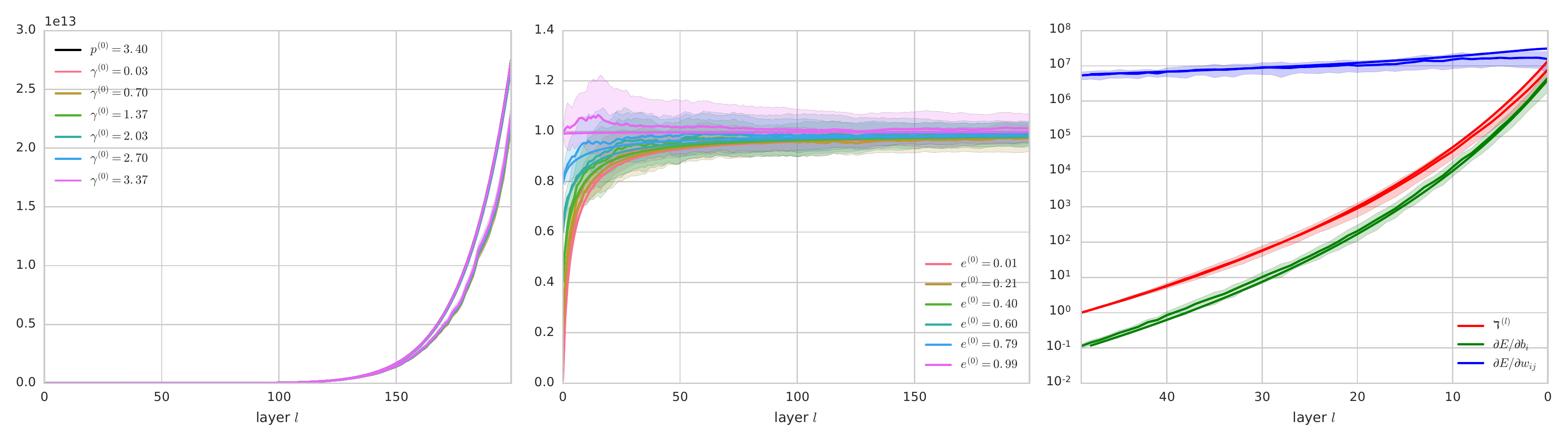}\\

	$\alpha=.8$ & \includegraphics[height=.08\textheight]{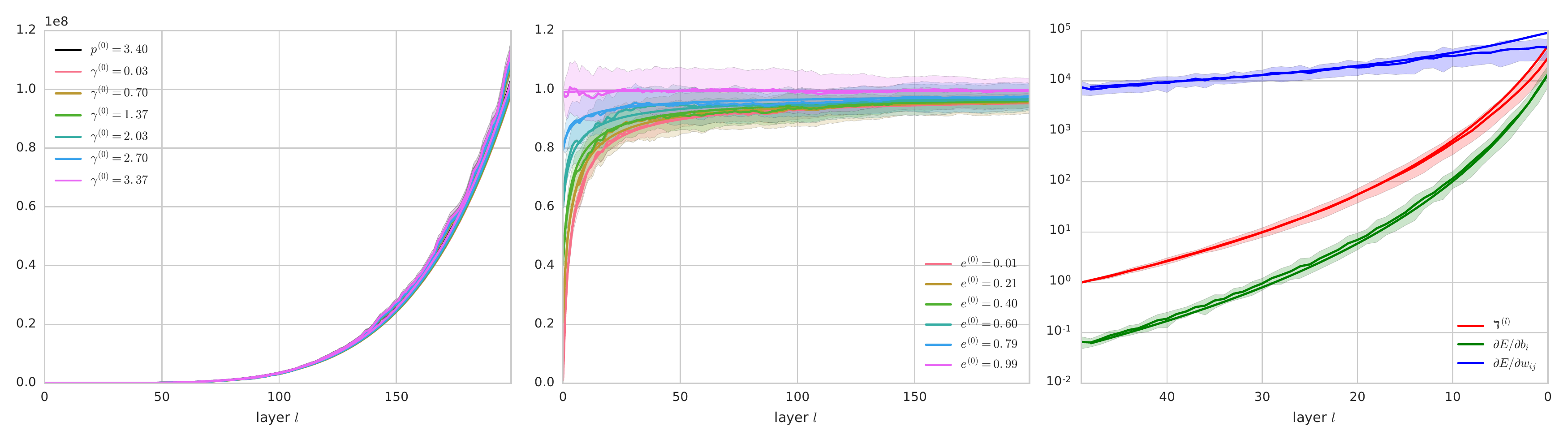}\\

	$\alpha=.7$ & \includegraphics[height=.08\textheight]{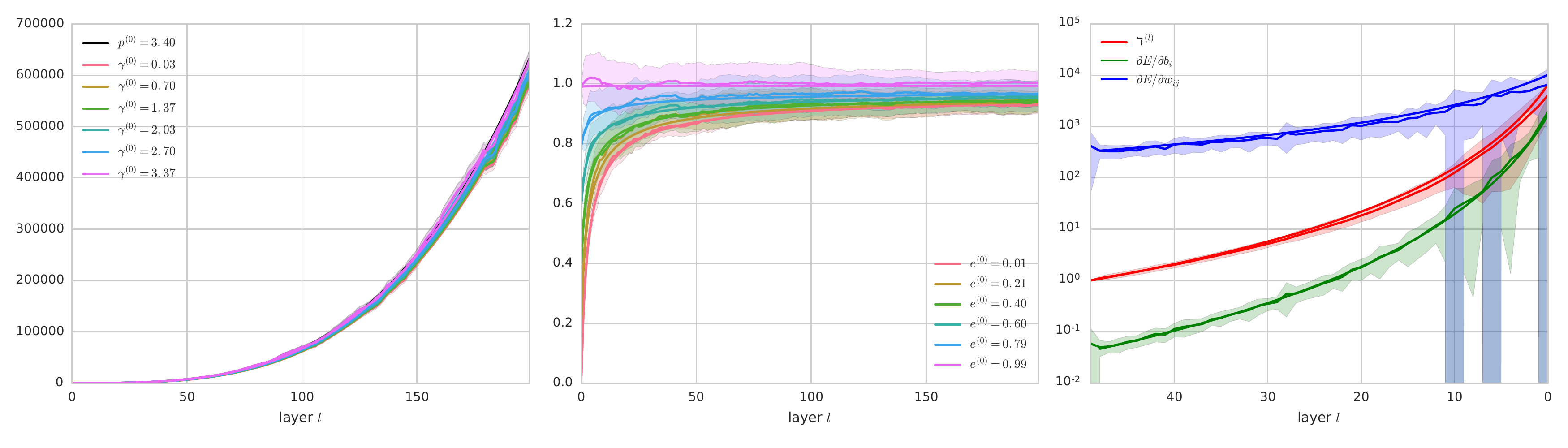}\\

	$\alpha=.6$ & \includegraphics[height=.08\textheight]{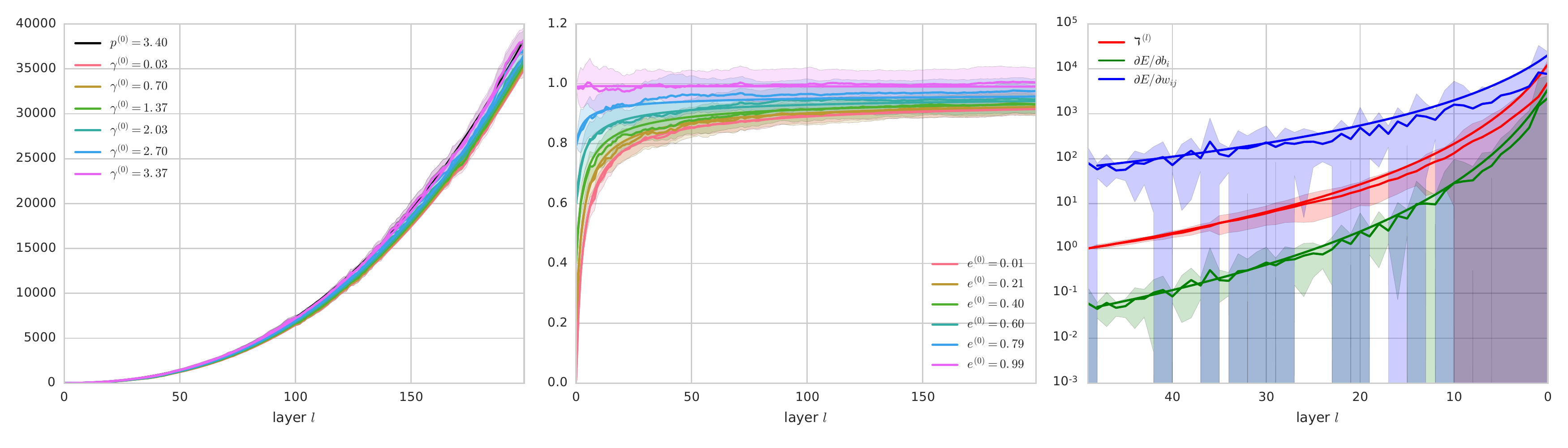}\\

	$\alpha=.55$ & \includegraphics[height=.08\textheight]{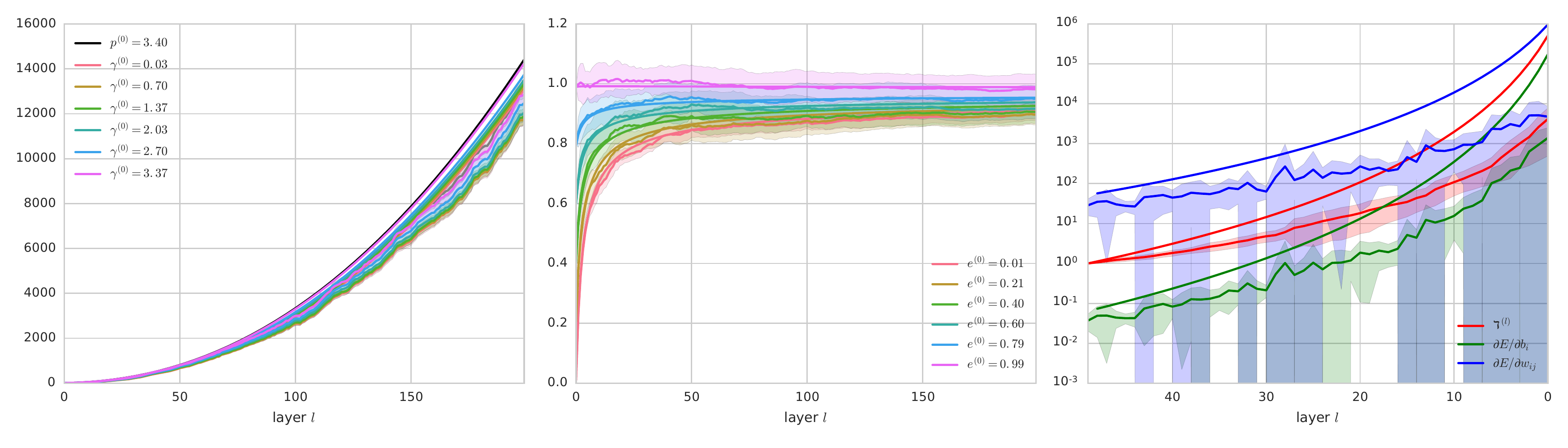}\\

	$\alpha=.51$ & \includegraphics[height=.08\textheight]{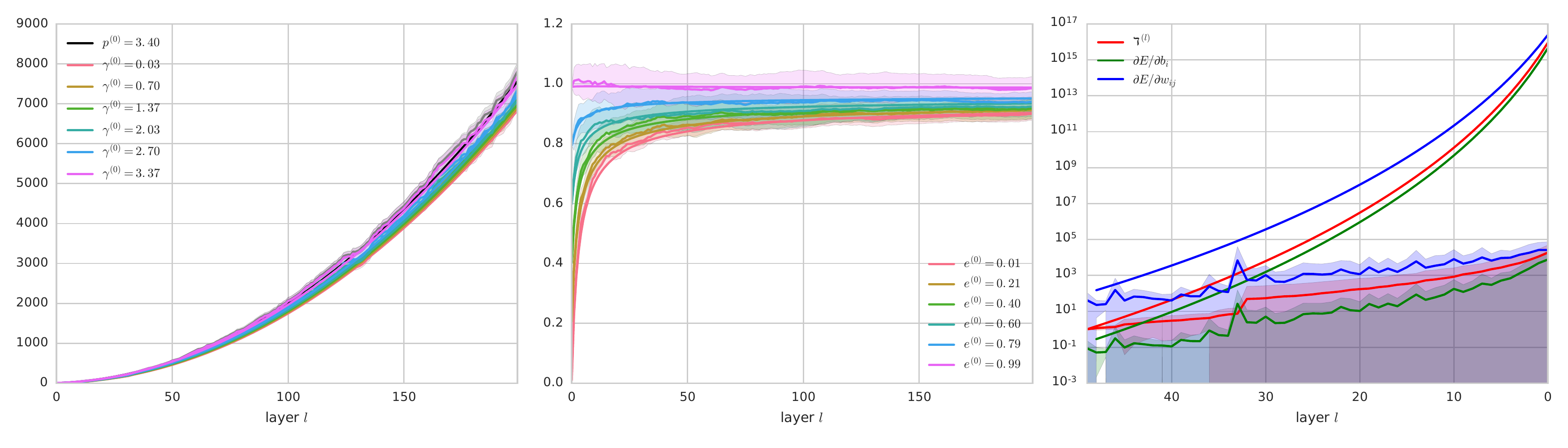}
	\end{tabular}

\caption{Empirical vs theoretical dynamics for $\p \pp l, \p \ee l$, and different gradient quantities for $\alpha$-ReLU, with format similar to \cref{fig:theory_tracks_pratice}.
	We refer to each figure on each row from left to right as (a), (b), and (c).
Note that in the $\alpha=1$ case, figure (a) ($\p \pp l$ and $\p \gamma l$ for different initial values) has log scale y-axis and (a) and (b) have x-axis ranging from 1 to 50, while for other $\alpha$, (a) has normal y-axis and (a) and (b) have x-axis ranging from 1 to 200.
We do so because the norm of the activation vector in a typical ReLU resnet blows up into \texttt{NaN} at around layer 90, while this is not a problem for $\alpha < 1$.
Our theoretical predictions track the average of empirical values closely for forward quantities $\p \pp l, \p \gamma l,$ and $\p \ee \l$ for all $\alpha$, but variance is extremely large for $\p \ee l$ at $\alpha = 1$; it also predicts the average gradient norm accurately for $\alpha = 1$ to $\alpha = .7$ (despite the fact that we should not expect so for $\alpha \le .75$ due to exploding variance (\cref{thm:dalethInfVarAlphaReLU})), although variance is large for $\alpha = 1$ at earlier layers (i.e. later layers w.r.t backpropagation).
However it {\it consistently and significantly overestimates} the average gradient norm for $\alpha = .6$ to $\alpha = .5$, where the variance is so large that one standard deviation below the mean results in negative values.
All plots are made with parameters $\sigma_v^2 = 1.5, \sigma_a^2 = .5, \sigma_w^2 = 1.69, \sigma_b^2 = .49$; only $\alpha$ is varied.
All figures exhibit smooth curves, which are theoretical estimates, and irregular curves with shades around them, which indicate empirical means and standard deviations (both of which taken in regular scale, not log scale).
For each $\alpha$, figures (a) and (b) are made with 20 runs of resnets of width 1000.
(c) is made with 25 runs of resnets of width 250.}
\label{fig:alphaReLUTheoryVsEmpirics}
\end{figure}
\begin{figure}[h]
\centering
\begin{tabular}{m{.1\linewidth}m{0.4\linewidth}}
	$\alpha=.9$ & \includegraphics[height=.08\textheight]{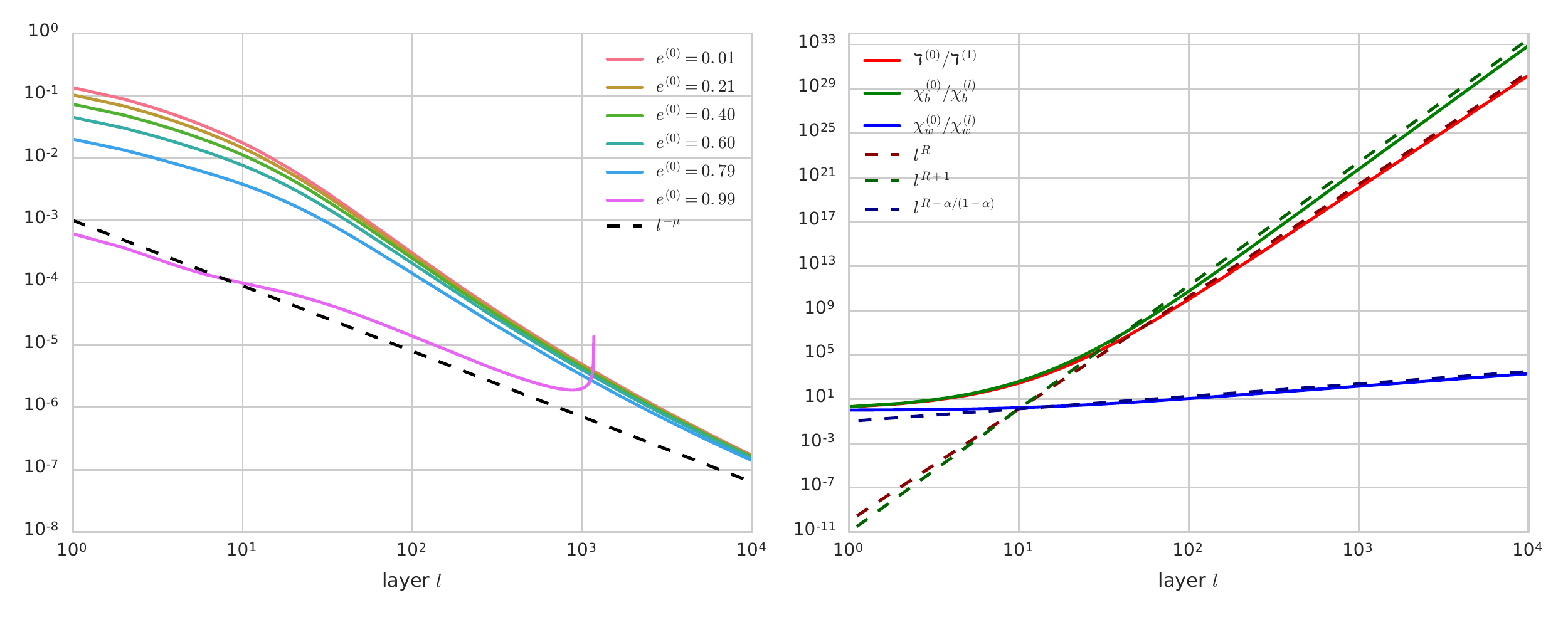}\\
	$\alpha=.8$ & \includegraphics[height=.08\textheight]{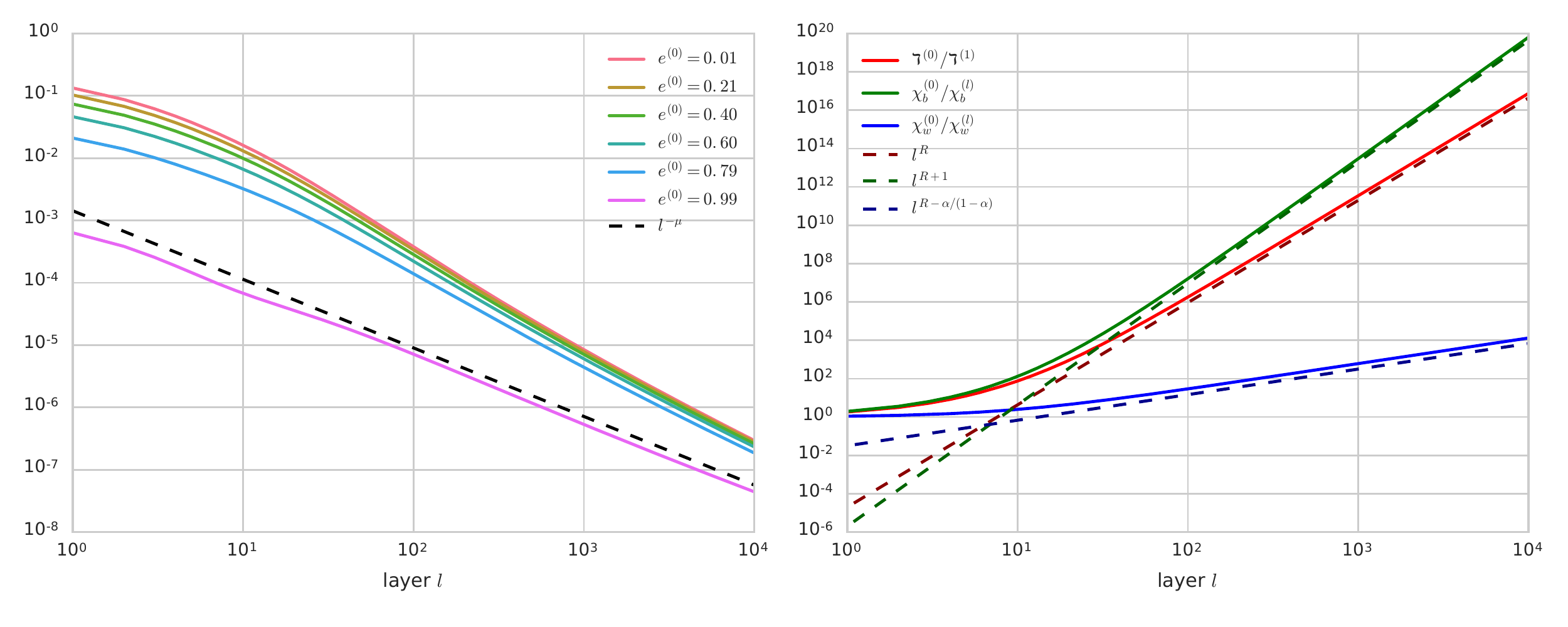}\\
	$\alpha=.7$ & \includegraphics[height=.08\textheight]{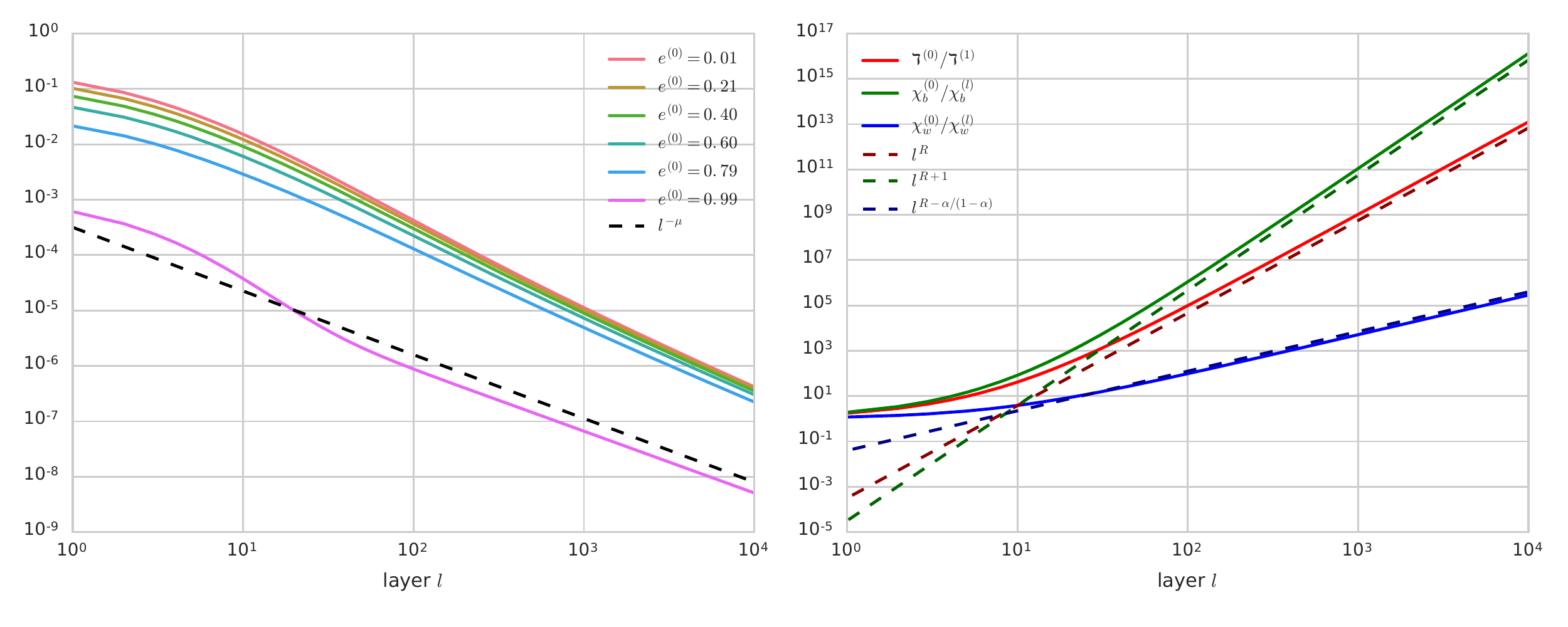}\\
	$\alpha=.6$ & \includegraphics[height=.08\textheight]{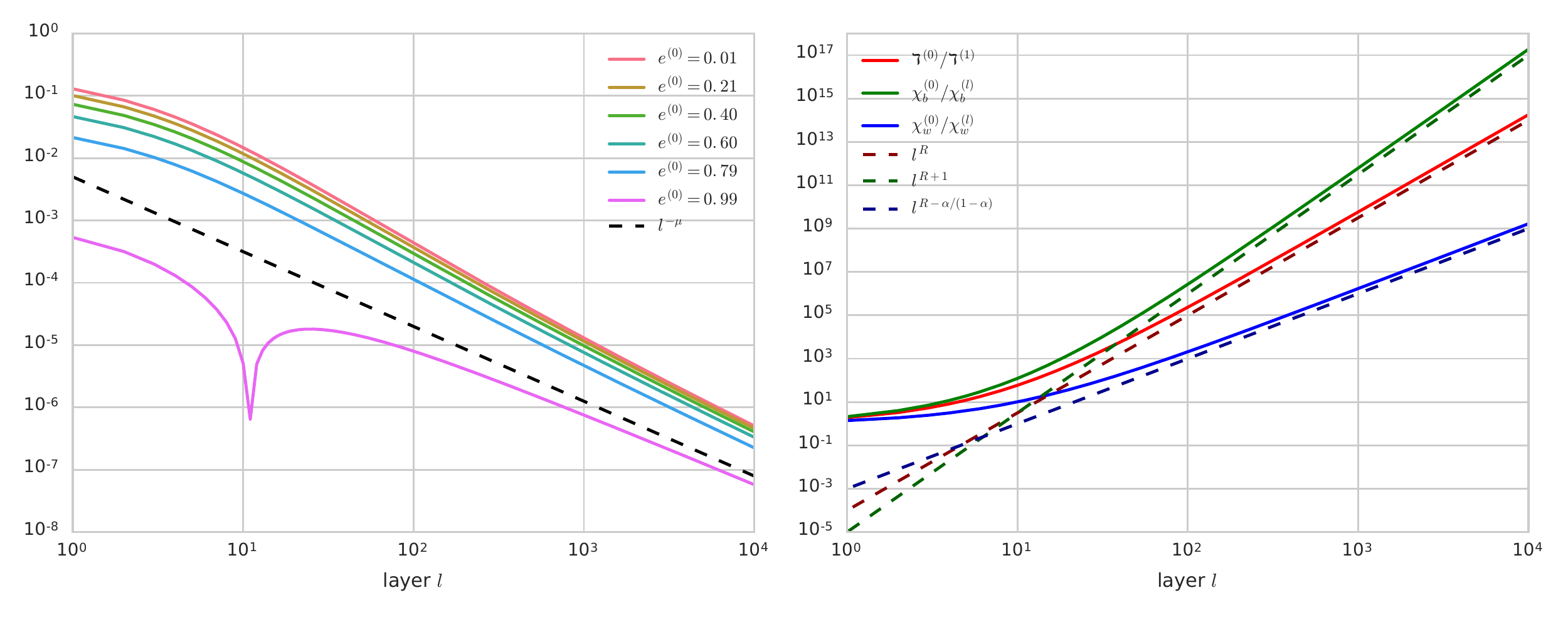}\\
	$\alpha=.55$ & \includegraphics[height=.08\textheight]{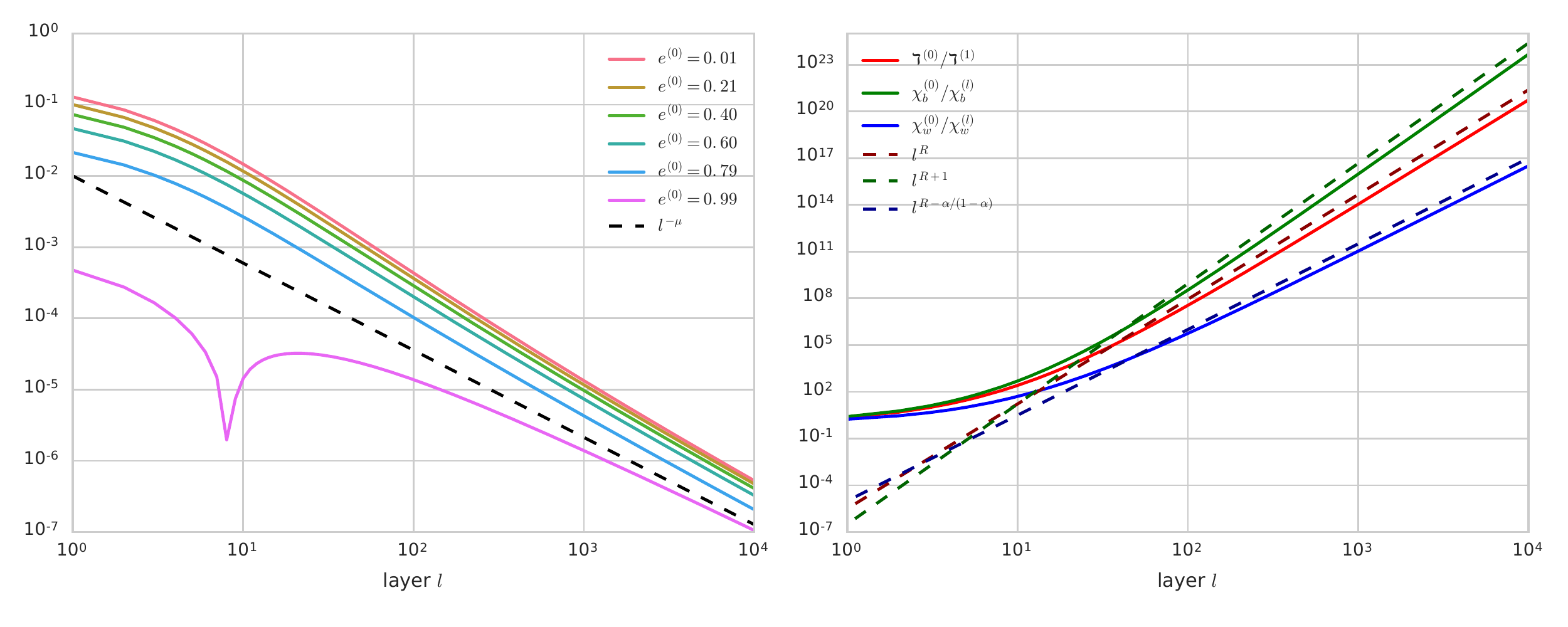}\\
	$\alpha=.51$ & \includegraphics[height=.08\textheight]{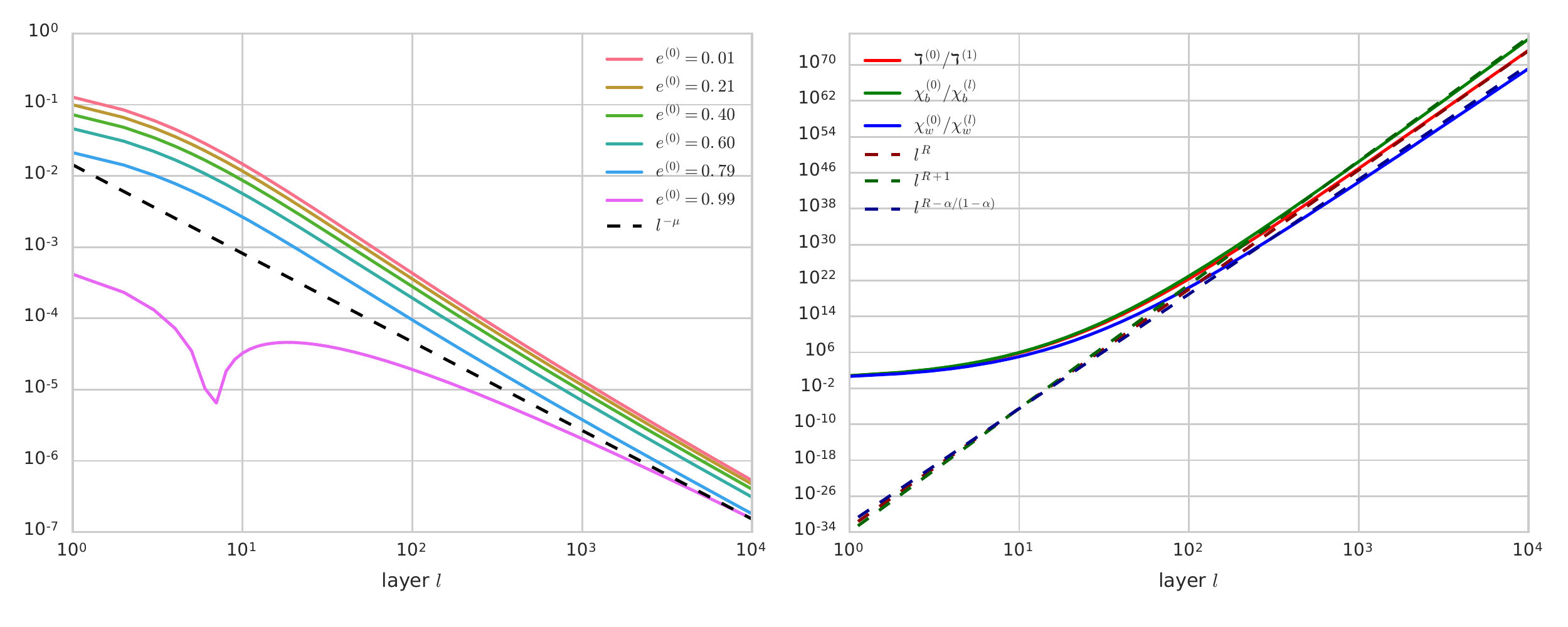}\\
\end{tabular}
\caption{We verify the exponents of the forward and backward dynamics for $\alpha$-ReLU FRN.
For each row, the figures are labeled (a) and (b) from left to right.
The format is the same as in \cref{fig:jjj_vs_id_main}.
All figures are in log-log scale.
{\bf (a)} We exhibit our theoretical dynamics of the cosine distance $\p \ee l$ based on the recurrences \cref{thm:fullResPQRec} and \cref{thm:full_res_l_g_recurr} for different initial conditions $\p \ee 0$.
We draw $|\p \ee l - \p \ee {l-1}|$ for each of these dynamics in colored solid lines.
We predict that each dynamic is $\TTheta(l^{-\mu})$, where $\mu = (1 - \dot\JJ_\alpha(\ee^*))/(1-\alpha)$, and the dashed line gives $l^{-\mu-1}$ (\cref{thm:alphaReLUeConvergence}), shifted vertically to better compare the slope in log scale (i.e. the exponent of the polynomial dynamics).
(See footnote \ref{footnote:plotDelta} for why we plot the dynamics this way).
We see that the our asymptotic prediction is very accurate for the sequence of $\p \ee l$ that starts with $\p \ee 0 = 0.99$, the closest to $\ee^*$ for each $\alpha$, while other lines only slowly converge to the same exponent (which is the slope in the log-log plot).
This is to be expected based on the proof of \cref{thm:alphaReLUeConvergence}.
For $\alpha = .9$, the $\p \ee 0 = .99$ line upticks at around $10^3$ and then turn into \texttt{NaN}s due to numerical instability.
{\bf (b)} Colored lines are $\p \bullet 0 / \p \bullet l$ for $\bullet = \daleth, \chi_b, \chi_w$ (we are not taking logs in addition to plotting in log-log scale like in \cref{fig:jjj_vs_id}).
The dashed lines are our asymptotic predictions for the dynamics with corresponding colors, based on \cref{thm:alphaReLUAllGradients}, again shifted appropriately to easily compare slope visually.
We see that for every alpha our asymptotic predictions are highly accurate.
For both (a) and (b), we did not show $\alpha = 1$ case as ReLU FRN runs into numerical issues quickly (i.e. with even for 100 layers) because of exponential explosions in $\p \pp l$ and $\p \daleth l$ as predicted by \cref{thm:pDynamicAlphaReLU,thm:dalethDynamicsAlphaReLU}, so we cannot expect to empirically verify the precise predicted asymptotics.
All plots are made with parameters $\sigma_v^2 = 1.5, \sigma_a^2 = .5, \sigma_w^2 = 1.69, \sigma_b^2 = .49$; only $\alpha$ is varied.
}
\label{fig:alphaReLUVerifyExponents}
\end{figure}

\begin{table}[t]
	\caption{Glossary of Symbols. ``Mean normalized'' is abbreviated ``m.n.''}
	\label{tab:glossary}
	\centering
	\begin{tabular}{lll}
		\toprule
		Symbol	&	Meaning	&	Ref\\
		\midrule
		$\sigma_\bullet$& standard deviation of trainable parameter $\bullet$	&\\
		$\p x l$		& activation vector/input vector					&\\
		$\p h l$		& hidden vector										&\\
		$N$				& width (same across all layers)					&\\
		$\p \pp l$		& m.n. squared length of activation vector $\p x l$	&	\ref{defn:length}\\
		$\p \qq l$		& m.n. squared length of hidden vector $\p h l$		&	\ref{defn:length}\\
		$\p \gamma l$	& m.n. dot product $\p x l \cdot \p x l{}'$			&	\ref{defn:corr}\\
		$\p \lambda l$	& m.n. dot product $\p h l \cdot \p h l{}'$			&	\ref{defn:corr}\\
		$\p \mfs l$		& m.n. squared distance $\|\p x l - \p x l {}'\|^2$	&	\ref{defn:corr}\\
		$\p \ee l$		& cosine distance $\p \gamma l / \sqrt{\p \pp l \p \pp l {}'}$	&	\ref{defn:corr}\\
		$\ee^*$			& limit value of $\p \ee l$ as $l \to \infty$	&	\\
		$\p \cc l$		& cosine distance $\p \lambda l / \sqrt{\p \qq l \p \qq l {}'}$ &	\ref{defn:corr}\\
		$\p \chi l$		& m.n. gradient squared norm w.r.t. $\p x l$	&	\ref{defn:grad}\\
		$\p {\chi_\bullet} l$	& m.n. gradient squared norm w.r.t. trainable parameter $\bullet$	&	\ref{defn:grad}\\
		$\phi$			& variable nonlinearity $\R \to \R$				&	\\
		$\psi_\alpha$	& $\alpha$-ReLU									&	\ref{defn:alphaReLU}\\
		$\Vt$			& variance integral transform					&	\ref{defn:integralTransform}\\
		$\Wt$			& covariance integral transform 				&	\ref{defn:integralTransform}\\
		$\delta^*$		& $\p \ee l$ converges like $\Theta(l^{-\delta^*})$ in tanh FRN	&	\ref{thm:eDynamicsFullResTanh}\\
		$\mathcal A$	& leading coeff of $\log \p \chi 0 / \p \chi L$ in tanh FRN	&	\ref{thm:dalethExpSqrtTanhFullRes}\\
		$R$				& $\log \p \chi 0 / \p \chi L \sim R \log L$ for $(\alpha<1)$-ReLU	&	\ref{thm:dalethDynamicsAlphaReLU}\\
		$\JJ_\alpha$	& kernel function of $\alpha$-ReLU				&	\ref{lemma:basicJalpha}\\
		\bottomrule
	\end{tabular}
\end{table}
%%%%%%%%%%%%%%%%%%%%%%%%%%%
%\section{Glossary of Symbols}

%%%%%%%%%%%%%%%%%%%%%%%%%%%
\section{A Listing of Main Theorems}
\subsection{Tanh}

\subsubsection{Reduced Residual Network}
\begin{restatable}{lemma}{pqrecurrence}\label{lemma:p_q_recurrence}
Suppose $\phi$ is antisymmetric.
Then in an RRN, $\pp$ and $\qq$ satisfy the recurrence
\begin{align*}
	\qq &= \sigma_w^2 \prv \pp + \sigma_b^2\\
	\pp &= \Vt \phi( \qq) + \prv \pp.
\end{align*}
\end{restatable}

\begin{restatable}{thm}{pqlinear}\label{thm:p_q_linear}
Suppose $\phi$ is tanh-like.
Assume RRN architecture.
\begin{itemize}
	\item If $\sigma_w = 0$, then $\p \pp l = l \Vt \phi( \sigma_b^2) + \p \pp 0$ and $\p \qq l = \sigma_b^2$.
	\item If $\sigma_w > 0$, $\lim_{l \to \infty} \p \pp l/ l = 1$ and $\lim_{l \to \infty} \p \qq l /(\sigma_w^2 l) = 1$.
	If $\phi = \tanh$, then we can obtain more terms of the asymptotic expansions:
	\begin{align*}
	\p \pp l &= l - 2 C \sigma_w^{-1} l^{1/2} - C^2 \sigma_w^{-2} \log l + O(1)\\
	\p \qq l &= \sigma_w^2 l - 2 C \sigma_w l^{1/2} - C^2 \log l + O(1)
	\end{align*}
	as $l \to \infty$, where $C = \sqrt{2 / \pi}$.
\end{itemize}
%Then in an RRN, 
% Then $\p \pp l = \Theta(l)$ and $\p \qq l = \Theta(l)$.
\end{restatable}

\begin{restatable}{thm}{lambdagammarecurrence}\label{thm:lambda_gamma_recurrence}
Suppose $\phi$ is antisymmetric.
Then in an RRN, $\lambda$ and $\gamma$ satisfy the recurrence
\begin{align*}
	\lambda &= \sigma_w^2 \prv \gamma + \sigma_b^2\\
	\gamma &= \Wt \phi(\qq, \lambda) + \prv \gamma.
\end{align*}
\end{restatable}

\begin{restatable}{thm}{edynamics}\label{thm:edynamics}
Suppose $\phi$ is a tanh-like nonlinearity in an RRN.
Assume $\p \ee 0 < 1$.
\begin{itemize}
	\item If $\sigma_w = 0$, then $\p \gamma l = l \Wt \phi( \sigma_b^2, \sigma_b^2) + \p \gamma 0 = l \Vt \phi( \sigma_b^2) + \p \gamma 0$ and $\p \lambda l = \sigma_b^2$, so that $\p \ee l \to 1$ and $1 - \p \ee l = \Theta(l^{-1})$.
	As a result, $\p \mfs l = \p \pp l (1 - \p \ee l) = \Theta(1).$
	\item If $\sigma_w > 0$, then $\p \gamma l = \TTheta(l^{\f 2\pi})$, and $\p \ee l \to 0$ like $\TTheta(l^{\f 2 \pi - 1})$.
	Thus $\p \mfs l = \Theta(\p \pp l) = \Theta(l).$
\end{itemize}
%In an RRN, if $\p \ee 0 \not = 1$, then 
\end{restatable}

\begin{restatable}{thm}{dalethRecReduced}\label{thm:dalethRecReduced}
For any nonlinearity $\phi$ in an RRN, under assumptions \cref{ass:symAct} and \cref{ass:gradInd}, whenever $\dot \phi^2(\zeta)$ has finite variance for Gaussian variable $\zeta$,
\begin{align*}
\prv \daleth &= (\sigma_w^2  \Vt \dot \phi( \qq) + 1)\daleth,&
\chi_b &= \daleth \Vt\dot \phi( \qq),&
\chi_w &= \daleth \Vt \dot \phi( \qq) \prv \pp. 
\end{align*}

\end{restatable}

\begin{restatable}{thm}{dalethExpSqrtTanh}\label{thm:dalethExpSqrtTanh}
For $\phi = \tanh$ in an RRN,
\begin{itemize}
	\item If $\sigma_w = 0$, $\p \daleth m = \p \daleth l$ for all $l, m$.
	\item If $\sigma_w > 0$, \begin{align*}
	\log(\p \daleth {m}/\p \daleth l) &= \mathcal A (\sqrt l - \sqrt m) + \mathcal B (\log l - \log m) + O(1)
	\end{align*}
	where $\mathcal A = \f 4 3 \sqrt{\f 2 \pi} \sigma_w$ and $\mathcal B = \f 4 {3\pi} - \sigma_w^2 \f 4 {9\pi}$.
\end{itemize}
\end{restatable}

% \cref{fig:residual_correlation_sim_vs_theory} shows that the theoretical predictions of $\ee$ are fairly accurate, and \cref{fig:simp_res_exponent_verification} shows that the exponent $\f 2 \pi - 1$ is correct for the decay of $\ee$.

\begin{restatable}{thm}{dalethExpSqrtTanhAllGrad}\label{thm:dalethExpSqrtTanhAllGrad}
Suppose $\phi = \tanh$.
Then in an RRN
\begin{itemize}
	\item If $\sigma_w = 0$, $\p \chi l _b = \p \daleth L \Vt \dot \phi( \sigma_b^2)$ and $\p \chi l _w = \p \daleth L \Vt \dot \phi( \sigma_b^2) ((l-1) \Vt \phi( \sigma_b^2) + \p \pp 0),$ where $L$ is the last layer.
	\item If $\sigma_w > 0$, \begin{align*}
	\log(\p\chi m _b / \p \chi l _b) &= \mathcal A (\sqrt l - \sqrt m) + \mathcal B_b (\log l - \log m) + O(1)\\
	\log(\p\chi m _w / \p \chi l _w) &= \mathcal A (\sqrt l - \sqrt m) + \mathcal B_w (\log l - \log m) + O(1)
	\end{align*}
	where $\mathcal A = \f 4 3 \sqrt{\f 2 \pi} \sigma_w$ (same as $\mathcal A$ in \cref{thm:dalethExpSqrtTanh}) and $\mathcal B_b = \mathcal B + \f 1 2, \mathcal B_w = \mathcal B - \f 1 2$, with $\mathcal B = \f 4 {3\pi} - \sigma_w^2 \f 4 {9\pi}$ (same as $\mathcal B$ in \cref{thm:dalethExpSqrtTanh}).
\end{itemize}
\end{restatable}
\subsubsection{Full Residual Network}

\begin{restatable}{thm}{fullResPQRec} \label{thm:fullResPQRec}
For any nonlinearity $\phi$ in an FRN,
\begin{align*}
\qq &= \sigma_w^2 \prv \pp + \sigma_b^2\\
\pp &= \sigma_v^2 \Vt \phi( \qq) + \sigma_a^2 + \prv \pp
\end{align*}
\end{restatable}

\begin{restatable}{thm}{pIsLinearTanh}\label{thm:pIsLinearTanh}
Suppose $\phi$ is tanh-like.
Assume the FRN architecture.
\begin{itemize}
	\item If $\sigma_w = 0$, then $\p \pp l = (\sigma_v^2 \Vt \phi( \sigma_b^2) + \sigma_a^2)l +\p \pp 0$, and $\p \qq l = \sigma_b^2$.
	\item If $\sigma_w > 0$, then $\p \pp l = b_0 l + b_1 l^{1/2} + b_2 \log l + O(1)$, where
	\begin{align*}
	b_0 &= \sigma_v^2 + \sigma_a^2\\
	b_1 &= \f{-2C \sigma_v^2 \sigma_w^{-1}}{\sqrt{\sigma_v^2 + \sigma_a^2}}\\
	b_2 &= \f{-C^2 \sigma_v^4 \sigma_w^{-2}}{(\sigma_v^2 + \sigma_a^2)^2}
	\end{align*}
	and $C = \sqrt{\f 2 \pi}$.
	Additionally, $\p \qq l = \sigma_w^2 b_0 l + \sigma_w^2 b_1 l^{1/2} + \sigma_w^2 b_2 \log l + O(1)$.
\end{itemize}
\end{restatable}

\begin{figure}[]
	\centering
	\includegraphics[height=.16\textheight]{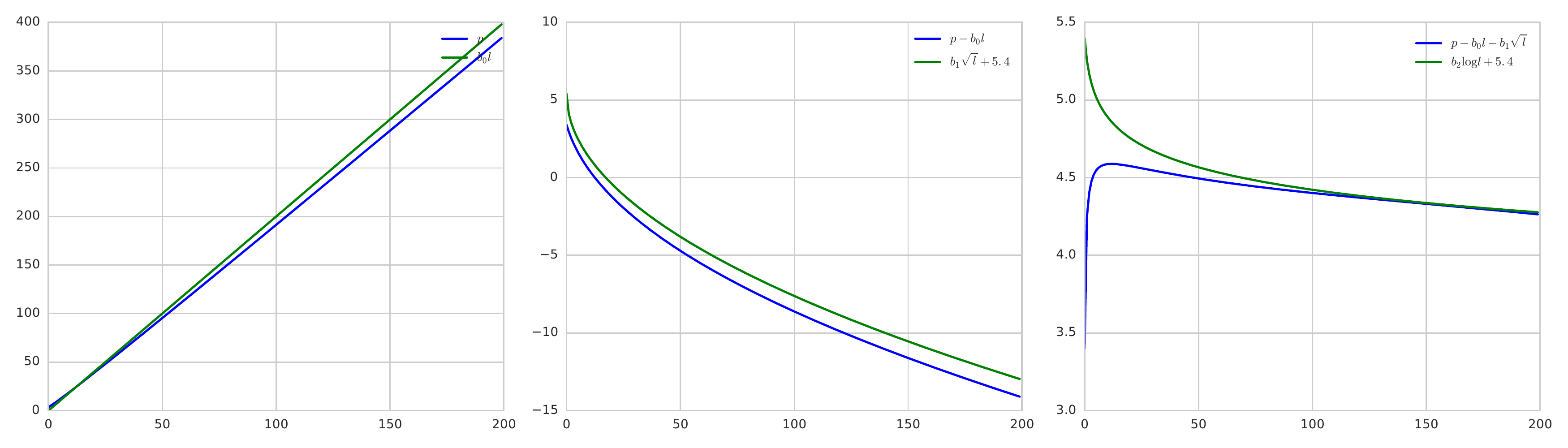}
	\caption{Empirical verification of \cref{thm:pIsLinearTanh}.}
\end{figure}

\begin{restatable}{thm}{LGRecFullRes}\label{thm:full_res_l_g_recurr}
For any nonlinearity $\phi$, in an FRN
\begin{align*}
\lambda &= \sigma_w^2 \prv \gamma + \sigma_b^2\\
\gamma &= \sigma_v^2 \Wt \phi(\qq, \lambda) + \sigma_a^2 + \prv \gamma
\end{align*}
\end{restatable}

\begin{restatable}{thm}{eDynamicsFullResTanh} \label{thm:eDynamicsFullResTanh}
	Assume $\phi = \tanh$ in an FRN.
	Suppose $\p \ee 0 < 1$.
	\begin{itemize}
		\item If $\sigma_w = 0$, then $\p \lambda l = \sigma_b^2$ and $\p \gamma l = l (\sigma_v^2 \Wt \phi( \sigma_b^2, \sigma_b^2) + \sigma_a^2) + \p \gamma 0 = l (\sigma_v^2 \Vt \phi( \sigma_b^2) + \sigma_a^2) + \p \gamma 0$.
		Thus $\p \ee l \to 1$ and $1 - \p \ee l = \Theta(l^{-1})$.
		As a result, $\p \mfs l = \p \pp l (1 - \p \ee l) = \Theta(1).$
		\item If $\sigma_w > 0$, then $\p \ee l$ converges to the unique fixed point $\ee^* \not = 1$ determined by the equation
		$$\ee^* = \f 1 {\sigma_v^2 + \sigma_a^2}[\sigma_v^2 \f 2 \pi \arcsin\lp \ee^* \rp + \sigma_a^2].$$
		Furthermore, $\p \ee l$ converges to $\ee^*$ polynomially:
		$|\p \ee l - \ee^*|$ is $\TTheta(l^{-\delta^*})$, where
		$$\delta^* := 1 - \f 2 \pi \f 1 {\sqrt{1 - (\ee^*)^2}} \f{\sigma_v^2 }{\sigma_v^2 + \sigma_a^2} \in [\f 2 \pi - 1, \f 1 2)$$
		Since $\ee^* < 1$, $\p \mfs l = \Theta(\p \pp l) = \Theta(l).$
	\end{itemize}
\end{restatable}

\begin{restatable}{thm}{dalethRecFull} \label{thm:dalethRecFull}
For any nonlinearity $\phi$ in an FRN, under assumptions \cref{ass:symAct} and \cref{ass:gradInd}, whenever $\dot \phi(\zeta)^2$ has finite variance for Gaussian variable $\zeta$,
\begin{align*}
\prv \daleth &= (\sigma_v^2\sigma_w^2  \Vt \dot \phi( \qq) + 1)\daleth,&
\chi_b &= \sigma_v^2\daleth \Vt\dot \phi( \qq),\\
\chi_w &= \sigma_v^2\daleth \Vt \dot \phi( \qq) \prv \pp,&
\chi_v &= \daleth\Vt \phi( \qq),&
\chi_a &= \daleth 
\end{align*}
\end{restatable}

\begin{restatable}{thm}{dalethExpSqrtTanhFullRes}\label{thm:dalethExpSqrtTanhFullRes}
Assume $\phi = \tanh$ in an FRN.
\begin{itemize}
	\item If $\sigma_w = 0$, $\p \daleth m = \p \daleth l$ for all $l, m$.
	\item If $\sigma_w > 0$, then for $l \ge m \ge 0,$
	$$\log(\p \daleth {m} / \p \daleth l) = \mathcal A (\sqrt l - \sqrt m) + \mathcal B (\log l - \log m) + O(1)$$
	where
	\begin{align*}
	\mathcal A &= \f 4 3 \sqrt{\f 2 \pi} \f{\sigma_v^2 \sigma_w}{\sqrt{\sigma_v^2 + \sigma_a^2}}\\
	\mathcal B &= \f 4 {9\pi}\f{ \sigma_v^4 }{\sigma_v^2 + \sigma_a^2}\lp \f 3 {\sigma_v^2 + \sigma_a^2} - \sigma_w^2\rp
	\end{align*}
\end{itemize}
\end{restatable}

\cref{fig:tanhasymptoticsgrid} shows empirical verification of the asymptotic expansion of $\daleth$ for various values of $\sigma_\bullet$s.
\begin{figure}
	\centering
	\includegraphics[height=.3\textheight]{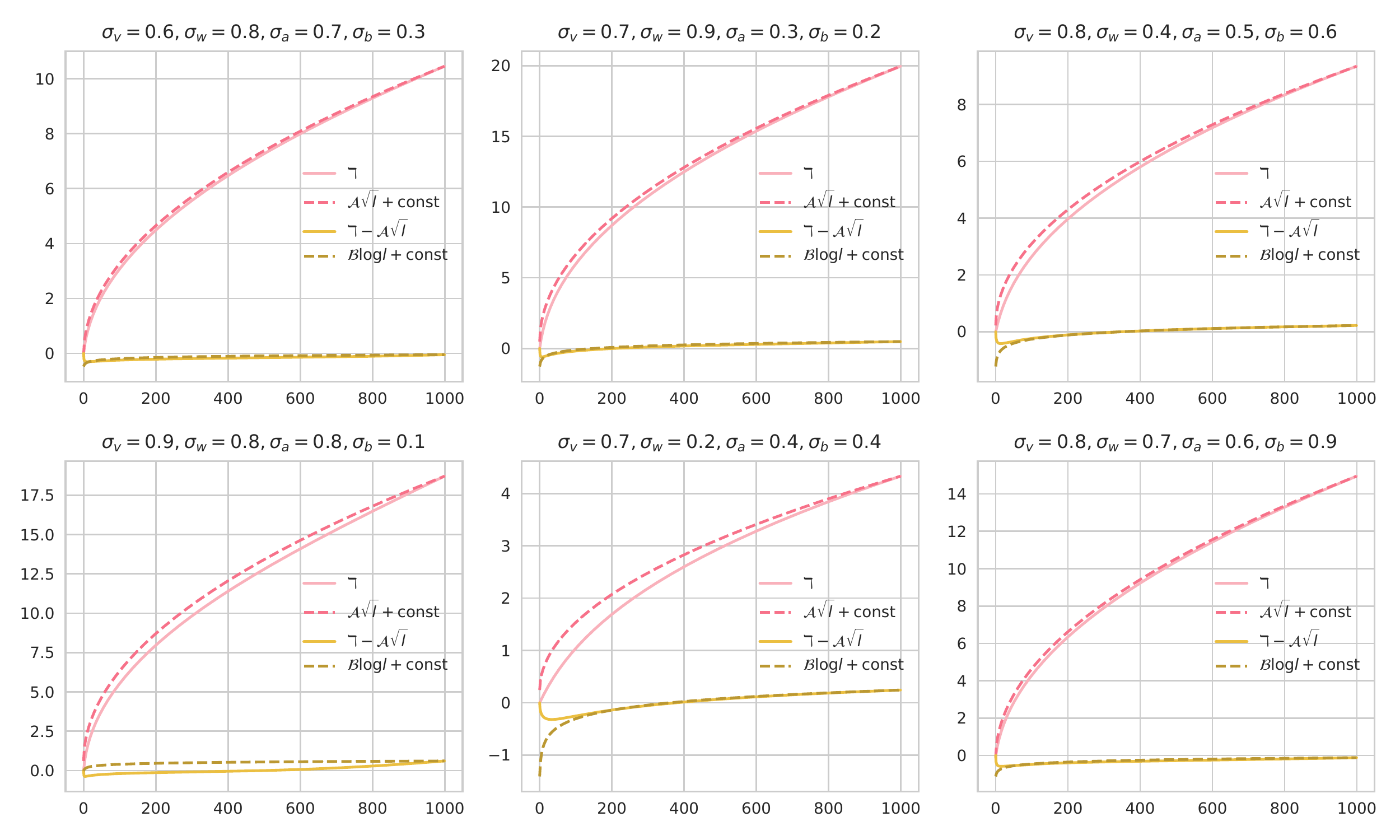}
	\caption{Empirical verification of the asymptotic expansion of $\daleth$ for various values of $\sigma_\bullet$s.
	Note that we have chosen all small values for $\sigma_\bullet$s.
	For larger values, the constant term in \cref{thm:dalethExpSqrtTanhFullRes} begins to dominate (primarily because of the expansion $\log(1+x) = x + \Theta(x^2)$ has large $\Theta$ term when $x$ is large), and $\daleth$ behaves more like $\exp(\Theta(l))$ up to depth 1000.}
	\label{fig:tanhasymptoticsgrid}
\end{figure}

\begin{restatable}{thm}{dalethExpSqrtTanhFullResAllGrad}\label{thm:dalethExpSqrtTanhFullResAllGrad}
Suppose $\phi = \tanh$ in an FRN.
\begin{itemize}
	\item If $\sigma_w = 0$, then
	\begin{align*}
	\p \chi l _b &= \sigma_v^2 \p \daleth L \Vt \dot \phi( \sigma_b^2)\\
	\p \chi l _w &= \sigma_v^2 \p \daleth L \Vt \dot \phi( \sigma_b^2) ( (\sigma_v^2 \Vt \phi( \sigma_b^2) + \sigma_a^2)(l-1) +\p \pp 0)\\
	\p \chi l _v &= \p \daleth L \Vt \phi( \sigma_b^2)\\
	\p \chi l _a &= \p \daleth L.
	\end{align*}
	\item If $\sigma_w > 0$, then for $l \ge m \ge 0,$
	
	\begin{align*}
	\log(\p\chi m _b / \p \chi l _b) &= \mathcal A (\sqrt l - \sqrt m) + \mathcal B_b (\log l - \log m) + O(1)\\
	\log(\p\chi m _w / \p \chi l _w) &= \mathcal A (\sqrt l - \sqrt m) + \mathcal B_w (\log l - \log m) + O(1)\\
	\log(\p\chi m _a / \p \chi l _a) &= \mathcal A (\sqrt l - \sqrt m) + \mathcal B (\log l - \log m) + O(1)\\
	\log(\p\chi m _v / \p \chi l _v) &= \mathcal A (\sqrt l - \sqrt m) + \mathcal B (\log l - \log m) + O(1)
	\end{align*}
	where $\mathcal A = \f 4 3 \sqrt{\f 2 \pi} \f{\sigma_v^2 \sigma_w}{\sqrt{\sigma_v^2 + \sigma_a^2}}$ and $\mathcal B = \f 4 {9\pi}\f{ \sigma_v^4 }{\sigma_v^2 + \sigma_a^2}\lp \f 3 {\sigma_v^2 + \sigma_a^2} - \sigma_w^2\rp$ are as in \cref{thm:dalethExpSqrtTanhFullRes} and $\mathcal B_b = \mathcal B + \f 1 2$ and $\mathcal B_w = \mathcal B - \f 1 2$.
	
\end{itemize}
\end{restatable}

\subsection{$\alpha$-ReLU}
\begin{restatable}{lemma}{VtPsiAlpha}\label{lemma:VtPsiAlpha}
If $\alpha > -\f 1 2$, then
$\Vt\psi_\alpha( q) = \cV_\alpha q^{\alpha}$, where $\cV_\alpha = \f 1 {\sqrt \pi} 2^{\alpha - 1}  \Gamma \left(\alpha+ \f 1 2\right)$.
\end{restatable}

Note that if $\alpha \le - \f 1 2$, then $\Vt \psi_\alpha( q)$ is not defined (its defining integral does not converge).

\subsubsection{Full Residual Network}
By \cref{thm:fullResPQRec} and \cref{lemma:VtPsiAlpha}, we have the length recurrences
\begin{align*}
\qq &= \sigma_w^2 \prv \pp + \sigma_b^2\\
\pp &= \sigma_v^2 \cV_\alpha \qq^\alpha + \sigma_a^2 + \prv \pp
\end{align*}

\begin{restatable}{thm}{pDynamicAlphaReLU}\label{thm:pDynamicAlphaReLU}
Suppose we have the nonlinearity $\phi = \psi_\alpha$.
The in an FRN:
If $\alpha = 1$, then $\p \pp l = \Theta((1 + \sigma_v^2 \sigma_w^2/2)^l)$, with the hidden constant depending on the initial condition.
If $0 < \alpha < 1$, then $\p \pp l = \Theta(l^{\f 1 {1- \alpha}})$.
More precisely,
$\lim_{l \to \infty} \pp/ l^{\f 1 {1-\alpha}} = [\sigma_v^2 \sigma_w^{2\alpha} \cV_\alpha (1 - \alpha)]^{\f 1 {1-\alpha}}$.
\end{restatable}
\cref{fig:reluverifypasymptotics} empirically verifies the asymptotics for $\alpha=1$ for various $\sigma_v$ and $\sigma_w$.
\begin{figure}
	\centering
	\includegraphics[height=.2\textheight]{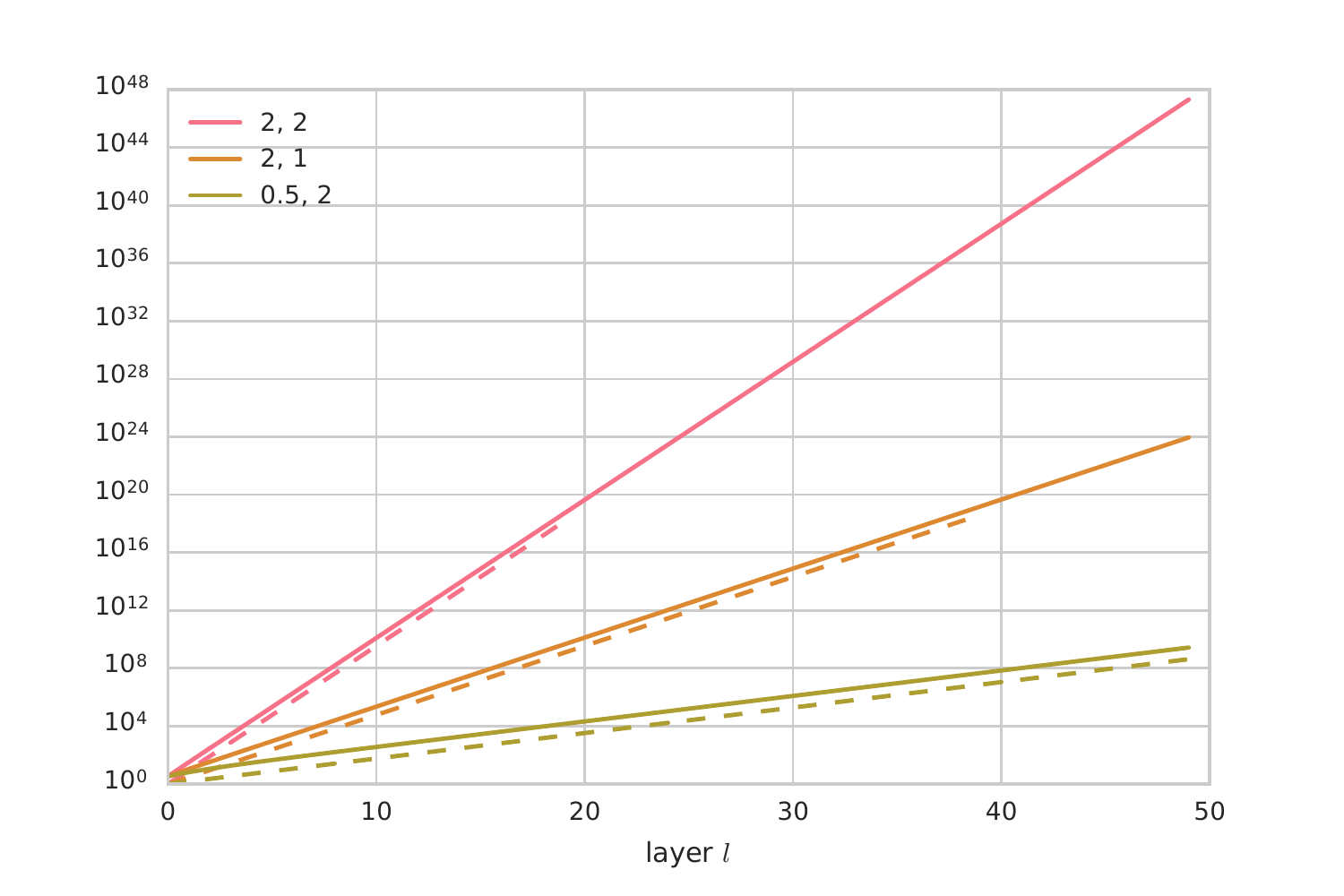}
	\caption{Verification of the exponential asymptotics of $\p \pp l$ when $\alpha=1$.
	The lines of each color correspond to different $(\sigma_w, \sigma_v)$ pairs, which are given in the legend.
	The solid lines are given by the recurrences \cref{thm:fullResPQRec}, and the dashed lines are given by our asymptotics $(1+\sigma_v^2\sigma_w^2/2)^l$ (\cref{thm:pDynamicAlphaReLU}).
	Note that the y-axis is in log-scale.}
	\label{fig:reluverifypasymptotics}
\end{figure}

Similarly, by \cref{thm:full_res_l_g_recurr}, if $\qq = \qq'$, then
\begin{align*}
\lambda &= \sigma_w^2 \prv \gamma + \sigma_b^2\\
\gamma &= \sigma_v^2 \qq^\alpha \Wt \psi_\alpha( 1, \cc) + \sigma_a^2 + \prv \gamma
\end{align*}

% TODO: verify this with a plot
\begin{restatable}{thm}{ReLUSquaredConvergence}\label{thm:ReLUSquaredConvergence}
Suppose $\phi = \psi_1$.
Then in an FRN, $\p \ee l \to 1$ and $1 - \p \ee l \sim [\f 1 4 \sigma_v^2 \sigma_w^2 \inv B U l]^{-2}$ for $B = 1 + \sigma_v^2 \sigma_w^2/2$ and $U = \f {2\sqrt 2}{3\pi}$.
As a result, $\p \mfs l = (1 - \p \ee \l) \p \pp l = \Theta(l^{-2}\exp(\Theta(l))) = \exp(\Theta(l)).$
\end{restatable}

\begin{restatable}{thm}{alphaReLUeConvergence} \label{thm:alphaReLUeConvergence}
Suppose $\phi = \psi_\alpha$ for $0 < \alpha < 1$ in an FRN.
%Let $\ee^* < 1$ be the fixed point of $\JJ_\alpha$ on $(0, 1)$. 
Then $\ee$ converges to the unique nonunit fixed point $\ee^*$ of $\JJ_\alpha$, and $|\ee^* - \p \ee l|$ is $\TTheta(l^{-\mu})$, where $\mu = (1-\dot\JJ_\alpha(\ee^*))/(1-\alpha)$.
Additionally, $\p \mfs l = \Theta(\p \pp l) = \Theta(l^{1/(1-\alpha)}).$
% More precisely, $1 - \p \ee l \sim \f {1- \alpha} \beth (\log l)^{-1}$, where $\beth = \f 1 2\dot\JJ_\alpha(1)$.
\end{restatable}
\cref{fig:6reluverifyestar} verifies empirically that $\ee^*$ is indeed the fixed point of $\p \ee l$.
\cref{fig:alphaReLUVerifyExponents} verifies empirically the convergence rate $l^{-\mu}$.
\cref{fig:MuPlots} plots $\dot \JJ_\alpha(\ee^*)$ and $\mu$ versus $\alpha$.
It certainly looks like $\mu = \f 1 2 (1 - \alpha)$, but we have no proof for it.
Based on this conjecture, we see there is a ``discontinuity'' of $\mu$ at $\alpha = 1$: $\mu \to 0$ as $\alpha \to 1$, but for $\alpha = 1$, the actual convergence dynamics has exponent $-2$ by \cref{thm:ReLUSquaredConvergence}.

\begin{figure}
	\centering
	\includegraphics[height=.2\textheight]{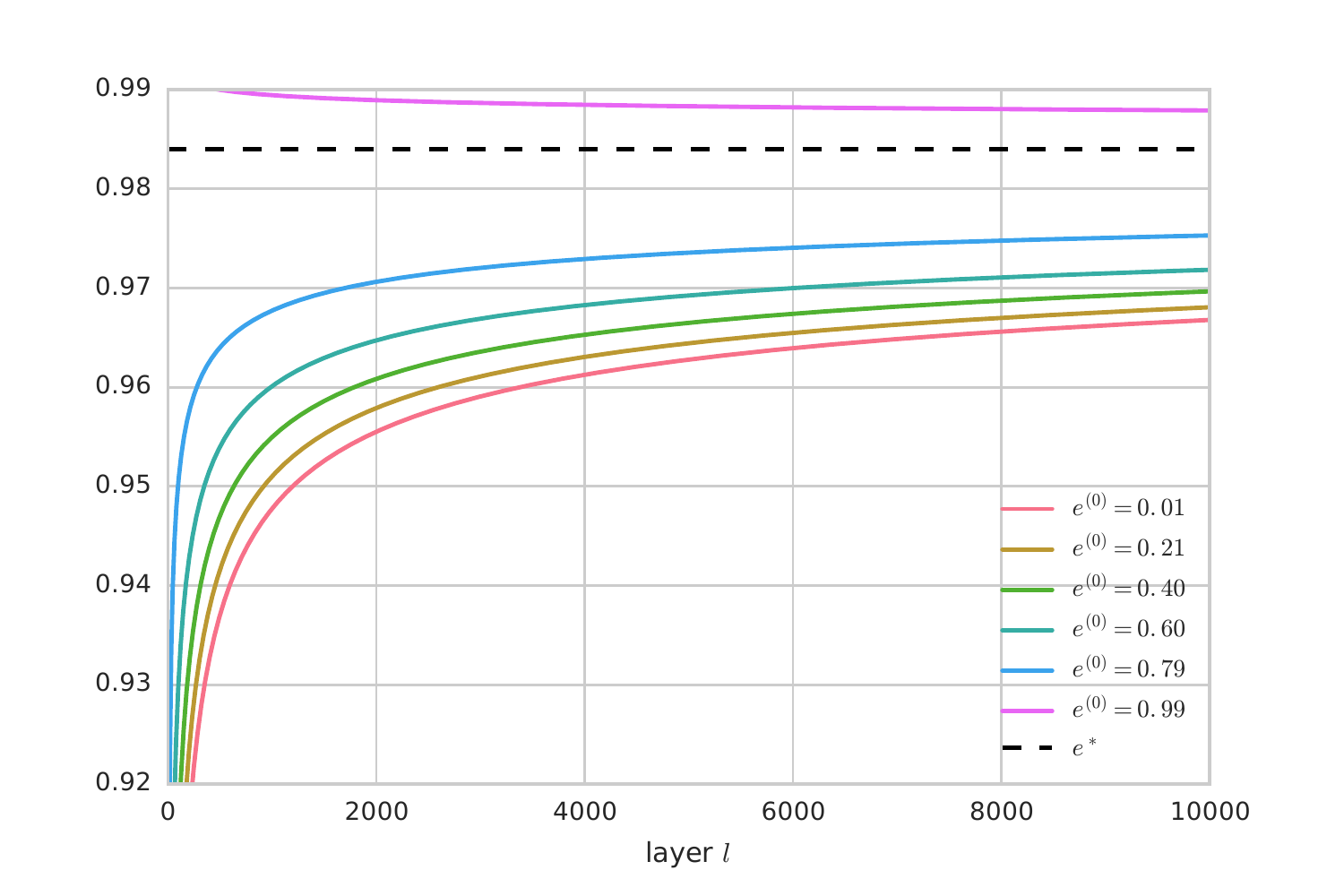}
	\caption{Verification of fixed point $\ee^*$ in \cref{thm:alphaReLUeConvergence} for $\alpha = .6$.
	Different colors correspond to different initial conditions $\p \ee 0$, and the dashed line gives the fixed point.}
	\label{fig:6reluverifyestar}
\end{figure}
\begin{figure}
\includegraphics[height=.2\textheight]{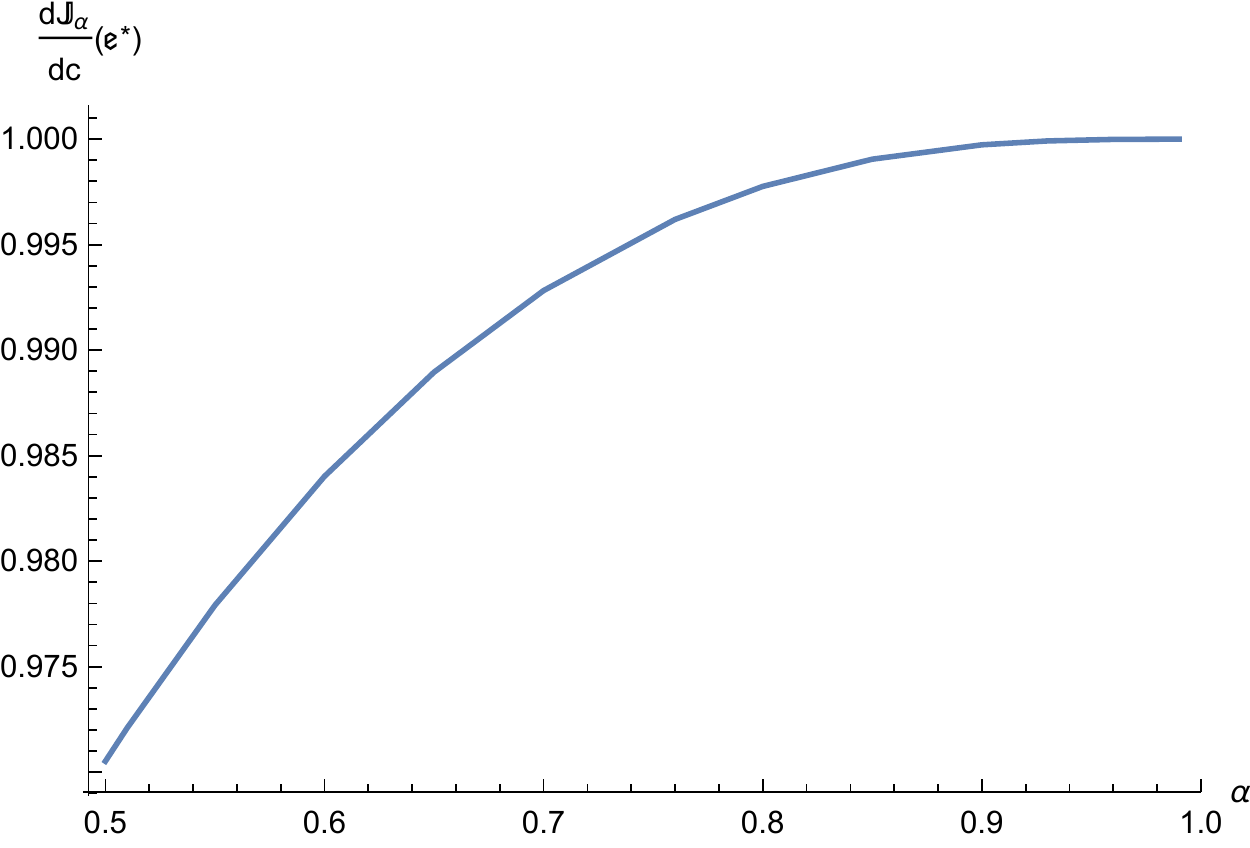}
\includegraphics[height=.2\textheight]{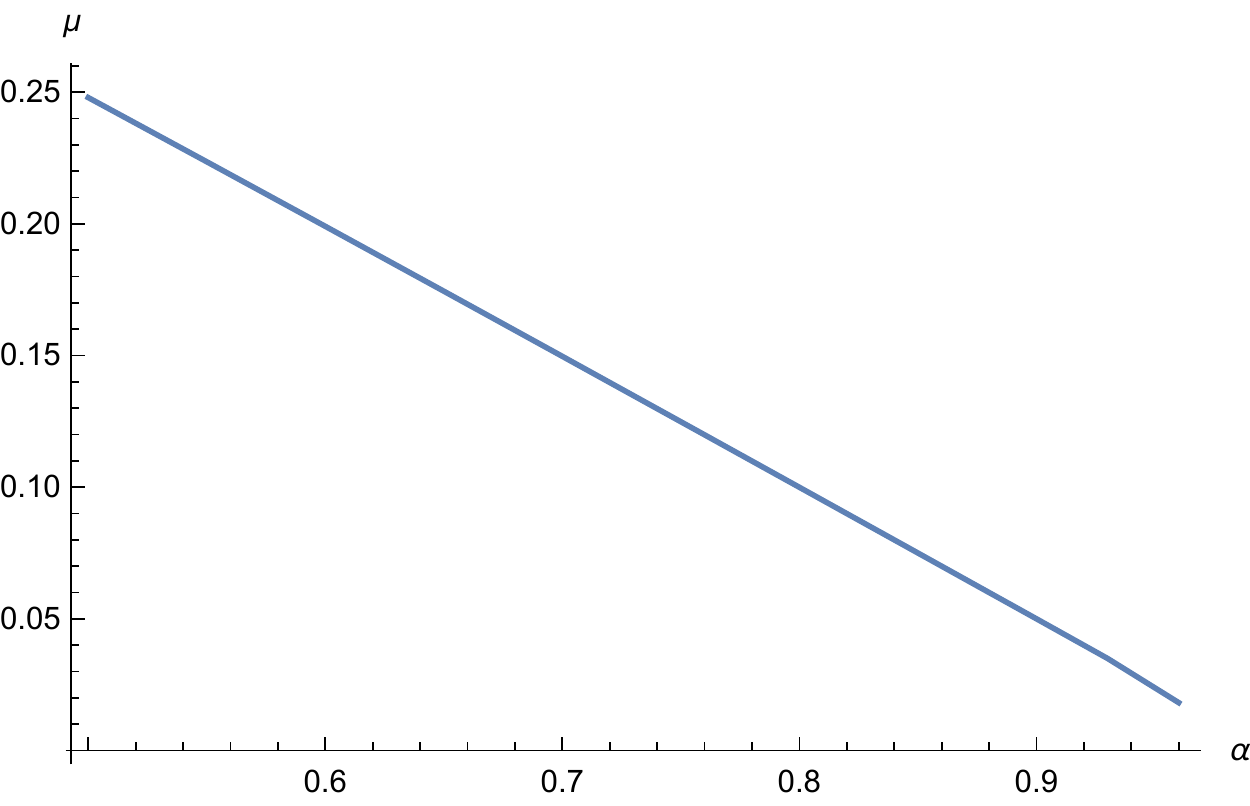}
\caption{{\bf(a)} A plot of $\dot \JJ_\alpha(\ee^*)$ versus $\alpha$.
	{\bf(b)} A plot of the exponent $\mu$ of the dynamics of $|\ee^* - \p \ee l|$ (see \cref{thm:alphaReLUeConvergence})
}
\label{fig:MuPlots}
\end{figure}

Because of the following theorem, we cannot expect the equations of \cref{thm:dalethRecFull} to hold for $\alpha \le \f 3 4$.
\begin{restatable}{thm}{dalethInfVarAlphaReLU} \label{thm:dalethInfVarAlphaReLU}
	Suppose we have the nonlinearity $\psi_\alpha$ in an FRN.
	$\Var(\dot \psi_\alpha(\zeta)^2)$ diverges for any Gaussian variable $\zeta$ with mean 0 if $\alpha \le \f 3 4$ but is finite if $\alpha > \f 3 4$.
\end{restatable}

\begin{restatable}{thm}{dalethDynamicsAlphaReLU} \label{thm:dalethDynamicsAlphaReLU}
Suppose we have the nonlinearity $\psi_\alpha$ in an FRN.
If $\alpha = 1$, then $\p\daleth{l-m} = \p \daleth{l} \left(\f 1 2 \sigma_v^2 \sigma_w^2 + 1\right)^m$.
If $\alpha \in (\f 3 4, 1)$, then $\p\daleth{l-m} = \Theta(1) \p \daleth{l} (l/(l-m))^R$ for $R = \f{\alpha^2}{(1-\alpha)(2 \alpha - 1)}$, where the constants in $\Theta(1)$ do not depend on $l$ or $m$.
\end{restatable}
This exponent $\f{\alpha^2}{(1 - \alpha)(2\alpha - 1)}$ is minimized at $\alpha = \f 3 4$ on $\alpha \in (3/4, 1)$, where the value is $\f 9 2$ (and at $\alpha = \f 2 3$ on $\alpha \in (1/2, 1)$, where the value achieved is 4) (\cref{fig:backprop_exponent_alpha-relu}(a)).

As a corollary,
\begin{restatable}{thm}{alphaReLUAllGradients}\label{thm:alphaReLUAllGradients}
If $\phi = \psi_1$ in an FRN, then for $l \ge m \ge 0,$
\begin{align*}
\p \chi {l-m} _b &= \Theta(1) \p \daleth l B^m,&
\p \chi {l-m} _w &= \Theta(1) \p \daleth l B^l,\\
\p \chi {l-m} _v&= \Theta(1) \p \daleth l B^l,&
\p \chi {l-m} _a &= \Theta(1) \p \daleth{l} B^m.
\end{align*}
where $B = 1 + \sigma_v^2\sigma_w^2/2$.

If $\phi = \psi_\alpha$ in an FRN, for $\alpha < 1$, then for $l \ge m \ge 0,$
\begin{align*}
\p \chi {l-m} _b &= \Theta(1) \p \daleth l l^R (l-m)^{-R-1},&
\p \chi {l-m} _w &= \Theta(1) \p \daleth l l^R (l-m)^{\f \alpha {1-\alpha} - R},\\
\p \chi {l-m} _v &= \Theta(1) \p \daleth l l^R (l-m)^{\f \alpha {1-\alpha} - R},&
\p \chi {l-m} _a &= \Theta(1) \p \daleth{l} (l/(l-m))^R.
\end{align*}
\end{restatable}

\cref{fig:alphaReLUVerifyExponents} verifies the backward asymptotic dynamics empirically for different $\alpha < 1$.
\cref{fig:backprop_exponent_alpha-relu}(b) graphs the exponent $\f \alpha {1-\alpha} - R$ in terms of $\alpha$.
We see that on $[0.5, 1]$, the maximum of this exponent is at $\alpha = 1$.
\begin{figure}
\centering
\includegraphics[width=.4\textwidth]{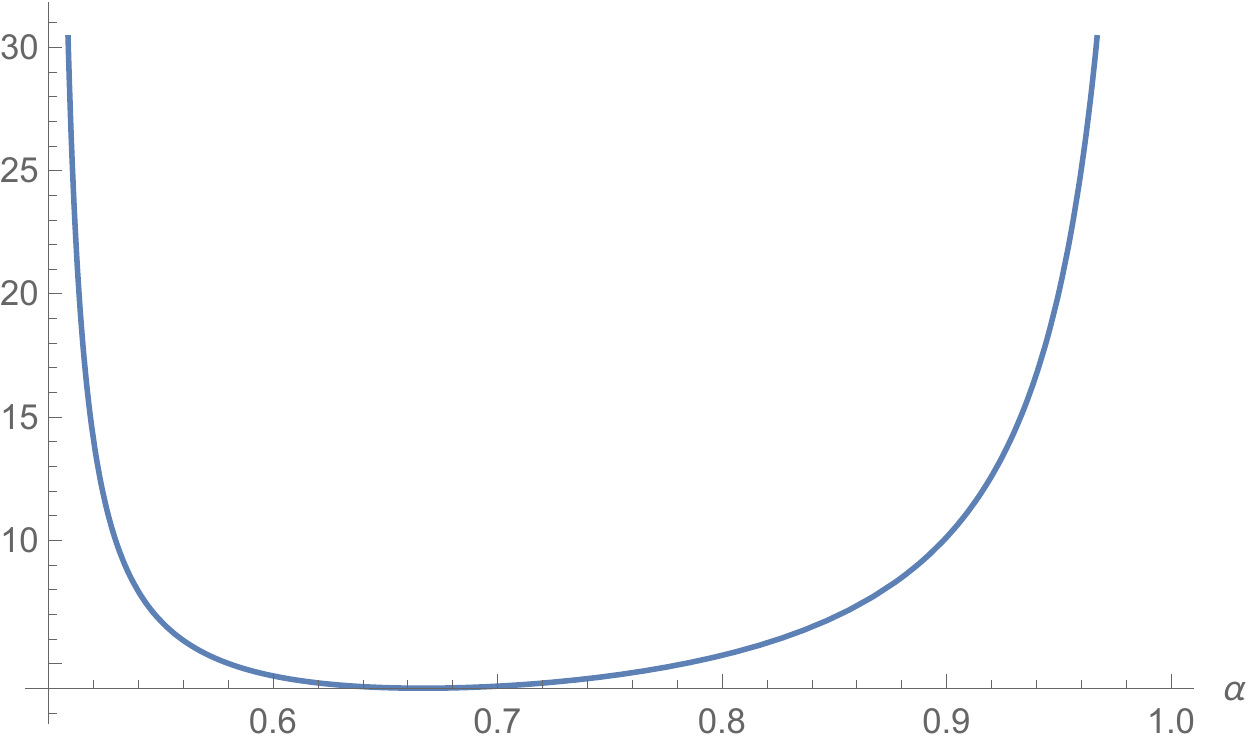}
\includegraphics[width=.4\textwidth]{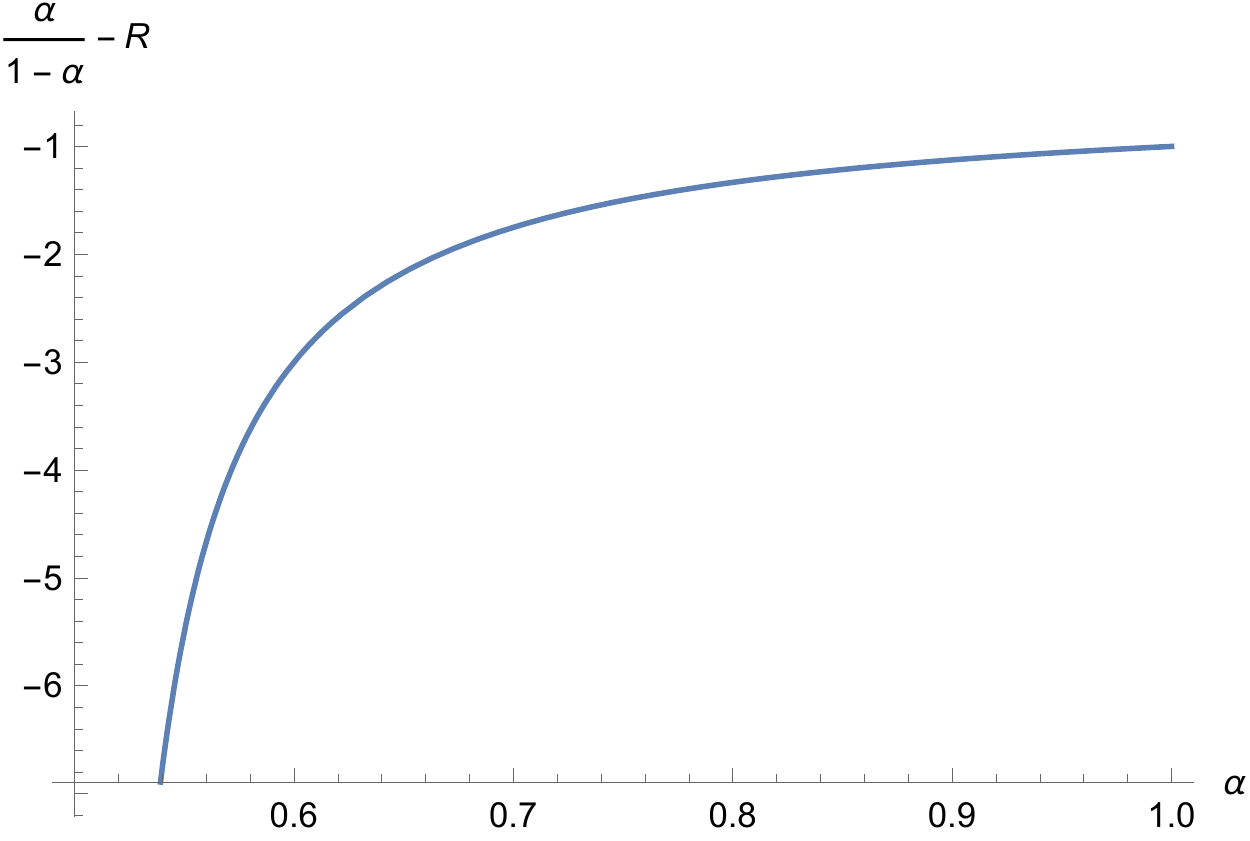}
\caption{\textbf{(a)} The exponent of the polynomial gradient dynamics with respect to $\alpha$-ReLU versus $\alpha$.
\textbf{(b)} The exponent of the dynamics of $\chi_v$ and $\chi_w$.}
\label{fig:backprop_exponent_alpha-relu}
\end{figure}

%However, we cannot expect $\daleth$ to track the empirical normalized squared gradient norms $\f 1 N \sum_{i=1}^N (\pd E/\pd x_i)^2$ for $\alpha \le \f 3 4$:
\section{Proofs}

A brief note about notation: We use $\sim$ to denote both how a random variable is sampled (ex: $x \sim \Gaus(0, 1)$ for a Gaussian $x$) and how a function behaves asymptotically, i.e. $f(x) \sim g(x)$ as $x \to a$ iff $\lim_{x \to a} f(x)/g(x) = 1$.
Context should be enough to differentiate between these two cases.
We in addition use $\simeq$ to denote asymptotic expansion.
For example, if $\{\alpha_i\}_{i \ge 0}$ is a sequence of strictly decreasing reals and $\{\beta_i\}_{i \ge 0}$ is a sequence of nonzero reals, then
$$f(x) \simeq \sum_{i \ge 0} \beta_i (x - \xi)^{\alpha_i}$$
means that as $x \to \xi$, $f(x) - \sum_{i = 0}^N \beta_i (x - \xi)^{\alpha_i} = \Theta((x- \xi)^{\alpha_{N+1}})$.
\subsection{Preliminary Lemmas}

\begin{lemma}\label{lemma:cExpansion}
We have
$$\f{\sigma_w^2  \gamma + \sigma_b^2}{\sigma_w^2  \pp + \sigma_b^2} =  \ee (1 + O( \gamma^{-1})).$$
regardless of whether $\p \ee l = \p \gamma l /\p \pp l$ converges.

But suppose $\p \ee l = \p \gamma l /\p \pp l \to \ee^*$.
If $\ee^* < 1$, then
$$\f{\sigma_w^2  \gamma + \sigma_b^2}{\sigma_w^2  \pp + \sigma_b^2} =  \ee (1 + \Theta( \gamma^{-1})).$$
If $\ee^* = 1$, then
$$\f{\sigma_w^2  \gamma + \sigma_b^2}{\sigma_w^2  \pp + \sigma_b^2} =  \ee (1 + \Theta( \eps  \pp^{-1})),$$
where $\eps = 1 - \ee$.
\end{lemma}
\begin{proof}
Write $M = \sigma_b^2/\sigma_w^2$.
\begin{align*}
\f{\sigma_w^2  \gamma + \sigma_b^2}{\sigma_w^2  \pp + \sigma_b^2} &=  \ee (1 + \f{1 + M \gamma^{-1}}{1 + M \pp^{-1}})\\
&=  \ee (1 + M(\inv \gamma - \inv \pp) + O(\inv \pp (\inv \gamma - \inv \pp))).
\end{align*}
In any situation, $\inv \gamma - \inv \pp = O(\inv \gamma)$ because $\gamma \le \pp$, so this gives the first statement.
If $\ee^*$ exists and $\ee^* < 1$, then $\inv \gamma - \inv \pp = \Theta(\inv \gamma)$, which yields the second statement.
If $\ee^*$ exists and $\ee^* = 1$, then $\inv \gamma - \inv \pp = \inv \pp((1-\eps)^{-1} - 1) = \inv \pp (\eps + O(\eps^2)) = \Theta(\eps \inv \pp)$.
\end{proof}

For any function $f$ that is $(k+1)$-times differentiable in a neighborhood of $0$, we have the asymptotic expansion
$$f(z) = \sum_{n = 0}^k \f{d^n f}{dz^n}(0) \f{z^n}{n!} + O(z^{k+1}), \text{as } z \to 0.$$
Since
\begin{align*}
	\left.\f{d^n}{d(1/q)^n}q^{1/2}\Vt \phi( q)\right\rvert_{q \to \infty} &= \f {(-1)^n} {2^n\sqrt{2\pi}}\int_{-\infty}^\infty \phi^2(z) z^{2n} \dd z
\end{align*}
whenever the RHS is integrable, we have
\begin{lemma}\label{lemma:Vt_asymptotic}
	Suppose $\phi^2(z) z^{2n}$ is integrable over $z \in \R$ for all $0 \le n \le N + 1$.
	Then $\Vt \phi( q) = q^{-1/2} (\sum_{n = 0}^N C_n q^{-n} + O(q^{-N-1}))$ as $q \to \infty$, where
	$$C_n := \f {(-1)^n} {2^n n! \sqrt{2\pi}}\int_{-\infty}^\infty \phi^2(z) z^{2n} \dd z.$$
\end{lemma}

Note that $\sech^d(z) = \Theta(e^{-d|z|})$ for $z \to \infty$ as long as $d > 0$, so that $C_n$ from the above result converges when $\phi = \sech^d$.
Therefore 
\begin{lemma}\label{lemma:Vt_sech_asymptotics}
	Let $d > 0$.
	We have $\Vt \sech^d( q) \simeq q^{-1/2} \sum_{n \ge 0} C_n q^{-n}$, where
	$$C_n := \f {(-1)^n} {2^n n! \sqrt{2\pi}}\int_{-\infty}^\infty \sech^{2d}(z) z^{2n} \dd z.$$
\end{lemma}
As corollaries, we obtain the following asymptotics.
\begin{lemma}\label{lemma:Vt_dot_tanh_asymptotics}
	$\Vt \dot \tanh( q) = \f 2 3 \sqrt{\f 2 \pi} q^{-1/2} + \Theta(q^{-3/2})$ as $q \to \infty$.
\end{lemma}
\begin{proof}
	Use \cref{lemma:Vt_sech_asymptotics} along with the fact that $\dot \tanh(z) = \sech^2(z)$ and $\int \sech^4 z \dd z = \f 2 3 \tanh z + \f 1 2 \sech^2 z \tanh z$.
\end{proof}
\begin{lemma}\label{lemma:vtanhSqrtConvergence}
	$1 - \Vt \tanh( q) = \sqrt{\f 2 \pi} q^{-1/2} + \Theta(q^{-3/2})$ as $q \to \infty$.
\end{lemma}
\begin{proof}
	Use \cref{lemma:Vt_sech_asymptotics} along with the fact that $1 - \tanh^2(z) = \sech^2(z)$ and $\int \sech^2 z  \dd z = \tanh z$.
\end{proof}

\begin{lemma}\label{lemma:sech_lower_bound}
$\sech^2(t) \ge \exp(-t^2)$ for all $t$, with equality iff $t = 0$.
\end{lemma}
\begin{proof}
The lower bound is equivalent to
\begin{align*}
2 & \ge e^{t-t^2/2} + e^{-t -t^2/2}
\end{align*}
The RHS has derivative $(1 - t)e^{t - t^2/2} - (1 + t)e^{-t-t^2/2}$.
This is 0 iff
\begin{align*}
\f{1 - t}{1 + t} = e^{-2t}
\end{align*}
which has a solution 0 and in general can only have solution $t\in (-1, 1)$ (by considering the sign of the LHS).
Since each side is analytic in $t \in (-1, 1)$, we expand
\begin{align*}
\log \f{1 - t}{1 + t} &= \log e^{-2t}\\
\log(1-t) - \log(1+t) &= -2t\\
(-t-t^2-\cdots) - (t-t^2+\cdots) &= -2t\\
-2t-2t^3-\cdots &= -2t
\end{align*}
which shows that the only solution is $t = 0$.
A simple plot shows that $t=0$ is a maximum, where the bound in question achieves equality. 

\end{proof}

\begin{lemma} \label{lemma:V_dot_tanh_lower_bound}
Suppose $\phi = \tanh$.
Then $\Vt \dot \phi( q) \ge \f 1 {\sqrt{4q+1}}$.
\end{lemma}
As a sanity check, \cref{lemma:Vt_dot_tanh_asymptotics} shows that $\Vt \dot \phi( q) \sim C_0 q^{1/2}$ where $C_0 \approx .5319$, which is above the .5 in this lemma.
\begin{proof}
	By \cref{lemma:sech_lower_bound},
	\begin{align*}
	\Vt\dot\phi( q) &= \int \dd \mu(z) \dot \phi^2(\sqrt q z)\\
	&\ge \f 1 {\sqrt{2\pi}} \int \dd z \exp(-z^2/2 - 2qz^2)\\
	&= \f 1 {\sqrt{2\pi}} \int \dd z \exp(-(4q + 1)z^2/2)\\
	&= \f 1 {\sqrt{4q + 1}}.
	\end{align*}
\end{proof}

\cref{fig:V_dot_tanh_lower_bound} demonstrates \cref{lemma:V_dot_tanh_lower_bound}.

\begin{figure}
\centering
\includegraphics[width=.4\textwidth]{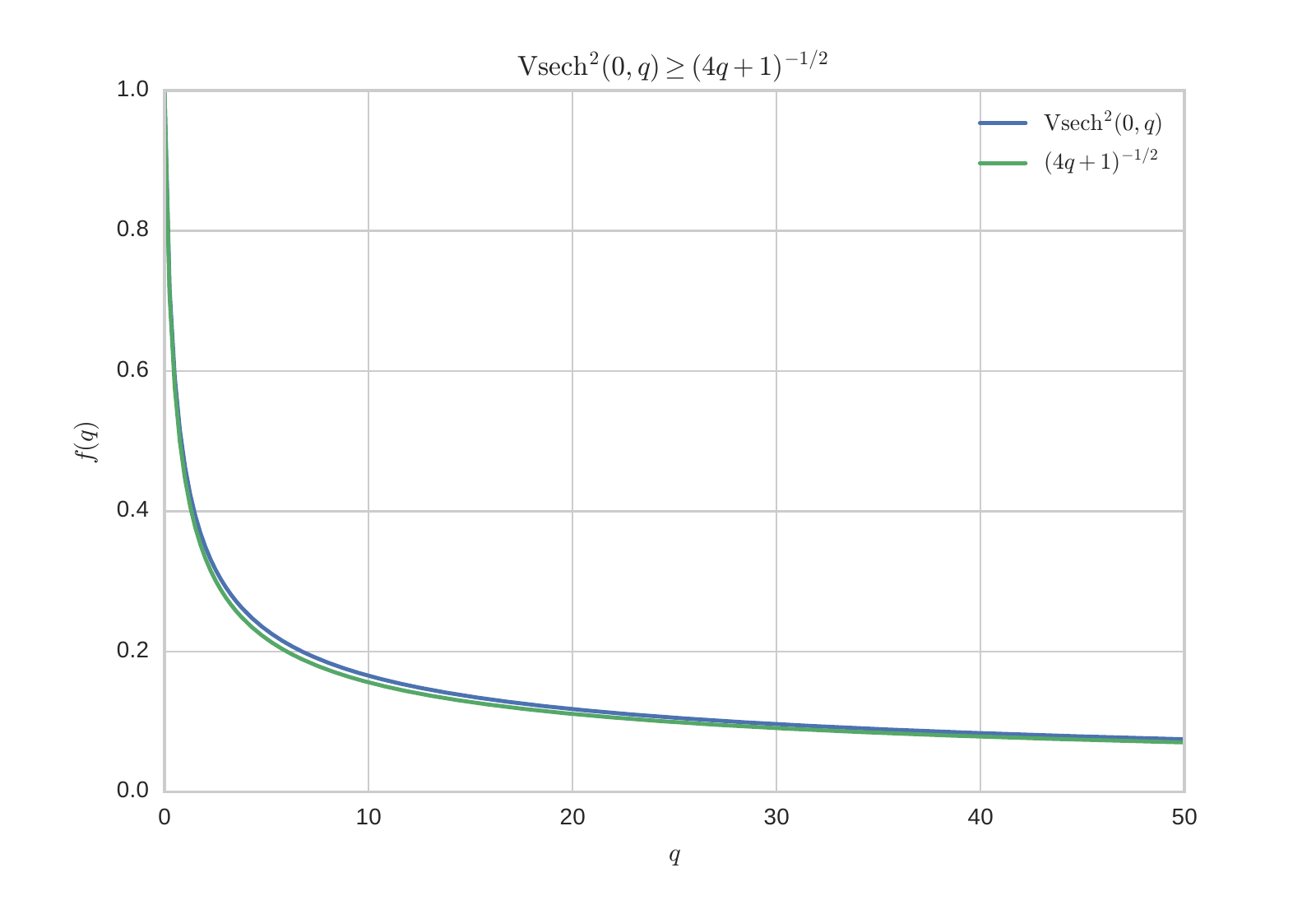}
\caption{Illustration of \cref{lemma:V_dot_tanh_lower_bound}:
$\Vt \dot \phi( q)$ vs $\f 1 {\sqrt{4q+1}}$ for $\phi = \tanh$.
This bound is very tight, and for most purposes, $\f 1 {\sqrt{4q+1}}$ can be taken as a good approximation of $\Vt \dot \phi( q)$.}
\label{fig:V_dot_tanh_lower_bound}
\end{figure}

\begin{lemma}\label{lemma:power_sum_asymptotics}
Let $d \in \R$ and $1 < M < N$ with $N - M \in \Z^{\ge 0}$.
Set $\Sigma(M, N, d) := \sum_{a=M}^N a^d$.
If we fix $M$ and let $N \to \infty$,
$$\Sigma(M, N, d) = \begin{cases}
	\Theta(1) & \text{if $d < -1$}\\
	\log N + O(1) & \text{if $d = -1$}\\
	\f{N^{d+1}}{d+1} + O(1) & \text{if $-1 < d < 0$}\\
	N - M + 1 & \text{if $d =0$}\\
	\f 1 {d+1}N^{d+1} + \f 1 2 N^d + O(N^{\max(0, d-1)}) & \text{if $d > 0$}
\end{cases}$$
\end{lemma}
\begin{proof}
	Consider the integrals $A = \int_M^{N+1} a^d \dd a$ and $B = \int^N_{M-1} a^d \dd a $.
	They evaluate to $A = \f 1 {d+1}((N+1)^{d+1} - M^{d+1})$ and $B = \f 1 {d+1}(N^{d+1} - (M-1)^{d+1})$ when $d \not = -1$ and to $A = \log(N+1) - \log M$ and $B = \log N - \log(M-1)$ when $d = -1$.
	When $d \le 0$, we have $A \le B$ and $\Sigma(M, N, d) \in [A, B]$; when $d > 0$, $B \le A$ and $\Sigma(M, N, d) \in [B, A].$
	Thus, as $N \to \infty$ with $M$ fixed, when $d < -1$, $\Sigma(M, N, d) = \Theta(1)$;
		when $d = -1$, $\Sigma(M, N, -1) = \log N + O(1)$;
		and when $d > -1$, we have $\Sigma(M, N, d) = \f{N^{d+1}}{d+1} + O(N^{d})$.
	
	Now for $a > 0$ and $d > -1$ and $d \not = 0, 1$,
	\begin{align*}
	\int_a^{a+1} z^d - a^d \dd z &= \f 1 {d+1} ( (a+1)^{d+1} - a^d )\\
		&= (a^d + \f d 2 a^{d-1} + \cdots ) - a^d\\
		&= \f d 2 a^{d-1} + \Theta(a^{d-2}).
	\end{align*}
	where the hidden constants in $\Theta$ depend only on $d$ (and in fact this term vanishes if $d = 1$).
	Thus
	\begin{align*}
	\Sigma(M, N, d) &= \int_M^{N+1} z^d \dd z - \sum_{a=M}^N [ \f d 2 a^{d-1} + \Theta(a^{d-2})]\\
		&= \f 1 {d+1} ((N+1)^{d+1} - M^{d+1}) - \f d 2 \Sigma(M, N, d-1) + \Theta(\Sigma(M, N, d-2))
	\end{align*}
	If $-1 < d < 0$, then $\Sigma(M, N, d - 1) = \Theta(1)$, so that $\Sigma(M, N, d) = \f{(N+1)^{d+1}}{d+1} + O(1) = \f{N^{d+1}}{d+1} + O(1).$
	If $d > 0$ and $d \not = 1$, then $\Sigma(M, N, d - 1) = \f{N^d}d$, so that 
	\begin{align*}
	\Sigma(M, N, d) &= \f 1 {d+1} N^{d+1} + N^{d} + \Theta(N^{\max(0, d-1)}) - \f 1 2 N^d + \Theta(\Sigma(M, N, d-2)) \\
				&= \f 1 {d+1}N^{d+1} + \f 1 2 N^d + O(N^{\max(0, d-1)}).
	\end{align*}
\end{proof}
We can obtain more terms in the expansion for higher $d$ via the Euler-Maclaurin formula, but this suffices for our purposes.

\subsection{Dynamics Zoo}
\renewcommand{\cc}{c}
This section deduces the asymptotic behaviors of some sequences governed by recurrence equations.
For the most part, the leading term of their asymptotic expansions is as one would expect from the corresponding differential equation.
However, in some cases we need subleading terms for later results.
They require slightly more nuanced reasoning.
First we present a technical lemma.
\begin{lemma}\label{lemma:simpleDynamics}
	Let $F: \R \times \N \to \R$ be a function such that for a subset $U \sbe \R$, and for all $z, z' \in U, z \ge z' \implies F(z, n) \ge F(z', n)$ for every $n$.
	Suppose sequences $\p a l, \p b l, \p \cc l$ satisfy 
	\begin{itemize}
		\item $\p a {l+1} = F(\p a l, l)$ for all $l$;
		\item $\p b {l+1} \le F(\p b l, l)$ for all $l$ above a constant $K_b$.
		\item $\p \cc {l+1} \ge F(\p \cc l, l)$ for all $l$ above a constant $K_c$.
	\end{itemize}
	and furthermore, $\p a l, \p b l, \p \cc l$ all fall into $U$ for $l$ above a constant $K_U$.
%	for a function $F(z, m)$ nondecreasing in $z$.

	If for some $m \ge \max(K_b, K_U)$, $\p b m \le \p a m$, then $\p b l \le \p a l, \forall l \ge m$.
	Similarly, if for some $n \ge \max(K_c, K_U)$, $\p \cc n \ge \p a n$, then $\p \cc l \ge \p a l, \forall l \ge n$.
\end{lemma}
\begin{proof}[]
	For the first claim: $\p b m \le \p a m \implies \p b {m+1} \le F(\p b m, m) \le F(\p a m, m) = \p a {m+1}$.
	Here the last inequality used the monotonicity of $F$.
	Induction gives the desired result.
	
	It's similar for the second claim, where the inductive step is $\p \cc m \ge \p a m \implies \p \cc {m+1} \ge F(\p \cc m, m) \ge F(\p a m, m) = \p a {m+1}$.
\end{proof}

\begin{lemma} \label{lemma:alphaDeltaRec}
	Suppose $\p \eps l$ satisfies the recurrence
	$$\p \eps l = \p \eps {l-1} (1 + \f \delta {l^\beta}).$$
	for some nonzero constant $\delta \in \R$ independent of $l$.
	\begin{itemize}
		\item If $\beta > 1$, then $\p \eps l = \Theta(1)$.
		\item If $\beta = 1$, then $\p \eps l = \Theta(l^{\delta})$.
		\item If $0 < \beta < 1$, then $\p \eps l = \exp(\f{\delta}{1 - \beta}l^{1-\beta} + \tilde\Theta(l^{\psi_1(1-2\beta)}))$, where $\psi_1(x) = \max(0, x)$ is the ReLU function.
	\end{itemize}
\end{lemma}
\begin{proof}
	We have
	\begin{align*}
	\log \p \eps l &= \log \p \eps {l-1} + \log(1 + \delta/l^\beta)\\
	&= \log \p \eps {l-1} + \delta/l^\beta + \Theta(\delta^2/l^{2\beta})
	\end{align*}
	for large $l$.
	If $\beta > 1$, then $\sum_l l^{-\beta}$ converges, and
	\begin{align*}
	\log \p \eps l &= \log \p \eps 0 - \Theta(1)\\
	\p \eps l &= \Theta(1).
	\end{align*}
	If $\beta = 1$, then
	\begin{align*}
	\log \p \eps l &= \log \p \eps 0 + \delta \log l + \Theta(1)\\
	\p \eps l &= \Theta(l^{\delta}).
	\end{align*}
	If $\beta < 1$, then
	\begin{align*}
	\log \p \eps l &= \log \p \eps 0 + \f\delta{1 - \beta} l^{1 - \beta} + \tilde\Theta(l^{1 - 2\beta})\\
	\p \eps l &= \exp( \f\delta{1 - \beta} l^{1 - \beta} + \tilde\Theta(l^{\psi_1(1 - 2\beta)})).
	\end{align*}
	
\end{proof}

\begin{lemma} \label{lemma:alphaDeltaDynamics}
	Suppose $\p \eps l = C l^{-\alpha} + \p \eps{l-1} (1 + \delta/l^{\beta})$ for $\alpha \in \R$, $C\not=0$, and $\delta \not = 0$.
	Then
	\begin{itemize}
		\item If $\beta > 1$, then
		\begin{itemize}
			\item $\p \eps l = \Theta(l^{1 - \alpha})$ if $\alpha \in (0, 1)$;
			\item $\p \eps l = \Theta(\log l)$ if $\alpha = 1$;
			\item $\p \eps l = \Theta(1)$ if $\alpha > 1$.
		\end{itemize}
		\item If $\beta = 1$, then
		\begin{itemize}
			\item $\p \eps l = \Theta(l^{\max(\delta, 1-\alpha)})$ if $1-\delta \not = \alpha$.
			\item $\p \eps l = \Theta(l^{\delta} \log l)$ if $1-\delta = \alpha$.
		\end{itemize}
		%     \begin{itemize}
		% 	    \item $\p \eps l = \Theta(l^{-\delta})$ if $\alpha > 1 + \delta$, and
		% 	    \item $\p \eps l = \Theta(l^{1-\alpha})$ if $\alpha \le 1 + \delta$.
		% 	\end{itemize}
	\end{itemize}
	Furthermore, for $\beta = -\delta = 1$, $\p \eps l \sim l^{-1}$ if $\alpha > 2$, $\p \eps l \sim l^{1-\alpha}$ if $\alpha < 2$, and $\p \eps l \sim l^{\delta} \log l$ if $\alpha = 2$.
\end{lemma}
\begin{proof}
	We can unwind the recurrence to get
	\begin{align*}
	\p \eps l &= \sum_{m=1}^l m^{-\alpha} \prod_{n=m+1}^l (1 + \f \delta {n^\beta}) + \p \eps 0 \prod_{n=1}^l (1 + \f \delta {n^\beta})
	\end{align*}
	
	Suppose $\beta > 1$.
	By \cref{lemma:alphaDeltaRec}, we get
	\begin{align*}
	\p \eps l &= \Theta(1)\sum_{m=1}^l m^{-\alpha} + \p \eps 0 \Theta(1)\\
	&= \begin{cases}
	\Theta(l^{1 - \alpha}) & \text{if $\alpha \in (0, 1)$}\\
	\Theta(\log l) & \text{if $\alpha = 1$}\\
	\Theta(1) & \text{if $\alpha > 1$.}
	\end{cases}
	\end{align*}
	
	Now suppose $\beta = 1$.
	By \cref{lemma:alphaDeltaRec}, we get
	\begin{align*}
	\p \eps l &= \sum_{m=1}^l m^{-\alpha} \Theta(m^{-\delta}l^\delta) + \p \eps 0 \Theta(l^{\delta})
	\end{align*}
	where the constants hidden inside the $\Theta$ are the same in every term of the sum.
	If $\alpha > 1 - \delta$, then $m^{-\delta - \alpha} = o(m^{-1})$, so that $\sum_{m=1}^l m^{-\delta - \alpha} = \Theta(1)$, and
	\begin{align*}
	\p \eps l &= \Theta(l^{\delta}) + \p \eps 0 \Theta(l^{\delta})\\
	&= \Theta(l^{\delta}).
	\end{align*}
	On the other hand, if $\alpha < 1 - \delta$, then $\sum_{m=1}^l m^{-\delta - \alpha} = \Theta(l^{1 - \delta - \alpha})$.
	So
	\begin{align*}
	\p \eps l &= \Theta(l^{1 - \alpha}) + \p \eps 0 \Theta(l^{\delta})\\
	&= \Theta(l^{1 - \alpha}).
	\end{align*}
	If $\alpha = 1 - \delta$, then $\sum_{m=1}^l m^{-\delta - \alpha} = \Theta(\log l)$.
	So
	\begin{align*}
	\p \eps l &= \Theta(l^{\delta}\log l) + \p \eps 0 \Theta(l^{\delta})\\
	&= \Theta(l^{\delta}\log l).
	\end{align*}
	
	Finally, if $\beta \in (0, 1)$, then 
	\begin{align*}
	\p \eps l &= e^{\f\delta{1-\beta}l^{1-\beta} + \Theta(l^{1-2\beta})}\sum_{m=1}^l m^{-\alpha} e^{\f{-\delta}{1-\beta} m^{1-\beta} + \Theta(m^{1-2\beta})} + e^{\f\delta{1-\beta}l^{1-\beta} + \Theta(l^{1-2\beta})}
	\end{align*}
	The case of $\delta = -1$ telescopes, so that the upper and lower constants hidden in $\Theta$ can both be taken to be 1.
\end{proof}

\begin{lemma} \label{lemma:invlogDynamics}
	Suppose for some $\beta > 0$, a sequence $\p \eps l$ satisfies
	$$\p \eps l = \p \eps {l-1} (1 - \mu (\p \eps {l-1})^\beta/l),\quad \p \eps 0 \in (0, \f 1 \mu).$$
	Then $\p \eps l \sim (\beta\mu \log l)^{-1/\beta}$.
\end{lemma}
\begin{proof}
	
	% Notice that $\p \eps l \not = O(l^{-\alpha})$ for any $\alpha > 0$; otherwise, $\prod_{l\ge 1} (1 - \mu \p \eps l / l) = O(1)$, and $\p \eps l$ does not go to 0.
	% Similarly, $\p \eps l \not = \Omega(l^\alpha)$ for any $\alpha \in (0, 1)$; otherwise $\prod_{l\ge 1} (1 - \mu \p \eps l / l) \to 0$, which contradicts $\p \eps l \to \infty$.
	% So $\p \eps l$ is $O(l^\alpha)$ and $\Omega(l^{-\alpha})$ for every $\alpha > 0$.
	
	Consider the differential equation
	$$\dot x_\mu = - \mu x^{\beta+1}_\mu/t$$
	for constant $\mu$ has solution $x_\mu = [\beta(\mu \log t + C)]^{-1/\beta}$ for some constant $C$ determined by initial condition.
	Note that
	$$
	-\mu x_\mu(t)^{\beta+1}/t \le x_\mu(t+1) - x_\mu(t) \le - \mu x_\mu(t+1)^{\beta+1}/(t+1) = - (1 - o(t^{-1}))\mu x_\mu (t)^{\beta+1}/t.$$
	For any small enough $\alpha > 0$, we apply \cref{lemma:simpleDynamics} with $F(\eps, l) = \eps - \mu \eps^{\beta+1}/l$ (which is monotonic in $\eps$ for small enough $\eps$), $\p \cc l = x_\mu(l)$, and $\p b l = x_{\mu-\alpha}(l)$ to obtain
	%Since the function $\eps \mapsto \eps - \mu \eps/l$ is monotonic for $l \gg 1$ and $\log \eps \ll -1$, 
	$$x_{\mu - \alpha}(l) \le \p \eps l \le x_\mu(l)$$
	for large enough $l$ and appropriately chosen initial conditions.
	This shows that $\p \eps l = \Theta(\log l ^{-1/\beta})$
%	So there exists a constant $C'$ such that for any $\alpha > 0$, $[\beta((\mu-\alpha) \log t + C)]^{-1/\beta} \p \eps l \le [\beta(\mu \log t + C')]^{-1/\beta}$ for constants $C$ depending on $\alpha$.
	Taking $\alpha \to 0$, we also obtain the leading coefficient $\p \eps l \sim [\beta\mu \log l]^{-1/\beta}$.
	
\end{proof}

\begin{lemma}\label{lemma:polyDynamics}
	Suppose a sequence $\p u l$ is governed by the equation
	$$\p u l - \p u {l-1} = A(\p u {l-1} + B)^\alpha,$$
	where $\alpha \in [0, 1)$ and $A > 0$.
	Then $\p u l = K_1 l^\oalpha - K_2 l^\aalpha \log l + o(l^\aalpha \log l)$, where $K_1 = [A(1-\alpha)]^\oalpha$ and $K_2 = \f 1 2 A^{\oalpha} (1-\alpha)^{\aalpha - 1} \alpha$.
\end{lemma}
\begin{proof}
	\noindent\textbf{Leading term.}
	The differential equation
	$$\dot x_{A, B} = A(x_{A, B} + B)^\alpha$$
	has solution $x_{A,B}(l) = [A(1-\alpha)(l + S)]^{\f 1 {1-\alpha}} - B$ for some constant $S$.
	Since $\dot x_{A, B}$ is monotonic, we have (writing $x = x_{A, B}$ for brevity)
	$$A(x_{A, B}(l) + B)^\alpha = \dot x_{A, B}(l) \le x_{A, B}(l+1) - x_{A, B}(l) \le \dot x_{A, B}(l+1) \le (A + o(1))(x_{A, B}(l) + B)^\alpha$$
	for large enough $l$.
	We apply \cref{lemma:simpleDynamics} with $F(x, l) = x + A(x + B)^\alpha$ (which is monotonic in $x$ for large $x$), $\p \cc l = x_{A, B}(l)$, and $\p b l = x_{A-\eps, B}(l)$ to obtain
	$$x_{A-\eps,B}(l) \le \p u l \le x_{A, B}(l)$$
	for large enough $l$ and appropriate initial conditions.
	Therefore $\lim \p u l / l^{\f 1 {1-\alpha}} \in [[(A - \eps) ( 1- \alpha)]^{\f 1 {1 - \alpha}}, [A ( 1- \alpha)]^{\f 1 {1 - \alpha}}].$
	Taking $\eps \to 0$ gives the leading term.
	
	\noindent\textbf{Subleading term.}
	Now let $\p v l := \p u l - \aleph l^{\f 1 {1-a}}$, where $\aleph = [A(1-\alpha)]^{\f 1 {1-\alpha}}.$
	Then we have the recurrence
	\begin{align*}
	\p v {l+1} + \aleph (l + 1)^{\oalpha} - \p v l - \aleph l^{\oalpha} &= A(\p v l + \aleph l^{\oalpha} + B)^\alpha\\
	\p v {l+1} - \p v l + \aleph(\oalpha l^{\aalpha} &+ \f 1 2 (\oalpha) (\aalpha) l^{\aalpha-1} + \Theta(l^{\aalpha - 2}))  \\&= A[\aleph^\alpha l^{\aalpha} + \alpha (\p v l+B) \aleph^{\alpha - 1} \inv l + \Theta((\p v l+B) l^{-1-\oalpha})]\\
	\p v {l+1} - \p v l &= \aalpha \p v l \inv l - \f 1 2 \aleph(\oalpha)(\aalpha) l^{\aalpha - 1} + g(l)
	\end{align*}
	for some $g(l) = O(l^{\aalpha - 2} + \inv l)$ and where, to get the last equation, we have used $A \alpha^\alpha = \oalpha\aleph$ to cancel the $l^{\aalpha}$ term and simplified $\alpha A \aleph^{\alpha-1} = \aalpha$.
	
	For any $J > 0$, the differential equation $\dot v_J(l) = \aalpha v_J(l) \inv l - J l^{\aalpha - 1}$ has solution $v_J(l) = C[l(1-\alpha)]^{\aalpha} - J l^{\aalpha} \log l$.
	Note that the functions $F_J(z, n) = z + \aalpha z \inv n - J n^{\aalpha - 1}$ and $G_J(z, n) = F_J(z, n) + g(n)$ is monotonic in $z$ (for positive $n$).
	For large $l$, we also have $\dot v_J(l)$ and $F_J(v_J(l), l) = v_J(l) + \dot v_J(l)$ decreasing in $l$.
	Thus for any $\eps > 0$ and $l$ large enough
	$$G_{J+\eps}(v_J(l), l) \le F_{J+\eps/2}(v_J(l), l) \le v_J(l) + \dot v_J(l+1) \le v_J(l+1) \le F_J(v_J(l), l) \le G_{J-\eps}(v_J(l), l).$$
	Now apply \cref{lemma:simpleDynamics} with $F = G_{K}$, $\p a l = \p v l, \p \cc l = v_{K-\eps}, \p b l = v_{K+\eps}$ where $K := \f 1 2 \aleph(\oalpha)(\aalpha) = \f 1 2 A^{\oalpha} (1-\alpha)^{\aalpha - 1} \alpha$, with appropriately chosen initial conditions.
	This yields $\lim_{l\to \infty} \p v l / (l^{\aalpha} \log l) \in [-K - \eps, -K + \eps]$ for every $\eps > 0$, and there it must be equal to $K$.
	We have thus obtained the asymptotic expansion
	$$\p u l = [A(1-\alpha) l]^\oalpha - \f 1 2 A^{\oalpha} (1-\alpha)^{\aalpha - 1} \alpha l^\aalpha \log l + o(l^\aalpha \log l).$$
	
\end{proof}

\begin{lemma}\label{lemma:polyDynamicsPosAlpha}
	Suppose a sequence $\p u l$ is governed by the equation
	$$\p u l - \p u {l-1} = -A(\p u {l-1} + B)^\alpha,$$
	where $\alpha > 1$ and $A > 0$.
	Then $\p u l \sim [A(\alpha-1)l]^{\f 1 {1-\alpha}}$.
\end{lemma}
\begin{proof}
	Similar to \cref{lemma:polyDynamics}.
\end{proof}

\begin{lemma}\label{lemma:polyDynamicsConstant}
	Suppose a sequence $\p u l$ is governed by the equation
	$$\p u l - \p u {l-1} = A(\p u {l-1} + B)^\alpha + C,$$
	where $\alpha \in (0, 1)$.
	Then $\p u l = K_1 l^\oalpha + R(l)$, where the remainder $R(l)$ is
	$$R(l) \sim \begin{cases}
	-K_2 l^{\aalpha} \log l	&\text{if $\alpha > \f 1 2$}\\
	(C-K_2) l \log l			&\text{if $\alpha = \f 1 2$ and $K_2 \not=C$}\\
	\f{C(1-\alpha)}{1-2\alpha} l	& \text{if $ \alpha < \f 1 2$}
	\end{cases}$$
	where $K_1 = [A(1-\alpha)]^{\f 1 {1-\alpha}}, K_2 = \f 1 2 A^{\oalpha} (1-\alpha)^{\aalpha - 1} \alpha$ as in \cref{lemma:polyDynamics}.
\end{lemma}
\begin{proof}
	$u$ is bounded below by the dynamics $\p v l - \p v {l-1} = A(\p v {l-1} + B)^\alpha$ and bounded above by the dynamics $\p w l - \p w {l-1} = (A + o(1))(\p w {l-1} + B)^\alpha$.
	By \cref{lemma:polyDynamics}, both $v$ and $w$ are asymptotic to $\p u l \sim [A(1-\alpha)l]^{\f 1 {1-\alpha}}$, which gives the result.
	
	Now define $\p v l = \p u l - [A(1-\alpha)l]^\oalpha$, and similar to the proof of \cref{lemma:polyDynamics}, we find
	$$\p v {l+1} - \p v l = \aalpha \p v l \inv l - K l^{\aalpha - 1} + C + g(l)$$
	where $K = \f 1 2 A^{\oalpha} (1-\alpha)^{\aalpha - 1} \alpha$ and $g(l) = O(l^{\aalpha - 2} + \inv l)$.
	If $\aalpha > 1 \iff \alpha > \f 1 2$, then $C + g(l) = o(l^{\aalpha - 1})$ and we can proceed as in the proof of \cref{lemma:polyDynamics} to find $\p v l \sim K l^{\aalpha} \log l.$
	If $\aalpha = 1 \iff \alpha = 1$ and $K \not = C$, then $\p v {l+1} - \p v l = \aalpha \p v l l^{-1} - (K-C) l^{\aalpha - 1} + g(l)$, so that the technique used in \cref{lemma:polyDynamics} would obtain $\p v l \sim (K-C) l^{\aalpha} \log l = (K-C) l \log l.$
	If $\aalpha < 1 \iff \alpha < \f 1 2$, then $\p v {l+1} - \p v l = \aalpha \p v l l^{-1} + C + o(1)$, then by using the differential equation $\dot v_J(l) = \aalpha v_J(l) l^{-1} + J$ to approximate the difference equation solution and applying \cref{lemma:simpleDynamics} as in the proof of \cref{lemma:polyDynamics}, we obtain $\p v l(l) \sim \f{C(1-\alpha)}{1-2\alpha} l$.
\end{proof}

\subsection{Forward Dynamical Equations}
Here we derive the recurrences governing the forward length and correlation quantities $\pp, \qq, \lambda, \gamma.$
We start with reduced residual networks.
\pqrecurrence*
\begin{proof}
	We have
	\begin{align*}
	\qq &= \la h_j^2 \ra = \la \sum_i (w_{ji}\prv x_i + b_j)^2 \ra\\
	&= \la b_j^2 \ra + \sum_i \la w_{ji}^2 \prv x_i^2\ra + 2\sum_i \la w_{ji} \prv x_i b_j \ra + 2 \sum_{j \not= l} \la w_{ji} w_{li} x_i^2\ra 
	\end{align*}
	But $w_{ji}, w_{li}, \prv x,$ and $b_j$ form an independency, so the last two sums are 0, and the terms in the first sum split multiplicatively.
	Therefore
	\begin{align*}
	\qq &= \sigma_b^2 + \sum_i \la w_{ji}^2 \ra \la \prv x_i^2 \ra\\
	&= \sigma_b^2 + N \cdot \f {\sigma_w^2}{N} \prv \pp\\
	&= \sigma_b^2 + \sigma_w^2 \prv \pp. 
	\end{align*}
	
	For the recurrence of $\pp$, we have
	\begin{align*}
	\pp &= \la x_i^2 \ra = \la (\phi(h_i) + \prv x_i)^2 \ra \\
	&= \la \phi(h_i)^2\ra + \la \prv x_i^2 \ra + 2 \la\phi(h_i) \prv x_i \ra
	\end{align*}
	As $N \to \infty$, the coefficient $w_{ii}$ of $\prv x_i$ in $h_i$ has vanishing covariance, so $h_i$ and $\prv x_i$ become independent.
	Therefore $\la \phi(h_i) \prv x_i \ra = \la \phi(h_i) \ra \la \prv x_i \ra$.
	Because $h_i$ is the sum of a large number of independent random variables, by CLT, $h_i$ is a Gaussian with mean $\sum_i \la w_{ji}\ra \la \prv x_i\ra + \la b_j\ra = 0$ since $\la w_{ji} \ra = \la b_j \ra = 0$.
	Our antisymmetry assumption on $\phi$ then implies $\la \phi(h_i) \ra = 0$.
	Therefore,
	\begin{align*}
	\pp &= \la \phi(h_i)^2\ra + \la \prv x_i^2 \ra\\
	&= \Vt \phi( \qq) + \prv \pp
	\end{align*}
	as desired.
\end{proof}

\lambdagammarecurrence*
\begin{proof}
	Similar to \cref{lemma:p_q_recurrence}.
\end{proof}

Now, for the full residual networks, the proofs are similar, but we no longer need to assume that $\phi$ is antisymmetric because of the randomization via the extra sets of weights.
\fullResPQRec*

\begin{proof}
	\begin{align*}
	\qq &= \la  h_j^2 \ra
	= \la ( w_j^i \prv x_i +  b_j)^2\ra
	= \la ( w_j^i \prv x_i)^2 \ra  + \la  b_j^2\ra\\
	&= \sigma^2_w \la \prv x_i^2 \ra + \sigma_b^2\\
	&= \sigma^2_w \prv \pp + \sigma^2_b\\
	\pp &= \la  x_i^2 \ra
	= \la (v^j_i\phi( h_j) +\prv x_i + a_i)^2 \ra\\ 
	&= \sigma_v^2\la \phi( h_i)^2 \ra + \la \prv x_i^2 \ra + \sigma_a^2\\
	&=\sigma_v^2 \Vt \phi(  \qq) + \sigma_a^2 + \prv \pp\\
	\end{align*}
	where in the third equality for $\pp$, we are now using the independence of $v_i^j$ from all other variables to cancel out the terms, whereas before we had to rely on $\phi$ being antisymmetric.
\end{proof}

\LGRecFullRes*
\begin{proof}[]
	Similar to \cref{thm:fullResPQRec}.
\end{proof}

\subsection{Backward Dynamical Equations}
Here we derive the recurrences governing the gradient quantities $\daleth$ and $\chi_{\bullet}$ for different $\bullet$, all under the gradient independence assumption.
Write $\p \beth l_i = \pdf E {\p x l_i}$ for a cost function $E$.
\dalethRecReduced*
\begin{proof}
	For a reduced residual network, we have the following derivative computation:
	\begin{align*}
	\pdf {x_i} {\prv x_j} &= \delta_{ji} + \dot \phi(h_i) \pdf {h_i}{\prv x_j},&
	\pdf {x_i}{h_j} &= \delta_{ji} \dot \phi(h_j),&
	\pdf {h_i}{\prv x_j} &= w_{ij},&
	\pdf {h_i}{w_{ij}} &= \prv x_j,&
	\pdf {h_i}{b_j} &= \delta_{ij}.
	\end{align*}
	
	Then
	\begin{align*}
	% \pdf{E}{\prv x_j} &= \sum_i \pdf E{x_i} \pdf{x_i}{\prv x_j}  
	\prv \beth_j &= \beth_j + \sum_i \beth_i \dot \phi(h_i) \pdf {h_i}{\prv x_j}\\
	&= \beth_j + \sum_i \beth_i\dot \phi(h_i) w_{ij}\\
	\la \prv \beth_j^2 \ra &= \la [\beth_j + \sum_i \beth_i\dot \phi(h_i) w_{ij}]^2 \ra\\
	&= \la \beth_j^2 \ra + \sum_i \la \beth_i^2 \dot \phi^2(h_i) (w_{ij})^2\ra \\
	&\phantom{={}}+ 2\sum_{i < k} \la \beth_i \beth_k\dot\phi(h_i)w_{ij}\dot\phi(h_k)w_{kj}\ra + 2 \sum_i\la \beth_j \beth_i \dot\phi(h_i)w_{ij} \ra
	% &= \la \beth_i^2 \dot \phi^2(h_i) (w^j_i)^2\ra + \la \beth_j^2\ra + 2 \la \beth_i \beth_k\dot\phi(h_i)w^j_i\dot\phi(h_k)w^j_k\ra + 2 \la \beth_j \beth_i \dot\phi(h_i)w^j_i \ra
	\end{align*}
	The last two terms of the above vanish as $w_{ij}$ is independent from $w_{kj}$, $h_i, h_k$ and $\beth_i, \beth_j, \beth_k$ by \cref{ass:gradInd}, and $\la w_{ij} \ra = 0$.
	
	Therefore, applying \cref{ass:symAct},
	\begin{align*}
	\la \prv \beth_j^2 \ra &=  \sigma_w^2\la \beth_j^2\ra\la\dot \phi^2(h_i)\ra + \la \beth_j^2\ra\\
	&= (\sigma_w^2  \Vt \dot \phi( \qq) + 1)\la \beth_j^2 \ra
	\end{align*}
	
	We similarly have
	\begin{align*}
	\pdf E {b_j} &= \sum_i \pdf E {x_i} \pdf {x_i}{h_j}
	= \beth_j \dot \phi(h_j), & \text{since $\pdf {x_i}{h_j} = \delta_{ji} \dot \phi(h_j)$}\\
	\la \lp\pdf E {b_j}\rp^2 \ra &= \la \beth_j^2 \dot \phi(h_j)^2 \ra
	= \la \beth_j^2 \ra \Vt \dot \phi( \qq), & \text{by \cref{ass:gradInd}(b)};\\
	\pdf E {w_{ji}} &= \sum_i \pdf E {x_i} \pdf {x_i}{h_j}\pdf {h_j}{w_{ji}}
	= \beth_j \dot \phi(h_j)\prv x_i, & \text{since $\pdf {x_i}{h_j} = \delta_{ji} \dot \phi(h_j)$}\\
	\la \lp\pdf E {w_{ji}}\rp^2 \ra &= \la \beth_j^2 \dot \phi^2(h_j) \prv x_i^2 \ra
	= \la \beth_j^2 \ra \Vt \dot \phi( \qq) \prv \pp , & \text{by \cref{ass:gradInd}(b)}
	\end{align*}
	In the last equation we have also used the fact that as $N \to \infty$, $h_j$ and $x_i$ become independent (they are jointly Gaussian and their correlation $\la w_{ji}^2 \ra$ goes to 0 with $N$).
\end{proof}

\dalethRecFull*
\begin{proof}
	For the full residual network, we have the following derivative computations:
	\begin{align*}
	\pdf {x_i} {\prv x_j} &= \delta_{ji} + \sum_k v_{ik}\dot \phi(h_k) \pdf {h_k}{\prv x_j},&
	\pdf {x_i}{h_j} &= v_{ij} \dot \phi(h_j),&
	\pdf {h_i}{\prv x_j} &= w_{ij},&
	\pdf {h_i}{w_{ij}} &= \prv x_j,&
	\pdf {h_i}{b_i} &= 1,\\
	&&
	\pdf {x_i}{v_{ik}} &= \phi(h_k),&
	\pdf {x_i}{a_i} &= 1.
	\end{align*}
	
	Again let $\beth_j = \pdf E {x_j}$.
	Then
	\begin{align*}
	\prv \beth_j &= \sum_i \beth_i (\delta_{ji} + \sum_k v_{ik}\dot \phi(h_k) \pdf {h_k}{\prv x_j})\\
	&= \sum_i \beth_i (\delta_{ji} + \sum_k v_{ik}\dot \phi(h_k) w_{kj})
	\end{align*}
	
	Thus,
	\begin{align*}
	\la \prv \beth_j^2 \ra &= \la [\sum_i \beth_i (\delta_{ji} + \sum_k v_{ik}\dot \phi(h_k) w_{kj})]^2 \ra\\
	&= \la \beth_j^2 \ra + \sum_{i, k} \la v_{ik}^2 \ra \la w_{kj}^2\ra \Vt \dot \phi( \qq) \la \beth_i^2 \ra\\
	&= \la \beth_j^2 \ra (1 + \sigma_v^2 \sigma_w^2 \Vt \dot \phi( \qq))
	% &= \la \beth_i^2 (v^k_i)^2\dot \phi^2(h_k) (w_k^j)^2\ra + \la \beth_j^2\ra
	\end{align*}
	where in the second equality we applied the independence argument as in the proof of \cref{thm:dalethRecReduced}, leveraging \cref{ass:gradInd}, and in the third equality we used \cref{ass:symAct} to get $\la \beth_i^2\ra = \la \beth_j^2 \ra$.
	
	The other computations are similar to the proof of \cref{thm:dalethRecFull}.

\end{proof}

\subsection{Tanh: Reduced Residual Network}
\subsubsection{Forward Dynamics}
\pqlinear*
\begin{proof}
The case with $\sigma_w = 0$ is trivial.
We assume $\sigma_w > 0$ from here on.

\noindent\textbf{$\pp$ and $\qq$ are asymptotically linear with $l$.}
We first show that, for any $\omega < 1$, 
$$l + \p \pp 0 \ge \p \pp l \ge \omega l$$
and
$$\sigma_w^2 (l + \p \pp 0) + \sigma_b^2 \ge \p \qq l \ge \sigma_w^2 \omega (l-1) + \sigma_b^2,$$
so that $\p \pp l \sim l$ and $\p \qq l \sim \sigma_w^2 l$.

The upper bounds are trivial, given $\Vt \phi( \qq) \le 1$ for any $\qq$.
We show the lower bounds for any $\omega < 1$.

For any $\eps > 0$, define $\aleph_\eps$ by $\phi^2(\aleph_\eps) = \exp(-\eps)$.
Then
\begin{align*}
\Vt \phi( \qq) &\ge \exp(-\eps) \Pr[z \not \in [-\aleph_\eps, \aleph_\eps]: z \sim \Gaus(0, \qq)]\\
	&\ge \exp(-\eps)\lp 1 - \f{2 \aleph_\eps}{\sqrt{2\pi \qq}} \rp
\end{align*}
where the second inequality follows from an overestimate of the $\Pr[z \in [-\aleph_\eps, \aleph_\eps]]$ via the mode of $\Gaus(0, \qq)$.

For any $\qq \ge \p \qq 0$, $\Vt \phi( \qq)$ is then lower bounded by 
$$\phi^2\lp \sqrt{\p \qq 0} \rp
\lp1 - \f{2\sqrt{\p \qq 0}}{\sqrt{2 \pi \p \qq 0}}\rp
= \phi^2\lp\sqrt{\p \qq 0}\rp
\lp1 - \sqrt{\f 2 \pi}\rp > 0.$$
Thus $\p \pp l$ and $\p \qq l$ are unbounded with $l$.

Furthermore, as $\qq \to \infty$, the lower bound $\exp(-\eps)\lp 1 - \f{2 \aleph_\eps}{\sqrt{2\pi \qq}} \rp$ goes to $\exp(-\eps)$, for any $\eps$.
Therefore, for any $\omega < 1, \p \pp l \ge \omega l$ and $\p \qq l \ge \sigma_w^2 \omega (l-1) + \sigma_b^2$.

\noindent\textbf{Asymptotic expansion.}
Now we repeat the following to get each successive asymptotic term of $\p \pp l$ and $\p \qq l$:
We plug in the current asymptotic form of $\p \qq l$ into $\Vt \tanh( \qq) = 1 - C \qq^{-1/2} + \Theta(\qq^{-3/2})$ (\cref{lemma:vtanhSqrtConvergence}), where $C = \sqrt{2/\pi}$.
Next we take the sum $\p \qq l = \sum_{r=1}^l \Vt \tanh( \p \qq r)$, which yields one more term in the asymptotic expansion of $\pp$ than the last round.
We then repeat until we get only constant terms.

% Caution about summing over little-o: \sum o(f(n)) = o(\sum f(n)) works only if \sum f(n) -> \infty.
The following exhibits a trace of this procedure, where in the summation step for $\p \qq l$, we implicitly apply 
\begin{align*}
\qq &= \sigma_w^2 l + o(l) = \sigma_w^2 l (1 + o(1))\\
\qq^{-1/2} &= \sigma_w^{-1} l^{-1/2}(1 + o(1)) = \sigma_w^{-1} l^{-1/2} + o(l^{-1/2})\\
\pp &= \sum_{r=1}^l 1 - C (\p \qq r)^{-1/2} + \Theta((\p \qq r)^{-3/2}) \\
	&= \sum_{r=1}^l 1 - C (\sigma_w^{-1} r^{-1/2} + o(r^{-1/2})) + \Theta(r^{-3/2}) \\
	&= l - 2C \sigma_w^{-1} l^{1/2} + o(l^{1/2})\\
\qq &= \sigma_w^2 l - 2 C \sigma_w l^{1/2} + o(l^{1/2}) = \sigma_w^2 l (1 - 2 C \sigma_2^{-1} l^{-1/2} + o(l^{-1/2}))\\
\qq^{-1/2} &= \sigma_w^{-1} l^{-1/2} ( 1 + C \sigma_w^{-1} l^{-1/2} + o(l^{-1/2}))= \sigma_w^{-1} l^{-1/2} + C \sigma_w^{-2} l^{-1} + o(l^{-1})\\
\pp &= \sum_{r=1}^l 1 - C(\sigma_w^{-1} l^{-1/2} + C \sigma_w^{-2} l^{-1} + o(l^{-1})) + \Theta(l^{-3/2})\\  
	&= l - 2 C \sigma_w^{-1} l^{1/2} - C^2 \sigma_w^{-2} \log l + o(\log l)\\
\qq &= \sigma_w^2 l (1 - 2 C \sigma_w^{-1} l^{-1/2} - C^2 \sigma_w^{-2} \f{\log l}{l} + o(\f{\log l}{l}))\\
\qq^{-1/2} &=\sigma_w^{-1} l^{-1/2}(1 + C \sigma_w^{-1} l ^{-1/2} + \f 1 2 C^2 \sigma_w^{-2} \f{\log l}{l} + o (\f{\log l}{l}))\\
\pp &= \sum_{r=1}^l 1 - C(\sigma_w^{-1} r^{-1/2} + C \sigma_w^{-2} r ^{-1} + \f 1 2 C^2 \sigma_w^{-3} \f{\log r}{r^{3/2}} + o (\f{\log r}{r^{3/2}}))) + \Theta(r^{-3/2})\\
	&= l - 2 C \sigma_w^{-1} l^{1/2} - C^2 \sigma_w^{-2} \log l + O(1)
\end{align*}
which is what we want.
\end{proof}

\begin{lemma}\label{lemma:WtPhiAntisymmetric}
	\renewcommand{\qq}{q}
	Let $\phi$ is antisymmetric.
	Then for $\tau \in [0, \pi/2]$,
	\begin{align*}
	\Wt\phi( \qq, \qq\cos \tau) &= \lim_{t \to \tau}\f 1 {\pi \sin t} \int_{w' \ge |w|}  \dd w \dd w' \Upsilon(w, w';\tau) \phi(\f {\sqrt{\qq}} {\sqrt{2}}(w + w')) \phi(\f {\sqrt \qq} {\sqrt{2}}(w' - w))\\
		&= \f 1 \pi \int_{0}^\infty r\dd r e^{-r^2/2} \int_{0}^{\pi}\dd \theta  \Sigma(\sqrt \qq r, \theta; \tau)\\
		&= \f 1 \pi \int_0^\infty s \dd s \inv \qq e^{-s^2 \inv \qq /2} \int_0^\pi \dd \theta \Sigma(s, \theta; \tau)\\
		&= \f 1 \pi \int_0^\pi \dd \theta \int_0^\infty \dd s e^{-s^2 \inv \qq /2} \pdf{}{s} \Sigma(s, \theta; \tau)
	\end{align*}
	where $\Upsilon(w, w';\tau) :=e^{-\f 1 2(\f{w^2}{1-c} + \f{(w')^2}{1+c})} - e^{-\f 1 2(\f{(w')^2}{1-c} + \f{w^2}{1+c})}$ with $c = \cos \tau$, and $\Sigma(s, \theta; \tau) := \phi(s \sin \theta) \phi(s \sin(\theta - \tau)).$
\end{lemma}
Of course, in the above lemma, the limit in the first equation is only necessary when $\tau = 0$ or $\tau = \pi/2$.
\begin{proof}[]
	\renewcommand{\qq}{q}
	Let $c := \cos \tau$ and
	$$\Gamma := \Wt \phi( \qq, cq) = \f 1 {2\pi \qq\sqrt{1 - c^2}} \int \dd \mathbf z \exp(-\mathbf z^T \Sigma^{-1} \mathbf z /2) \phi(z)\phi(z'),$$
	where $\Sigma = \begin{pmatrix}\qq & cq \\ cq & \qq \end{pmatrix}$.
	
	Our proof will have two portions: Symmetrization of the $\Gamma$ integral and trigonometric change of variables for evaluation.
	
	\noindent\textbf{Symmetrization.} $\Sigma$ is diagonalized by $\Omega = \f 1 {\sqrt{2q}}\begin{pmatrix}-1 & 1 \\ 1 & 1 \end{pmatrix}$,
	$$\Sigma = \Omega^T \Diag(1 - c, 1 + c) \Omega.$$
	By a change of variable $\mathbf w = \Omega \mathbf z$, so that $\dd \mathbf w = \inv \qq \dd \mathbf z$, we have
	\begin{align*}
	\Gamma &= \f 1 {2\pi \sqrt{1 - c^2}}\int \dd \mathbf w \exp(-\mathbf w^T \Diag(1 - c, 1 + c)^{-1}\mathbf w/2) \phi(\f {\sqrt \qq} {\sqrt{2}}(w' - w)) \phi(\f {\sqrt \qq} {\sqrt{2}}(w + w'))\\
	&= \f 1 {2\pi \sqrt{1 - c^2}} \int  \dd w \dd w' e^{-\f 1 2(\f{w^2}{1-c} + \f{(w')^2}{1+c})} \phi(\f {\sqrt{\qq}} {\sqrt{2}}(w' - w)) \phi(\f {\sqrt \qq} {\sqrt{2}}(w + w'))
	\end{align*}
	
	By a change of variable swapping $w$ with $w'$, we get
	
	\begin{align*}
	\Gamma &= -\f 1 {2\pi \sqrt{1 - c^2}} \int  \dd w \dd w' e^{-\f 1 2(\f{(w')^2}{1-c} + \f{w^2}{1+c})} \phi(\f {\sqrt{\qq}} {\sqrt{2}}(w + w')) \phi(\f {\sqrt \qq} {\sqrt{2}}(w' - w))
	\end{align*}
	
	Thus
	\begin{align*}
	2 \Gamma &= \f 1 {2\pi \sqrt{1 - c^2}} \int  \dd w \dd w' \Upsilon(w, w';\tau) \phi(\f {\sqrt{\qq}} {\sqrt{2}}(w + w')) \phi(\f {\sqrt \qq} {\sqrt{2}}(w' - w))
	\end{align*}
	where
	$$\Upsilon(w, w';\tau) = e^{-\f 1 2(\f{w^2}{1-c} + \f{(w')^2}{1+c})} - e^{-\f 1 2(\f{(w')^2}{1-c} + \f{w^2}{1+c})}.$$
	
	\begin{figure}
		\centering
		\includegraphics[height=.2\textheight]{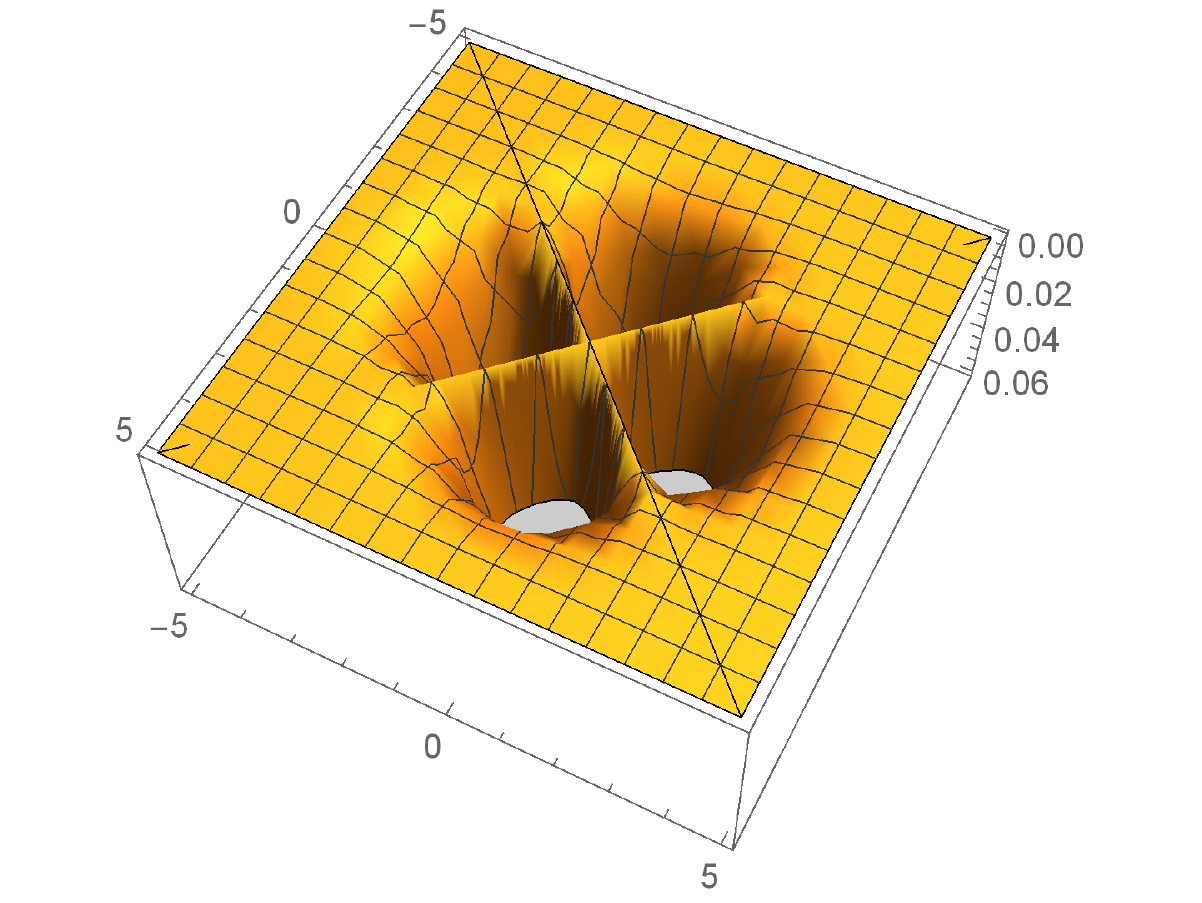}
		\caption{The integrand of $\Gamma$ after symmetrization. Here $c = .2$ and $\qq = 100$ and $\phi = \tanh$.}
		\label{fig:symmetrization_of_Gamma}
	\end{figure}
	Note that, by the antisymmetry of $\phi$, the integrand $K := \Upsilon(w, w';\tau)\phi(\ldots)\phi(\ldots)$ above has the symmetries $K(w, w') = K(w', w) = K(w, -w')$, and is everywhere nonnegative.
	\cref{fig:symmetrization_of_Gamma} displays a contour plot of $K$ for typical values of $\qq$ and $c$.
	So
	\begin{align*}
	\Gamma &= \f 1 {\pi \sqrt{1 - c^2}} \int_{w' \ge |w|} \dd w \dd w' K(w, w').
	\end{align*}
	This gives the first equation in the lemma.
	
	\noindent\textbf{Polar Coordinates.}
	Let $\f w {\sqrt{1 - c}} = r \cos \theta, \f {w'} {\sqrt{1 + c}} = r \sin \theta$, so that
	\begin{align*}
	w &= r \cos \theta \sqrt{1-c} = \sqrt 2 r \cos \theta \sin \f \tau 2\\
	w' &= r \sin \theta \sqrt{1+c} = \sqrt 2 r \sin \theta \cos \f \tau 2\\
	\dd w \dd w' &= \sqrt{1-c^2}r \dd r \dd \theta = (\sin^2 \tau) r \dd r \dd \theta.
	\end{align*}
	Then
	\begin{align*}
	\mathbf A &:= \int_{w' \ge |w|} e^{-(\f{w^2}{1-c} + \f{(w')^2}{1 +c})/2}\phi(\sqrt{\qq/2}(w+w'))\phi(\sqrt{\qq/2}(w'-w)) \dd w \dd w'\\
	&= \sin^2 \tau \int_{0}^\infty r\dd r e^{-r^2/2} \int_{\tau/2}^{\pi - \tau/2}\dd \theta  \phi(\sqrt \qq r \sin(\theta + \tau/2))\phi(\sqrt \qq r \sin(\theta - \tau/2)).
	\end{align*}
	Similarly, let $\f w {\sqrt{1 + c}} = r \cos \theta, \f {w'} {\sqrt{1 - c}} = r \sin \theta$, so that
	\begin{align*}
	w &= r \cos \theta \sqrt{1+c} = \sqrt 2 r \cos \theta \cos \f \tau 2\\
	w' &= r \sin \theta \sqrt{1-c} = \sqrt 2 r \sin \theta \sin \f \tau 2\\
	\dd w \dd w' &= \sqrt{1-c^2}r \dd r \dd \theta = (\sin^2 \tau) r \dd r \dd \theta,
	\end{align*}
	and
	\begin{align*}
	\mathbf B &= \int_{w' \ge |w|} e^{-(\f{w^2}{1+c} + \f{(w')^2}{1 - c})/2}\phi(\sqrt{\qq/2}(w+w'))\phi(\sqrt{\qq/2}(w'-w)) \dd w \dd w'\\
	&= -\sin^2 \tau \int_{0}^\infty r\dd r e^{-r^2/2} \int_{\pi/2 - \tau/2}^{\pi/2 + \tau/2}\dd \theta  \phi(\sqrt \qq r \cos(\theta + \tau/2))\phi(\sqrt \qq r \cos(\theta - \tau/2))\\
	&= -\sin^2 \tau \int_{0}^\infty r\dd r e^{-r^2/2} \int_{- \tau/2}^{\tau/2}\dd \theta  \phi(\sqrt \qq r \sin(\theta + \tau/2))\phi(\sqrt \qq r \sin(\theta - \tau/2)).
	\end{align*}
	Thus
	\begin{align*}
	\Gamma &= \f 1 {\pi \sqrt{1-c^2}} (\mathbf A - \mathbf B)\\
		&= \f 1 \pi \int_{0}^\infty r\dd r e^{-r^2/2} \int_{- \tau/2}^{\pi - \tau/2}\dd \theta  \phi(\sqrt \qq r \sin(\theta + \tau/2))\phi(\sqrt \qq r \sin(\theta - \tau/2))\\
		&= \f 1 \pi \int_{0}^\infty r\dd r e^{-r^2/2} \int_{0}^{\pi}\dd \theta  \phi(\sqrt \qq r \sin(\theta))\phi(\sqrt \qq r \sin(\theta - \tau)).
	\end{align*}
	This gives the second equation in the lemma, and a change of variables $s = \sqrt \qq r$ gives the third.
	
	For the fourth equality, we start from the third equality, and apply integration by parts:
	\begin{align*}
	&\phantom{{}={}}\f 1 \pi \int_0^\infty s \dd s \inv \qq e^{-s^2 \inv \qq /2} \int_0^\pi \dd \theta \Sigma(s, \theta; \tau)\\
	&= \f 1 \pi \int_0^\pi \dd \theta \int_0^\infty \dd s s \inv \qq e^{-s^2 \inv \qq /2}  \Sigma(s, \theta; \tau)\\
	&= \f 1 \pi \int_0^\pi \dd \theta \biggl(-e^{-s^2 \inv \qq /2} \Sigma(s, \theta; \tau)\biggr\rvert_{s=0}^\infty + \int_0^\infty \dd s e^{-s^2 \inv \qq /2} \pdf{}{s} \Sigma(s, \theta; \tau) \biggr)\\
	&= \f 1 \pi \int_0^\pi \dd \theta \int_0^\infty \dd s e^{-s^2 \inv \qq /2} \pdf{}{s} \Sigma(s, \theta; \tau).
	\end{align*}
	where the last equality follows because $\Sigma(0, \theta; \tau) = 0$ and $e^{-s^2\inv \qq / 2} \to 0$ as $s \to \infty$.
\end{proof}

In the following lemmas, the ``2'' is not important, and can be any arbitrary finite or infinite value.
\begin{lemma}
	Suppose a function $f: (0, 2) \to \R$ is $C^k$ on $(0,2)$.
	If $\lim_{x \downarrow 0} \p f i (x)$ exists and is finite for every $i \in [0, k]$, then $f$ can be extended to $[0, 2)$ such that one sided $i$th derivatives exist at 0 for all $i \in [0, k]$.
\end{lemma}
\begin{proof}
	Consider $\overline{\p f i}(0) := \p f i (1) - \int_0^1 \p f {i+1}(x) \dd x$ for $i \in [0, k-1]$, which naturally is also equal to $\p f i (\eps) - \int_0^\eps \p f {i+1}(x)\dd x$ for any $\eps > 0$.
	Certainly $\p f i (x) \to \overline{\p f i}(0)$ as $x \to 0$ if this limit exists --- and by assumption it does, for $0 \le i \le k -1$.
	Therefore, we can define the extension of $\p f i$ to $x = 0$ to be $\p f i (0) := \overline{\p f i}(0)$.
	But we need to check that for $i \in [0, k-1]$.
	\begin{align*}
	\lim_{\eps \to 0} \f 1 \eps(\p f i(\eps) - \p f i (0)) = \p f {i+1}(0)
	\end{align*}
	so that all one sided $i$th derivatives exist.
	But
	\begin{align*}
	\f 1 \eps(\p f i(\eps) - \p f i (0)) &= \f 1 \eps \int_0^\eps \p f {i+1}(x)\dd x\\
		&=\p f {i+1}(0) + \int_0^1 (\p f {i+1}(x) - \p f i (0)) \I(x \in [0, \eps]) \dd x
%		&= \f 1 \eps \int_0^\eps \p f {i+1}(0) + \int_0^x \p f {i+2}(y)\dd y\dd x\\
%		&= \p f {i+1}(0) +  \f 1 \eps \int_0^\eps \int_0^x \p f {i+2}(y)\dd y\dd x
	\end{align*}
	Since $\lim_{x \downarrow 0} \p f {i+1}(x) = \p f {i+1} (0)$, $\p f {i+1}(x) - \p f {i+1}(0)$ is bounded for small $x$, and by dominated convergence, $\int_0^1 (\p f {i+1}(x) - \p f i (0)) \I(x \in [0, \eps]) \dd x \to \int_0^1 0 \dd x = 0$ as $\eps \to 0$.
	Thus
	\begin{align*}
	\lim_{\eps \to 0} \f 1 \eps(\p f i(\eps) - \p f i (0)) &= \p f {i+1}(0)
	\end{align*}
	as desired.
\end{proof}

\begin{lemma}
	If $f: [0, 2) \to \R$ is $C^k$ on $(0, 2)$ and has one sided derivatives at 0 up to order $k$, then
	\begin{align*}
	f(\eps) &= f(0) + \eps \p f 1 (0) + \cdots + \f{\eps^{i-1}}{(i-1)!} \p f {i-1} (0) + O(\eps^i)
	\end{align*}
	for any $i \le k$.
\end{lemma}
\begin{proof}
	We have
	\begin{align*}
	f(\eps) &= f(0) + \int_0^\eps \p f 1 (x) \dd x\\
		&= f(0) + \eps \p f 1(0) + \int_0^\eps \p f 1(x) - \p f 1 (0) \dd x\\
		&= f(0) + \eps \p f 1(0) + \int_0^\eps \int_0^{x_0} \p f 2 (x_2) \dd x_2 \dd x_1\\
		&= f(0) + \eps \p f 1(0) + \f{\eps^2}{2} \p f 2(0) + \int_0^\eps \int_0^{x_1} \p f 2 (x_2) - \p f 2 (0) \dd x_2 \dd x_1\\
		&\ \ \mathbin{\vdots}\ \\
	f(\eps) &= f(0) + \eps \p f 1 (0) + \cdots + \f{\eps^{i-1}}{(i-1)!} \p f {i-1} (0) + \int_0^\eps \dd x_1 \int_0^{x_1} \dd x_2 \cdots \int_0^{x_{i-1}} \dd x_{i} \p f {i}(x_{i})
	\end{align*}
	for any $i\le k$.
	It suffices then to bound the size of the integral.
	Since $\p f i (x) \to \p f i (0)$ as $x \downarrow 0$ by assumption, $|\p f i (x_i)|$ is bounded by some constant $C$ on the integration region $\mathbb{A} := \{(x_1, \ldots, x_i): \eps \ge x_1 \ge \cdots \ge x_i\}$ for small enough $\eps$.
	Therefore,
	\begin{align*}
	&\phantom{{}={}}\int_0^\eps \dd x_1 \int_0^{x_1} \dd x_2 \cdots \int_0^{x_{i-1}} \dd x_{i} \p f {i}(x_{i})\\
	&= \int \p f {i}(x_i) \I(\vec x \in \mathbb A) \dd \vec x\\
	&\le C |\mathbb A|\\
	&= \Theta(\eps^i).
	\end{align*}
		
\end{proof}

As a corollary,
\begin{lemma}\label{lemma:smoothExtensionBoundary}
	If $f: (0, 2) \to \R$ is smooth on $(0, 2)$ and $\lim_{x \to 0}\p f i (x)$ exists and is finite for all $i$, then $f$ can be extended to $[0, 2)$ and be {\it one-sided smooth} at 0, and 
	\begin{align*}
	f(\eps) &= f(0) + \eps \p f 1 (0) + \cdots + \f{\eps^{i-1}}{(i-1)!} \p f {i-1} (0) + O(\eps^i)
	\end{align*}
	for any $i$.
\end{lemma}

\begin{lemma}\label{lemma:WtExtension}
	\renewcommand{\qq}{q}
	Let $\phi = \tanh$.
	For any fixed $c$, $\Wt \phi( \qq, cq)$ is smooth (infinitely differentiable) on $\qq \in (0, \infty)$.
	As a function of $Q := \inv \qq$, it can be extended smoothly to the point $Q = 0$, so that
	\begin{align*}
	\Wt \phi( \qq, cq ) &= \lim_{\qq' \to \infty} \Wt \phi( \qq', cq') + \inv \qq \lim_{\qq' \to \infty} \pd \Wt \phi( \qq', cq')/\pd (\qq')^{-1} + \cdots \\
		&\qquad+ \f{\qq^{-i+1}}{(i-1)!} \lim_{\qq' \to \infty} \pd^{i-1} \Wt \phi( \qq', cq')/\pd (\qq')^{-i+1} + O(\qq^{-i})
	\end{align*}
	for any $i \ge 0$.
	Furthermore, for $c$ bounded away from $1$, the constants hidden $O$ can be taken independent of $c$.
\end{lemma}
\begin{proof}
	\renewcommand{\qq}{q}
	\textbf{Smoothness on $(0, \infty)$.}
	By the third equation of \cref{lemma:WtPhiAntisymmetric}, for $Q \in (0, \infty) \iff \qq \in (0, \infty)$,
	\begin{align*}
	&\phantom{{}={}}\f 1 \pi \int_0^\infty s \dd s \left|\pdf{^n}{Q^n}\lp Q e^{-s^2 Q/2}\rp\right| \int_0^\pi \dd \theta |\phi(s \sin \theta)\phi(s \sin(\theta - \tau))|\\
	&\le \int_0^\infty s \dd s \left|\pdf{^n}{Q^n}\lp Q e^{-s^2 Q/2}\rp\right| < \infty,
	\end{align*}
	so by Leibniz's integral rule and a simple induction, all derivatives of $\Wt \phi( \qq, cq)$ against $Q$ exists for any $Q \in (0, \infty)$.
	
	\textbf{Extension to $Q = 0$.}
	By \cref{lemma:smoothExtensionBoundary}, it suffices to show that the limit of $
	\pdf{^k\Wt\phi( \qq, cq)}{Q^k}$ exists and is finite as $Q \to 0$, for all $k$.
	Let $\tau = \arccos c$.
	By the fourth equation of \cref{lemma:WtPhiAntisymmetric}, we have explicitly
	\begin{align*}
	\pdf{^k\Wt\phi( \qq, cq)}{Q^k}
	&= \f 1 \pi \int_0^\pi \dd \theta \int_0^\infty \dd s (\nicefrac{-s^2}{2})^k e^{-s^2 Q /2} \pdf{}{s} \Sigma(s, \theta; \tau)\\
	&= \f {(-2)^{-k}} \pi \int_0^\pi \dd \theta \int_0^\infty \dd s\ s^{2k} e^{-s^2 Q /2} \pdf{}{s} \Sigma(s, \theta; \tau)
	\end{align*}
	for any $Q \in (0, \infty)$.
	Note that for $\phi = \tanh$, $\dot \phi = \sech^2$,
	\begin{align*}
	\pdf{}{s} \Sigma(s, \theta; \tau)
		&= \sin \theta \dot \phi(s \sin \theta)\phi(s \sin (\theta - \tau)) + \sin(\theta-\tau)\phi(s \sin \theta)\dot \phi(s \sin (\theta - \tau)).\\
%	\left|\pdf{}{s} \Sigma(s, \theta; \tau)\right|
%		&\le |\sin \theta| \sech^2(s \sin \theta) + |\sin(\theta-\tau)| \sech^2(s \sin(\theta-\tau))
	\end{align*}
	We split the integral of $\pdf{^k \Wt \phi}{Q^k}$ as follows:
	\begin{align*}
	\pdf{^k\Wt\phi( \qq, cq)}{Q^k}
	&= \f {(-2)^{-k}} \pi \int_0^\pi \dd \theta \int_0^\infty \dd s\ s^{2k} e^{-s^2 Q /2} \sin \theta \dot \phi(s \sin \theta)\phi(s \sin (\theta - \tau))\\
	&\qquad + \f {(-2)^{-k}} \pi \int_0^\pi \dd \theta \int_0^\infty \dd s\ s^{2k} e^{-s^2 Q /2} \sin(\theta-\tau)\phi(s \sin \theta)\dot \phi(s \sin (\theta - \tau))
	\end{align*}
	We show that for each piece, the limit as $Q \to 0$ exists and is finite, for any $k$.
	This will prove the smooth extendability of $\Wt \phi$ to $Q = 0$.
	We will do this for the first piece; the second is similar.
	
	For $Q > 0$, the integrand is absolutely integrable, so we may switch the integrals.
	\begin{align*}
	&\phantom{{}={}} \int_0^\pi \dd \theta \int_0^\infty \dd s\ s^{2k} e^{-s^2 Q /2} \sin \theta \dot \phi(s \sin \theta)\phi(s \sin (\theta - \tau))\\
	&= \int_0^\infty \dd s s^{2k} e^{-s^2 Q/2}\int_0^\pi \dd \theta \sin \theta \dot \phi(s \sin\theta) \phi(s \sin (\theta - \tau))
	\end{align*}
	We now try to bound the inner integral by an exponentially decreasing term $e^{-s \mu}$ for some $\mu$; clearly, by monotone convergence on the outer integral as $Q \to 0$, this would show the limit of the integral exists and is finite.
	
	Because $\phi$ is odd and $\dot \phi$ is even, the inner integrand is negative on $\theta \in [0, \tau)$ and positive on $\theta \in (\tau, \pi]$.
	We will break up the inner integral as follows, for some fixed $\eps > 0$ satisfying $\tau - \eps > 0$ independent of $s$ (recall $\tau \in (0, \pi/2]$).
	\begin{align*}
	&\phantom{{}={}} \int_0^\pi \dd \theta \sin \theta \dot \phi(s \sin\theta) \phi(s \sin (\theta - \tau))\\
	&= \left(\int_0^\eps + \int_{\pi - \eps}^\pi \right) \dd \theta \sin \theta \dot \phi(s \sin\theta) \phi(s \sin (\theta - \tau)) + \int_\eps ^{\pi - \eps} \dd \theta \sin \theta \dot \phi(s \sin\theta) \phi(s \sin (\theta - \tau))
	\end{align*}
	Now because $\dot \phi(z) = \sech^2(z) \le 2 e^{-z}$, and $\sin \theta \ge \sin \eps$ on $\theta \in [\eps, \pi - \eps]$,
	\begin{align*}
	&\phantom{{}={}}\left|\int_\eps ^{\pi - \eps} \dd \theta \sin \theta \dot \phi(s \sin\theta) \phi(s \sin (\theta - \tau))\right|\\
	&\le 2\int_\eps ^{\pi - \eps} \dd \theta \exp(-s \sin\eps) \\
	&= 2(\pi - 2\eps) \exp(-s \sin\eps).
	\end{align*}
	
	For the other part:
	\begin{align*}
	&\phantom{{}={}}\int_{\pi - \eps}^\pi \dd \theta \sin \theta \dot \phi(s \sin\theta) \phi(s \sin (\theta - \tau))\\
		&= \int_\eps^{0} \sin(\pi - \theta) \dot \phi(s \sin \pi - \theta) \phi(s \sin (\pi - \theta - \tau)) \dd(\pi - \theta)\\
		&= \int_0^\eps \dd \theta \sin\theta \dot \phi(s \sin \theta) \phi(s \sin \theta + \tau)
	\end{align*}
	so that
	\begin{align*}
	&\phantom{{}={}}\left(\int_0^\eps + \int_{\pi - \eps}^\pi \right) \dd \theta \sin \theta \dot \phi(s \sin\theta) \phi(s \sin (\theta - \tau))\\
	&= \int_0^\eps \dd \theta \sin \theta \dot \phi(s \sin\theta) [\phi(s \sin (\tau + \theta)) - \phi(s \sin(\tau - \theta))]
	\end{align*}
	But by intermediate value theorem, $\phi(s \sin (\tau + \theta)) - \phi(s \sin(\tau - \theta)) = 2 \theta \pd \phi(s \sin(\tau + \theta))/\pd \theta |_{\theta = \psi} = 2 \theta \dot \phi(s \sin (\tau + \psi))s \cos(\tau + \psi)$ for some $\psi \in [- \theta, \theta]$.
	By the assumption on $\eps$, $\phi(s \sin (\tau + \theta)) - \phi(s \sin(\tau - \theta)) \le 2\eps \dot \phi(s \sin(\tau - \eps))s \cos(\tau - \eps).$
	Then
	\begin{align*}
	&\phantom{{}={}} \int_0^\eps \dd \theta \sin \theta \dot \phi(s \sin\theta) [\phi(s \sin (\tau + \theta)) - \phi(s \sin(\tau - \theta))]\\
	&\le \int_0^\eps \dd \theta \sin \theta \dot \phi(s \sin\theta) 2 \eps \dot \phi(s \sin(\tau - \eps))s \cos(\tau - \eps)\\
	&\le 2 \eps \dot \phi(s \sin(\tau - \eps))s \cos(\tau - \eps) O(1)
	\end{align*}
	Because $\tau -\eps > 0$ by assumption on $\eps$, and because $\dot \phi(z) = \exp(-\Theta_+(z))$, this quantity is $\exp(-\Theta_+(z))$, as desired (here $\Theta_+$ denotes a positive quantity).
	
	Thus
	\begin{align*}
	&\phantom{{}={}} \int_0^\pi \dd \theta \sin \theta \dot \phi(s \sin\theta) \phi(s \sin (\theta - \tau))\\
	&= \left(\int_0^\eps + \int_{\pi - \eps}^\pi \right) \dd \theta \sin \theta \dot \phi(s \sin\theta) \phi(s \sin (\theta - \tau)) + \int_\eps ^{\pi - \eps} \dd \theta \sin \theta \dot \phi(s \sin\theta) \phi(s \sin (\theta - \tau))\\
	&= \exp(-\Theta_+(s))
	\end{align*}
	and similarly for the other piece of $\pdf{^k \Wt \phi}{Q^k}$, so that
	\begin{align*}
	&\phantom{{}={}} \int_0^\infty \dd s s^{2k} e^{-s^2 Q/2}\int_0^\pi \dd \theta \sin \theta \dot \phi(s \sin\theta) \phi(s \sin (\theta - \tau))\\
	&= \int_0^\infty \dd s s^{2k} e^{-s^2\f Q 2 - \Theta_+(z)}\\
	&\to \int_0^\infty \dd s s^{2k} e^{-\Theta_+(z)}
	\end{align*}
	is finite as $Q \to 0$, by monotone convergence.
	
	\textbf{Independence of constant hidden in $O((\qq')^{-i})$.}
	The constant hidden is a function of the $\eps$ chosen above, which depend on $\tau$, but only to the extent that it must satisfy $\tau - \eps > 0$.
	As long as we are interested in a set $\mathcal C$ of $c$ that is bounded away from 1, the corresponding set of $\tau$ is bounded away from $0$, so $\eps$ can be taken to be some number smaller than all of the corresponding $\tau$.

\end{proof}

\begin{lemma} \label{lemma:Wt_le_arcsin}
	
	\renewcommand{\qq}{q}
	Suppose $\phi$ is tanh-like.
	Then for $c \in [0, 1]$,
	$$\Wt \phi( \qq, cq) \le \f 2 \pi \arcsin(c),$$
	and weakly increases to this upper bound as $\qq \to \infty$.
	Furthermore, 
	\begin{itemize}
		\item If $c = 0$ or 1, then equality holds regardless of $\qq$.
		\item If $c \in (0, 1)$ is held constant, $\f 2 \pi \arcsin(c) - \Wt \phi( \qq, cq) = \Theta(\inv \qq)$, where the hidden constants in $\Theta$ depend on $c$.
		But the constants can be made independent of $c$ if $c \in [\eps, 1-\eps]$ for some $\eps > 0$.
	\end{itemize}
\end{lemma}

\begin{proof}[]
	\renewcommand{\qq}{q}
	The cases of $c = 0$ or 1 are obvious by the definition of $\Wt$.
	So from here on we assume $c \in (0, 1)$.
	
	Let $\tau := \arccos c$.
	By the first equation of \cref{lemma:WtPhiAntisymmetric} and the assumption that $\phi$ is tanh-like, it is immediate that $\Wt \phi(\qq ,cq)$ is nondecreasing in $\qq$.
	By dominated convergence, using the second equation of \cref{lemma:WtPhiAntisymmetric}, we get
	\begin{align*}
	\lim_{\qq \to \infty} \Wt \phi(\qq ,cq) &= \f 1 \pi \int_{0}^\infty r\dd r e^{-r^2/2} (\pi - 2 \tau)\\
		&= \f {\pi - 2 \tau}{\pi}\\
		&= \f 2 \pi \arcsin c.
	\end{align*}
	
	Then the convergence rate is $O(\inv \qq)$ by \cref{lemma:WtExtension} and Taylor's theorem.
	Thus to show the convergence rate is $\Theta(\inv \qq)$, it suffices to show that $\mathbf D := \pdf{\Wt \phi( \qq, cq)}{Q} < 0$.
	But this is apparent from the first equation of \cref{lemma:WtPhiAntisymmetric}:
	For $\tau \in (0, \pi/2)$,
	\begin{align*}
	\mathbf D &= \f 1 {\pi \sin \tau} \int_{w' \ge |w|} \dd w \dd w' \Upsilon(w, w';\tau)(-\f 1 {2\sqrt 2} Q^{-3/2})\\
		&\qquad\qquad\qquad\quad \times[\dot\phi(\sqrt{\qq/2}(w + w'))\phi(\sqrt{\qq/2}(w'-w))\\
		&\qquad\qquad\qquad\qquad + \phi(\sqrt{\qq/2}(w + w'))\dot\phi(\sqrt{\qq/2}(w'-w))]\\
		& < 0
	\end{align*}
	since $\Upsilon$ is positive on the integration domain, and $\dot \phi$ and $\phi$ are both positive for positive arguments, by the assumption of $\phi$ being tanh-like.
	
	\textbf{Independence of the constants in $\Theta(\inv \qq)$ from $c$ when $c \in [\eps, 1-\eps]$.}
	By \cref{lemma:WtExtension}, the upper constant can be made independent from $c$.
	Since $\mathbf D$ is monotonically decreasing in $c$ (or monotonically increasing in $\tau$) and $|\mathbf D|$ is monotonically increasing in $c$ (or monotonically decreasing in $\tau$), we have $|\mathbf D| > |\mathbf D|\bigg]_{c = \eps}$, which can be taken to be the lower constant in $\Theta(\inv \qq)$.
\end{proof}
\cref{fig:verifywtanhasymptotics} verifies empirically that the subleading term in $\Wt\tanh( q, cq)$ is linear in $\inv q$, for constant $c$.

\begin{figure}
	\centering
	\includegraphics[height=.15\textheight]{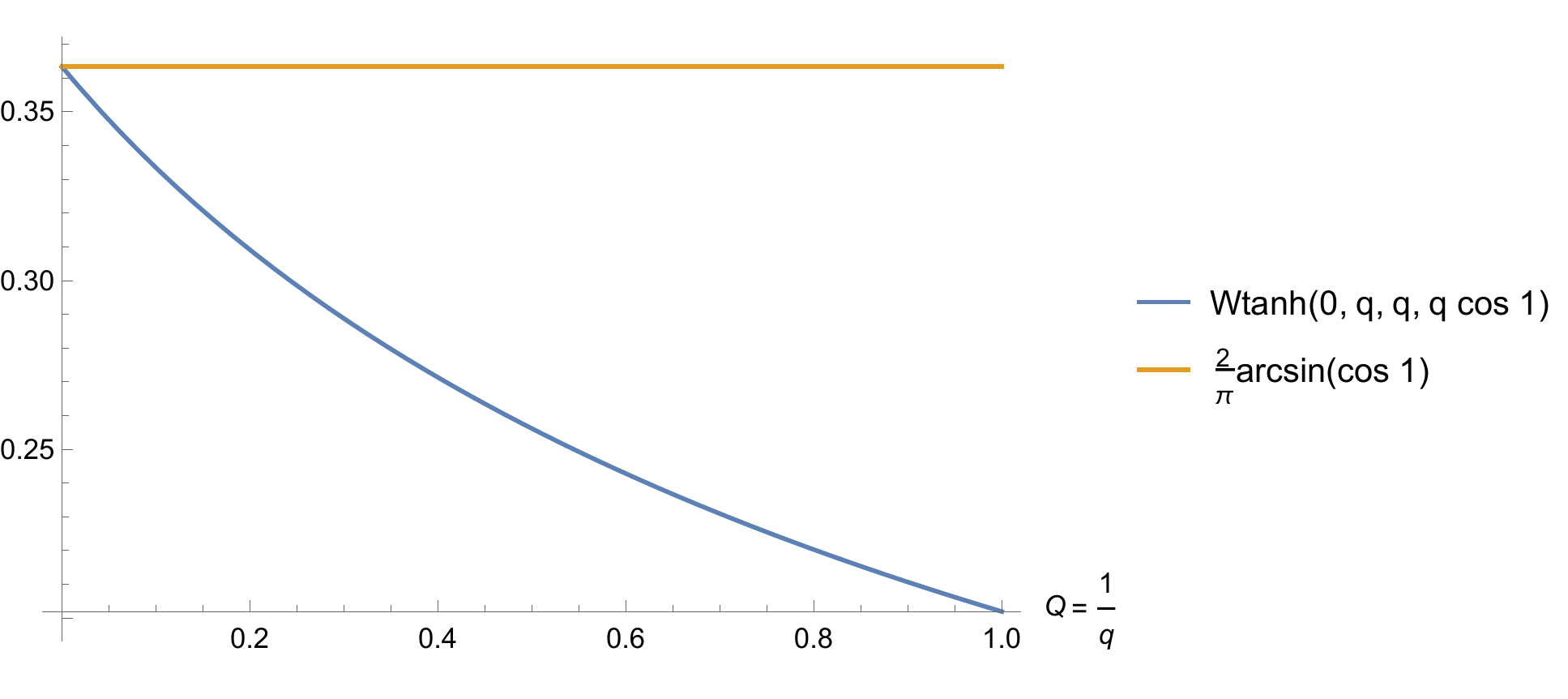}
	\caption{We verify empirically that the subleading term in $\Wt\tanh( q, cq)$ is linear in $\inv q$, for constant $c$.
	Indeed, observe that the curve of of $\Wt \tanh$ intersects the y-axis at an angle.}
	\label{fig:verifywtanhasymptotics}
\end{figure}

\edynamics*
\begin{proof}

We have by \cref{lemma:Wt_le_arcsin},
\begin{align*}
\gamma &= \f 2 \pi \arcsin(\lambda /\qq) - \Theta(\qq^{-1}) + \prv \gamma.
\end{align*}
% Since $\qq \ge \sigma_w^2 \omega (l-1) + \sigma_b^2$ for any $\omega < 1$
Since $\qq = \sigma_w^2 \prv \pp + \sigma_b^2$ by \cref{thm:p_q_linear}, and $\lambda = \sigma_w^2 \prv \gamma + \sigma_b^2$ by \cref{thm:lambda_gamma_recurrence},
\begin{align*}
\gamma	&= \f 2 \pi \arcsin\lp\f{\sigma_w^2 \prv \gamma + \sigma_b^2}{\sigma_w^2 \prv \pp + \sigma_b^2}\rp - \Theta(\qq^{-1}) + \prv \gamma.
\end{align*}
% and the difference between the two sides goes to 0 as $l \to \infty$ by \cref{lemma:Wt_le_arcsin}.

We claim that $\p \gamma l \to \infty$ as $l \to \infty$.
Otherwise, there is some $C$ such that $\p \gamma l \le C$ for all $l$.
For large enough $l$, $\p \pp l \ge \omega l$ for any $\omega < 1$ and $\arcsin \lp \f{C}{\sigma_w^2 \p \pp {l-1} + \sigma_b^2} \rp = \Theta(1/l)$ by linearization of $\arcsin$.
Thus $\p \gamma l = \Theta(\log l)$, but this contradicts our assumption that $\gamma$ is bounded.
This proves our claim.

Therefore, for large enough $l$,
$$\f{\sigma_w^2 \prv \gamma + \sigma_b^2}{\sigma_w^2 \prv \pp + \sigma_b^2} = \prv \gamma/\prv \pp + \Theta(l^{-1}).$$

% If $\sigma_b^2$ is small, then the quantity in the parenthesis is approximately $\prv \gamma / \prv \pp$.

\cref{fig:arcsin-vs-x} shows $\f2\pi\arcsin x$ vs $x$.
One sees that 1 is an unstable fixed point; if $\ee < 1 - \epsilon$, then $\f 2 \pi \arcsin \ee < 1 - \epsilon - \delta$ for some $\delta$.
Thus $\cc$ drops monotonically until some threshold under which the linearization of $\arcsin$, $\arcsin x = x + \Theta(x^3)$, is applicable.
So for large enough $l$,
\begin{align*}
\gamma - \prv \gamma &= \f 2 \pi \arcsin (\prv \gamma / \prv \pp + \Theta(\inv l)) - \Theta(\inv l)\\
	&= \f 2 \pi  \prv \gamma / \prv \pp + O(\inv l)
\end{align*}
As $\p \pp l \sim l$ by \cref{thm:p_q_linear}, this difference equation has solution $\gamma = \Omega(l^{\f 2 \pi -\eps}), O(l^{\f 2 \pi + \eps})$ for any $\eps$ by using the dynamics of \cref{lemma:alphaDeltaDynamics} to upper and lower bound this difference equation.
% Then the difference equation $\gamma - \prv \gamma = \f 2 \pi \arcsin \f{\prv \gamma}{\prv \pp}$ has the linearization $\gamma - \prv \gamma = \f 2 \pi \prv \gamma/\prv \pp$.
% Under the approximation $\p \pp l \approx l$ by \cref{thm:p_q_linear}, we obtain a solution $\p \gamma l = O(l^{\f 2 \pi})$.
\end{proof}

\subsubsection{Backward Dynamics}
\dalethExpSqrtTanh*
\begin{proof}
	The $\sigma_w = 0$ case is obvious.
	We will assume $\sigma_w > 0$ from here on.
	
	Let $\p \pp l = b_0 l + b_1 l^{1/2} + b_2 \log l + O(1)$.
	Then for $D = \f 2 3 \sqrt{\f 2 \pi}$, we have (implicitly applying \cref{lemma:Vt_dot_tanh_asymptotics} and \cref{lemma:power_sum_asymptotics}),
	\begin{align*}
	\qq^{-1/2} &= \sigma_w^{-1} b_0^{-1/2} l^{-1/2}(1 - b_1 b_0^{-1} \inv 2 l^{-1/2} - b_2 b_0^{-1} 2^{-1} l^{-1} \log l + O(l^{-1}))\\
	\Vt \dot \phi( \qq) & = D \qq^{-1/2} + \Theta(\qq^{-3/2})\\
		&= D\sigma_w^{-1} b_0^{-1/2} l^{-1/2}(1 - b_1 b_0^{-1} \inv 2 l^{-1/2} - b_2 b_0^{-1} 2^{-1} l^{-1} \log l + O(l^{-1}))\\
	\log(B \Vt \dot \phi( \qq) + 1) &= B D \sigma_w^{-1} b_0^{-1/2} l^{-1/2}\\
							&\phantom{={}}- (BD \sigma_w^{-1} b_0^{-3/2} b_1 2^{-1} +B^2 D^2 \sigma_w^{-2} b_0^{-1}2^{-1})l^{-1} + \Theta(l^{-3/2}\log l)\\
	\sum_{r=1}^l \log(B \Vt \dot \phi( \p \qq r) + 1) &= 2B D \sigma_w^{-1} b_0^{-1/2} l^{1/2}
\\
		&\phantom{={}} -(BD \sigma_w^{-1} b_0^{-3/2} b_1 2^{-1} +B^2 D^2 \sigma_w^{-2} b_0^{-1}2^{-1})\log l + O(1)\\
	\end{align*}
	In our case, we have $b_0 = 1, b_1 = -2C \sigma_w^{-1}, b_2 = C^2 \sigma_w^{-2},  B = \sigma_w^2, C = \sqrt{\f 2 \pi}$, which gives
	$$\sum_{r=1}^l \log(B \Vt \dot \phi( \p \qq r) + 1) = \f 4 3 \sqrt{\f 2 \pi} \sigma_w l^{1/2} + (\f 4 {3\pi} - \sigma_w^2 \f 4 {9\pi}) \log l + O(1).$$
	so that
	\begin{align*}
	\p \daleth {m}/\p \daleth l &= \exp\left[\f 4 3 \sqrt{\f 2 \pi} \sigma_w (\sqrt l - \sqrt m) + (\f 4 {3\pi} - \sigma_w^2 \f 4 {9\pi}) (\log l - \log m) + O(1)\right]\\
	\end{align*}
\end{proof}

\dalethExpSqrtTanhAllGrad*
\begin{proof}
	The $\sigma_w = 0$ case is obvious.
	We will assume $\sigma_w > 0$ from here on.
	
	As in the proof of \cref{thm:dalethExpSqrtTanh},
	$$\Vt \dot \phi( \qq) = D\sigma_w^{-1} b_0^{-1/2} l^{-1/2} + \Theta(\inv l)$$
	where $D = \f 2 3 \sqrt{\f 2 \pi}$.
	Thus by \cref{thm:dalethRecReduced},
	$$	\log(\p \daleth {m}/\p \daleth l) = \f 4 3 \sqrt{\f 2 \pi} \sigma_w (\sqrt l - \sqrt m) + (\f 4 {3\pi} - \sigma_w^2 \f 4 {9\pi}) (\log l - \log m) + O(1)$$
	$$\log(\p \chi m_b/\p \chi l _b) =\f 4 3 \sqrt{\f 2 \pi} \sigma_w (\sqrt l - \sqrt m) + (\f 4 {3\pi} - \f 1 2 - \sigma_w^2 \f 4 {9\pi}) (\log l - \log m) + O(1)$$
	Similarly, since $\pp = l + \Theta(\sqrt l)$ by \cref{thm:p_q_linear}, we have
	$$\log(\p \chi m_w/\p \chi l _w) = \f 4 3 \sqrt{\f 2 \pi} \sigma_w (\sqrt l - \sqrt m) + (\f 4 {3\pi} + \f 1 2 - \sigma_w^2 \f 4 {9\pi}) (\log l - \log m) + O(1)$$
\end{proof}

\subsection{Tanh: Full Residual Network}
\subsubsection{Forward Dynamics}

\pIsLinearTanh*
\begin{proof}
	The $\sigma_w = 0$ case is obvious.
	We will assume $\sigma_w > 0$ from here on.
	
	As in \cref{thm:p_q_linear}, $\pp$ will have expansion $\pp = b_0 l + b_1 l^{1/2} + b_2 \log l + O(1).$
	Then, for $C = \sqrt{\f 2 \pi}$,
	\begin{align*}
	\qq^{-1/2} &= \sigma_w^{-1} b_0^{-1/2} l^{-1/2}(1 - b_1 b_0^{-1} \inv 2 l^{-1/2} - b_2 b_0^{-1} 2^{-1} l^{-1} \log l + O(l^{-1}))\\
	\sum_{r=1}^l \Vt \phi( \p \qq r) &= \sum_{r=1}^l 1 - C(\p \qq r)^{-1/2} + \Theta((\p \qq r)^{-3/2})\\
		&= l - 2 C \sigma_w^{-1} b_0^{-1/2} l^{1/2} + C \sigma_w^{-1} b_1 b_0^{-3/2} 2^{-1} \log l + O(1)\\
	\p \pp l &= \sigma_v^2 \sum_{r=1}^l + \sigma_a^2 l\\
		&= (\sigma_v^2 + \sigma_a^2) l - 2C \sigma_v^2 \sigma_w^{-1} b_0^{-1/2} l^{1/2} + C \sigma_v^2 \sigma_w^{-1} b_1 b_0^{-3/2} 2^{-1} \log l + O(1)
	\end{align*}
	which yields
	\begin{align*}
	b_0 &= \sigma_v^2 + \sigma_a^2\\
	b_1 &= -2C \sigma_v^2 \sigma_w^{-1} b_0^{-1/2} = \f{-2C \sigma_v^2 \sigma_w^{-1}}{\sqrt{\sigma_v^2 + \sigma_a^2}}\\
	b_2 &= \f{-C^2 \sigma_v^4 \sigma_w^{-2}}{(\sigma_v^2 + \sigma_a^2)^2}
	\end{align*}
\end{proof}

\begin{lemma}
Suppose $\phi$ is tanh-like.
Then
$$ \gamma \le \sigma_v^2 \f 2 \pi \arcsin\lp { \lambda}/{\qq}\rp + \sigma_a^2 + \prv \gamma,$$
and
$$\sigma_v^2 \f 2 \pi \arcsin\lp { \lambda}/{\qq}\rp + \sigma_a^2 + \prv \gamma - \gamma = \Theta(\qq^{-1}).$$
% where the difference of the two sides goes to 0 as $l \to \infty$.
\end{lemma}
\begin{proof}
Similar to the proof of \cref{lemma:Wt_le_arcsin}.
\end{proof}

\begin{lemma}\label{lemma:timeDependentConvergence}
	Let $u^* \in [0, 1)$.
	Let $f_t: [0, 1) \to [0, 1]$ be a continuous function for each $t \in \N$, to each of which we associate two numbers $0 \le a_t \le u^* \le b_t \le 1$.
	Suppose for each $t$, $f_t(u) > u$ for all $u \in [0, a_t)$ and $f_t(u) < u$ for all $u \in (b_t, 1)$.
	Assume that for each $u$, $f_t(u) - u \to 0$ as $t \to \infty$ uniformly over $u$.
	If $a_t \nearrow u^*$ and $b_t \searrow u^*$, then for any $u_0 \in [0, 1)$, the dynamics $u_t = f_t(u_{t-1})$ has a limiting point.
	Furthermore, either $u_t \to u^*$ or $u_t$ eventually converges monotonically (decreasing or increasing) to a limit point.
\end{lemma}
\begin{proof}
%	We need to show that for every starting point $u_0 \in [0, 1)$, $u_t \to u^*$.
	Fix a $u_0 \in [0, 1)$.
	If $u_t \to u^*$ then we are done.
	Otherwise, suppose there is a neighborhood $[u^* - \eps, u^* + \eps]$ such that for an infinite sequence $t_1, t_2, \ldots$, $u_{t_i} \not \in [u^* - \eps, u^* + \eps]$.
	WLOG assume $u_{t_i} < u^* - \eps$ for all $i$ and $(t_i)_i$ is the sequence of all $t$s that satisfy this inequality.
	
	If $(t_i)_i$ contains $\{s: s \ge N\}$ for some $N$, then for some $M > N$, for every $t > M$, $a_t > u^* - \eps > u_t$.
	By assumption, $u_t$ is monotonic for all $t > M$ but is bounded above.
	Thus $u_t$ has a fixed point $\hat u \le u^* - \eps$ as desired.
	
	Now assume there are infinite $i$s such that $t_i - 1 \not= t_{i-1}$ (i.e. $t_i - 1$ is not part of the sequence $(t_i)_i$).
	We will show that this case is contradictory.
	Take $T$ large enough such that $a_t > u^* - \eps/2$ and $|f_t(u) - u| < \eps/4$ for all $u$ and for all $t \ge T$ ($T$ exists by premise).
	Let $j$ be the smallest index such that $t_j > T$ and $t_j - 1 \not = t_{j-1}$.
	By the definition of $j$, $u_{t_j - 1} \ge u^* - \eps$.
	If $u_{t_j - 1} \ge u^* - \eps/2$, then by definition of $T$, $u^* - \eps > u_{t_j} = f_{t_j}(u_{t_j - 1}) > u_{t_j-1} - \eps/4 > u^* - 3\eps/4 > u^* - \eps$, a contradiction.
	If $u^* - \eps \le u_{t_j - 1} \le u^* - \eps/2$, then by the definition of $T$, $u_{t_j-1} \le a_{t_j-1}$ so that $u_{t_j } = f_{t_j}(u_{t_j - 1}) > u_{t_j - 1} \ge u^* - \eps,$ a contradiction.
	
	The ``furthermore'' claim is clear from our proof above.
\end{proof}

\eDynamicsFullResTanh*
\begin{proof}
	The $\sigma_w = 0$ case is obvious.
	We will assume $\sigma_w > 0$ from here on.

If $\sigma_a = 0$, then $\ee^*$ as defined above is 0, and $\ee = \f \gamma \pp$ decreases as $\Theta(l^{\f 2 \pi - 1})$ to 0, by the same reason as before.

So from now on suppose $\sigma_a > 0$.
We apply \cref{lemma:timeDependentConvergence} first to show that $\ee$ converges.
We have
\begin{align*}
\sigma_v^2 \Wt \phi( \qq, cq) + \sigma_a^2 &= \ee \pp - \prv \ee \prv \pp\\
	&= \ee \pp - \prv \ee \pp + \prv \ee \pp - \prv \ee \prv \pp\\
	&= (\ee - \prv \ee) \pp + \prv \ee(\pp - \prv \pp)\\
	&= (\pp - \prv \pp)[(\ee - \prv \ee)\f \pp {\pp - \prv \pp} + \prv \ee]\\
\f{\sigma_v^2 \Wt\phi( \qq, cq) + \sigma_a^2}{\sigma_v^2 \Vt \phi( \qq) + \sigma_a^2} &= (\ee - \prv \ee) \f \pp {\pp - \prv \pp} + \prv \ee\\
\f {\pp- \prv \pp}\pp \left[\f{\sigma_v^2 \Wt \phi + \sigma_a^2}{\sigma_v^2 \Vt \phi + \sigma_a^2} - \prv \ee\right] &= \ee - \prv \ee
\end{align*}
If we define $f_l(u) := \f{\p \pp l - \p \pp {l-1}}{\p \pp l}\left[\f{\sigma_v^2 \Wt \phi( \p \qq l, \p \cc l \p \qq l) + \sigma_a^2}{\sigma_v^2 \Vt \phi( \p \qq l) + \sigma_a^2} - u\right] + u$ (the LHS of the above), then $f_l(u) - u = O(l^{-1})$ uniformly for all $u$ because $\p \pp l = \Theta(l)$, $\p \pp l - \p \pp {l-1} = \Theta(1)$, and the part in the bracket is $O(1)$, with constants all (able to be taken) independent of $u$.
We divide $[0, 1)$ into the following intervals $I_1 = [1, 1/2), I_2 = [1/2, 3/4), I_3 = [3/4, 7/8), \ldots$.
For each $I_k$, it is clear that the trajectories of $\p \ee l = f_l(\p \ee {l-1})$ with $\p \ee 0 \in I_k$ will fall into some interval $J_k$ bounded away from 1 for all $l \ge L$, for large enough $L$ (dependent on $k$).
Then we can apply \cref{lemma:cExpansion,lemma:vtanhSqrtConvergence,lemma:Wt_le_arcsin} to get $f_l(u) = \f{\p \pp l - \p \pp {l-1}}{\p \pp l}\left[\f{\sigma_v^2\f 2 \pi \arcsin(u) + \sigma_a^2}{\sigma_v^2 + \sigma_a^2} - u + o(1)\right] + u$ where the constants in $o(1)$ is uniform for all $\p \ee 0 \in I_k$.
For $u < \ee^*$ (as defined in the theorem statement), $\f{\sigma_v^2\f 2 \pi \arcsin(u) + \sigma_a^2}{\sigma_v^2 + \sigma_a^2} > u$ and for $u > \ee^*$, $\f{\sigma_v^2\f 2 \pi \arcsin(u) + \sigma_a^2}{\sigma_v^2 + \sigma_a^2} < u$ (see \cref{fig:arcsin-vs-x}).
Thus as $l \to \infty$, the $o(1)$ term gets smaller and smaller, and this monotonicity holds for $f_l(u) - u = \left[\f{\sigma_v^2\f 2 \pi \arcsin(u) + \sigma_a^2}{\sigma_v^2 + \sigma_a^2} - u + o(1)\right] > 0$ (resp. $<0$) on larger and larger intervals $[0, a_l]  \cap J_k$ (resp. $[b_l, 1) \cap J_k$).
This proves all the preconditions for \cref{lemma:timeDependentConvergence}, which yields that $I_k$ converges to a limit point.
As this argument is independent of $k$, we have that for all $\p \ee 0 \in [0, 1)$, $\p \ee l$ converges.

Now we solve for the limit point.
%applying $f_m \circ f_{m-1} \circ \cdots \circ f_{1}$ for large enough $m$ contracts the input space $[0, 1)$ away from 0 and from 1 (because $\sigma_a > 0$).

%To do so, we need to prove

%then $\lim_{l \to \infty} \p \gamma l = \infty$, so that for large $l$, $\lambda/\qq \to \prv\gamma/\prv \pp$.
Suppose $\ee$ has limit point $\ee^\dagger$ (possibly different from $\ee^*$ described in the theorem); if we express $\p \gamma l = (\ee^\dagger + \p \eps l ) \p \pp l$, then
\begin{align*}
\sigma_v^2 \Wt \phi( \qq, cq) + \sigma_a^2 &= \gamma - \prv \gamma\\
				&= (\ee^\dagger + \eps)\pp - (\ee^\dagger + \prv \eps) \prv \pp\\
				&= \ee^\dagger(\pp - \prv \pp) + \eps \pp - \prv \eps \prv \pp\\
\f{\sigma_v^2 \Wt \phi( \qq, cq) + \sigma_a^2}{\sigma_v^2 \Vt \phi( \qq) + \sigma_a^2} &= \ee^\dagger + \eps + (\eps - \prv \eps)\f{\prv \pp}{\pp - \prv \pp}
\end{align*}
As $l \to \infty$, $\cc \sim \ee \to \ee^\dagger$, and $\Wt\phi( \qq, \ee^\dagger \qq) \to \f 2 \pi \arcsin(\ee^\dagger)$, and $\Vt\phi( \qq) \to 1$.
Additionally, $\prv \pp /(\pp - \prv \pp) = \Theta(l)$ and $\eps = o(1)$ so that $\eps - \prv \eps = o(l^{-1})$.
Then we have, taking limits $l \to \infty$,
\begin{align*}
\f{\sigma_v^2 \f 2 \pi \arcsin(\ee^\dagger)+\sigma_a^2}{\sigma_v^2 + \sigma_a^2} &= \ee^\dagger.
\end{align*}
Since $f_l$ (as defined above) repels points away from 1, the only solution for $\ee^\dagger$ when $\p \ee 0 < 1$ is $\ee^\dagger = \ee^*$ as specified in the theorem statement.

We defer the proof of the convergence rate to $\ee^*$ to \cref{thm:TanhFullResConvergenceRate}.

\begin{figure}
	\centering
	\includegraphics[height=.15\textheight]{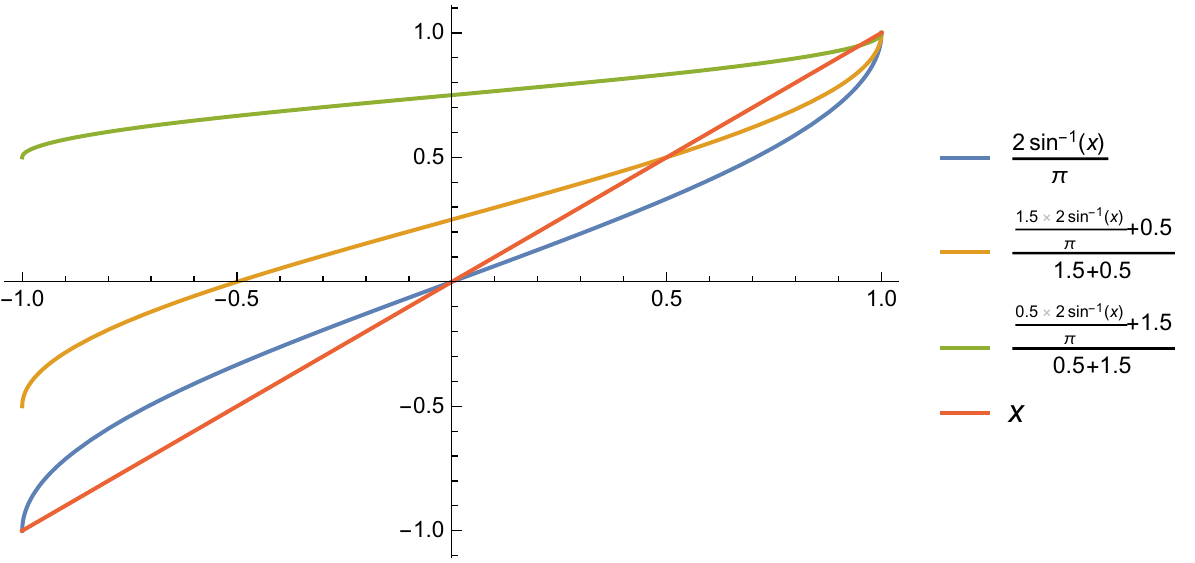}
	\caption{Graph of $y(\ee) = \f 1 {\sigma_v^2 + \sigma_a^2}[\sigma_v^2\f 2 \pi \arcsin(\ee) + \sigma_a^2]$ for various $\sigma_v$ and $\sigma_a$.}
	\label{fig:arcsin-vs-x}
\end{figure}
\end{proof}

\begin{lemma} \label{lemma:fixed_point_ineq}
Let $\ee^*$ be the stable fixed point determined by $\sigma_a$ and $\sigma_v$.
Then as long as $\sigma_v > 0$,
\begin{align*}
\f 2 \pi \f 1 {\sqrt{1 - (\ee^*)^2}} \f{\sigma_v^2 }{\sigma_v^2 + \sigma_a^2} \in (\f 1 2, \f 2 \pi]
\end{align*}
\end{lemma}
\begin{proof}
	\renewcommand{\EE}{{\ee^*}}
	Write $\rho := \f{\sigma_a^2}{\sigma_v^2}$.
	By definition of $\EE$, we get
	\begin{align*}
	\EE &= (1 - \rho) \f 2 \pi \arcsin \EE + \rho\\
	\rho = &= \f{\EE - \f 2 \pi \arcsin \EE}{1 - \EE}
	\end{align*}
	Substituting $\rho$ into the expression in question, it follows that we want to show
	\begin{align*}
	\f 2 \pi (1 - \EE^2)^{-1/2} (1 + \rho)^{-1}=  \f 2 \pi (1 - \EE^2)^{-1/2} \left(\f{1 - \f 2 \pi \arcsin \EE}{1 - \EE}\right)^{-1}\in (\f 1 2, \f 2 \pi]
	\end{align*}
	for $\EE \in [0, 1)$ (the endpoint at 1 is not included since $\sigma_v > 0$.
	But this is
	\begin{align*}
	&\phantom{={}} \f 2 \pi (1 - \EE)^{1/2} (1 + \EE)^{-1/2} (1 - \f 2 \pi \arcsin \EE)^{-1}.
	\end{align*}
	Set $g(\EE)$ to be this expression.
	We could proceed by finding critical points, but a simple plot \cref{fig:deltastarBound} shows that $g$ is decreasing on $[0, 1)$, with extremal values at the end points:
	$$g(\EE) \in [\lim_{\EE \to 1} g(\EE), g(0)), \quad \text{for }\EE \in [0, 1).$$
	Obviously $g(0) = \f 2 \pi$.
	For the limit, we note that $\arcsin \EE$ has an asymptotic expansion $\f \pi 2 - \sqrt 2 (1 - e)^{1/2} + \Theta((1-e)^{3/2})$ at 1, so that $(1 - \EE)^{1/2}(1 - \f 2 \pi \arcsin \EE)^{-1} \to \dfrac{\pi}{2 \sqrt 2}$, and $g(\EE) \to \dfrac{1}{2}$ as $\EE \to 1$.

	\begin{figure}[]
		\centering
		\includegraphics[width=.4\textwidth]{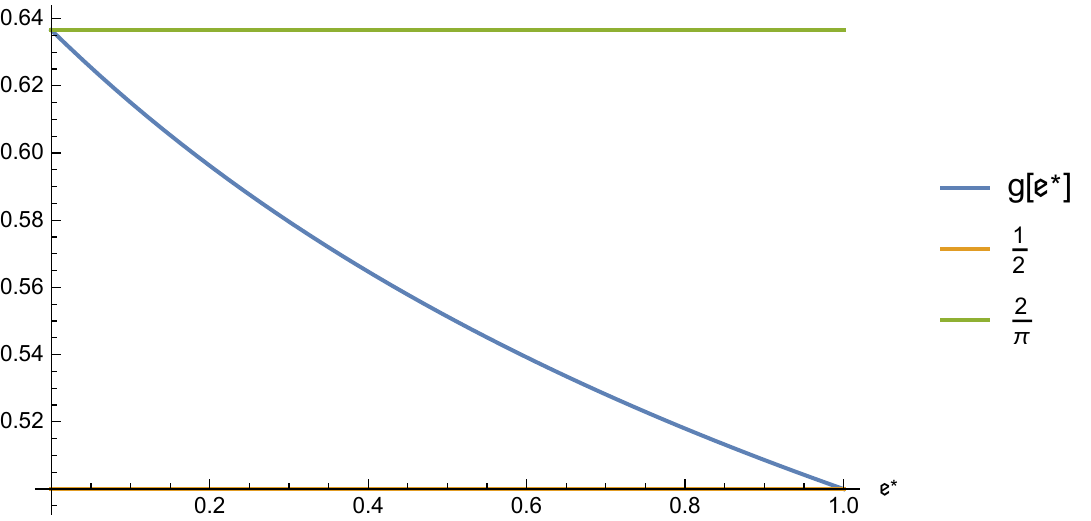}
		\caption{Plot of $g(\EE)$ in the proof of \cref{lemma:fixed_point_ineq}}
		\label{fig:deltastarBound}
	\end{figure}

\end{proof}

\begin{thm}\label{thm:TanhFullResConvergenceRate}
If $\p \ee 0 < 1$, then $|\p \ee l - \ee^*|$ is $\Omega(l^{-\delta^* - \varepsilon})$ and $O(l^{-\delta^* + \varepsilon})$ for any $\varepsilon > 0$, where
$$\delta^* := 1 - \f 2 \pi \f 1 {\sqrt{1 - (\ee^*)^2}} \f{\sigma_v^2 }{\sigma_v^2 + \sigma_a^2} \in [1-\f 2 \pi, \f 1 2),$$
where the bounds on the right follow from \cref{lemma:fixed_point_ineq}.
\end{thm}
\begin{proof}

Define $\omega(q, c) = \f 2 \pi \arcsin(c) - \Wt \tanh( q, c q)$.
By \cref{lemma:Wt_le_arcsin}, for large enough $l$, $\cc$ is close to $\ee^*$ bounded away from 0 or 1, so that $\omega(\qq, \cc) = \Theta(\inv \qq)$ with the constant hidden in $\Theta$ independent of $\cc$.
Additionally, by \cref{lemma:vtanhSqrtConvergence}, $1 - \Vt\tanh( \qq) = \Theta(\qq^{-1/2})$.
Therefore,
\begin{align*}
(\ee^* + \epsilon)\pp &= \sigma_v^2 (\f 2 \pi \arcsin(\ee^* + \prv\epsilon) - \omega(\qq, \cc))+ \sigma_a^2 + \prv \gamma\\
&= \sigma_v^2 \f 2 \pi [\arcsin(\ee^*) + \f{\prv \epsilon}{\sqrt{1 - (\ee^*)^2}} + \Theta(\prv \eps^2)] - \Theta(\inv l) + \sigma_a^2 + \prv \gamma\\
&= \ee^*(\sigma_v^2 + \sigma_a^2) + (\ee^* + \prv \eps) \prv \pp + \sigma_v^2 \f 2 \pi\f{\prv \epsilon}{\sqrt{1 - (\ee^*)^2}} + \Theta(\prv \eps^2) - \Theta(\inv l)\\
\ee^*(\pp - \prv \pp - \sigma_v^2 - \sigma_a^2) &= \prv \eps \prv \pp - \eps \pp + \sigma_v^2 \f 2 \pi\f{\prv \epsilon}{\sqrt{1 - (\ee^*)^2}} + \Theta(\prv \eps^2)- \Theta(\inv l)\\
\ee^*\sigma_v^2(\Vt \phi( \qq) - 1) &=  \prv \eps \prv \pp - \eps \pp + \sigma_v^2 \f 2 \pi\f{\prv \epsilon}{\sqrt{1 - (\ee^*)^2}} + \Theta(\prv \eps^2) - \Theta(\inv l)\\
\eps &= \f 1 \pp (\ee^*\sigma_v^2(1 - \Vt \phi( \qq)) + \Theta(\prv \eps^2) - \Theta(\inv l) + \prv \eps (\prv \pp+ \sigma_v^2 \f 2 \pi\f{1}{\sqrt{1 - (\ee^*)^2}}))\\
	&= \Theta(l^{-3/2}) + \prv \eps(1 - \p \delta l / l)\\
% 	&= \Theta(l^{-3/2}) + \Theta(\prv \eps^2/l) + \prv \eps (1 - \f{\sigma_v^2 \Vt \phi( \qq) + \sigma_a^2 -\sigma_v^2 \f 2 \pi\f{1}{\sqrt{1 - (\ee^*)^2}}}\pp)
\end{align*}
where
\begin{align*}
\p \delta l &= \f l \pp (\sigma_v^2 \Vt \phi( \qq) + \sigma_a^2 -\sigma_v^2 \f 2 \pi\f{1}{\sqrt{1 - (\ee^*)^2}}) + \Theta(\prv \eps/l)\\
	&= (1 + \Theta(l^{-1/2}))(\sigma_v^2 (1 - \Theta(l^{-1/2})) + \sigma_a^2 -\sigma_v^2 \f 2 \pi\f{1}{\sqrt{1 - (\ee^*)^2}})/(\sigma_v^2 + \sigma_a^2) + \Theta(\prv \eps / l)\\
	&= \delta^* + O(l^{-1/2}),
\end{align*}
where $\delta^* := 1 - \f 2 \pi \f 1 {\sqrt{1 - (\ee^*)^2}} \f{\sigma_v^2 }{\sigma_v^2 + \sigma_a^2}$, which is positive by \cref{lemma:fixed_point_ineq}.
By taking the $\delta$ of \cref{lemma:alphaDeltaDynamics} to be $\delta^* + \varepsilon$ or $\delta^* - \varepsilon$ respectively for lower and upper bounding the dynamics of $\p \eps l$, the solution $\p \eps l$ is $\Omega(l^{-\delta^* - \varepsilon})$ and $O(l^{-\delta^* + \varepsilon})$ for any $\varepsilon > 0$ since $\f 1 2 > \delta^*$.

\end{proof}

\subsubsection{Backward Dynamics}

\dalethExpSqrtTanhFullRes*
\begin{proof}
	The $\sigma_w = 0$ case is obvious.
	We will assume $\sigma_w > 0$ from here on.
	
As in the proof of \cref{thm:dalethExpSqrtTanh},
\begin{align*}
	\log(\p \daleth {m} / \p \daleth l) &= 2B D \sigma_w^{-1} b_0^{-1/2} (\sqrt l - \sqrt m)
\\
&\phantom{={}} -(BD \sigma_w^{-1} b_0^{-3/2} b_1 2^{-1} +B^2 D^2 \sigma_w^{-2} b_0^{-1}2^{-1})(\log l - \log m) + O(1)
\end{align*}
where $B = \sigma_v^2 \sigma_w^2, D = \f 2 3 \sqrt{\f 2 \pi},$
\begin{align*}
b_0 &= \sigma_v^2 + \sigma_a^2\\
b_1 &= \f{-2C \sigma_v^2 \sigma_w^{-1}}{\sqrt{\sigma_v^2 + \sigma_a^2}}\\
b_2 &= \f{-C^2 \sigma_v^4 \sigma_w^{-2}}{(\sigma_v^2 + \sigma_a^2)^2}.
\end{align*}
with $C = \sqrt{\f 2 \pi}$.
This simplifies to the desired form.
\end{proof}

\dalethExpSqrtTanhFullResAllGrad*
\begin{proof}
	Similar to \cref{thm:dalethExpSqrtTanhAllGrad}.
\end{proof}

\subsection{$\alpha$-ReLU: Full Residual Network}

The following can be checked readily
\VtPsiAlpha*

Since $\dot \psi_\alpha = \alpha \psi_{\alpha - 1}$, we have as a corollary,
\begin{lemma}\label{lemma:Vt_dotpsi_alpha}
If $\alpha > \f 1 2$, then
$\Vt \dot \psi_\alpha( q) = \alpha^2 \cV_{\alpha - 1} q^{\alpha-1}$.
\end{lemma}
As a special case, when $\alpha = 1$, $\cV_\alpha = \f 1 2$.

The following is a trivial computation, but useful for many simplifications.
\begin{lemma}
	$\cV_{\alpha+1}/\cV_\alpha = 2\alpha+1$.
\end{lemma}

\subsubsection{Forward Dynamics}

\label{sec:AlphaReluForwardProofs}
\begin{thm}\label{thm:pDynamic1ReLU}
Suppose we have the nonlinearity $\phi = \psi_1$.
Then $\p \pp l = \Theta((1 + \sigma_v^2 \sigma_w^2/2)^l)$, with the hidden constant depending on the initial condition.
\end{thm}
\begin{proof}
We have
\begin{align*}
\pp &= \f 1 2\sigma_v^2(\sigma_w^2 \prv \pp + \sigma_b^2) + \sigma_a^2 + \prv \pp\\
	&= (\f 1 2\sigma_v^2\sigma_w^2 + 1) \prv \pp + \f 1 2 (\sigma_v^2 \sigma_b^2 + \sigma_a^2).
\end{align*}
By the standard method of characteristic equation, we get that
$$\p \pp l = A + C B^l$$
where $A = - \f{\sigma_a^2 + \sigma_b^2 \sigma_v^2}{\sigma_v^2 \sigma_w^2}$, $B = 1 + \f{\sigma_v^2 \sigma_w^2}{2}$, and $C$ is a coefficient determined by initial conditions.
\end{proof}

\begin{thm}\label{thm:pDynamicLT1ReLU}
	Suppose $\alpha < 1$.
	We have the following asymptotic expansion
	$$\p \pp l = K_1 l^\oalpha + R(l)$$
	where the remainder term
	$$R(l) \sim \begin{cases}
	-K_2 l^{\aalpha} \log l	&\text{if $\alpha > \f 1 2$}\\
	(C-K_2) l \log l			&\text{if $\alpha = \f 1 2$ and $K_2 \not=C$}\\
	\f{C(1-\alpha)}{1-2\alpha} l	& \text{if $ \alpha < \f 1 2$}
	\end{cases}$$
	where $K_1 = [\sigma_v^2 \sigma_w^{2\alpha} \cV_\alpha (1 - \alpha)]^{\f 1 {1-\alpha}}, K_2 = \f 1 2 [\sigma_v^2 \cV_\alpha \sigma_w^{2\alpha}]^{\oalpha} (1-\alpha)^{\aalpha - 1} \alpha$ and $C = \sigma_a^2$.
\end{thm}
\cref{fig:55reluverifyleadingcoeffp} verifies the leading coefficient and the exponent of the leading term.
\begin{proof}
The difference equation governing the evolution of $\pp$ is
$$\pp - \prv \pp = A(\prv \pp + B)^\alpha + C$$
where $A =\sigma_v^2 \cV_\alpha \sigma_w^{2\alpha}$, $B = \sigma_b^2/\sigma_w^2$, and $C = \sigma_a^2$.
Then \cref{lemma:polyDynamicsConstant} yields the result.
\end{proof}

\begin{figure}
	\centering
	\includegraphics[height=.2\textheight]{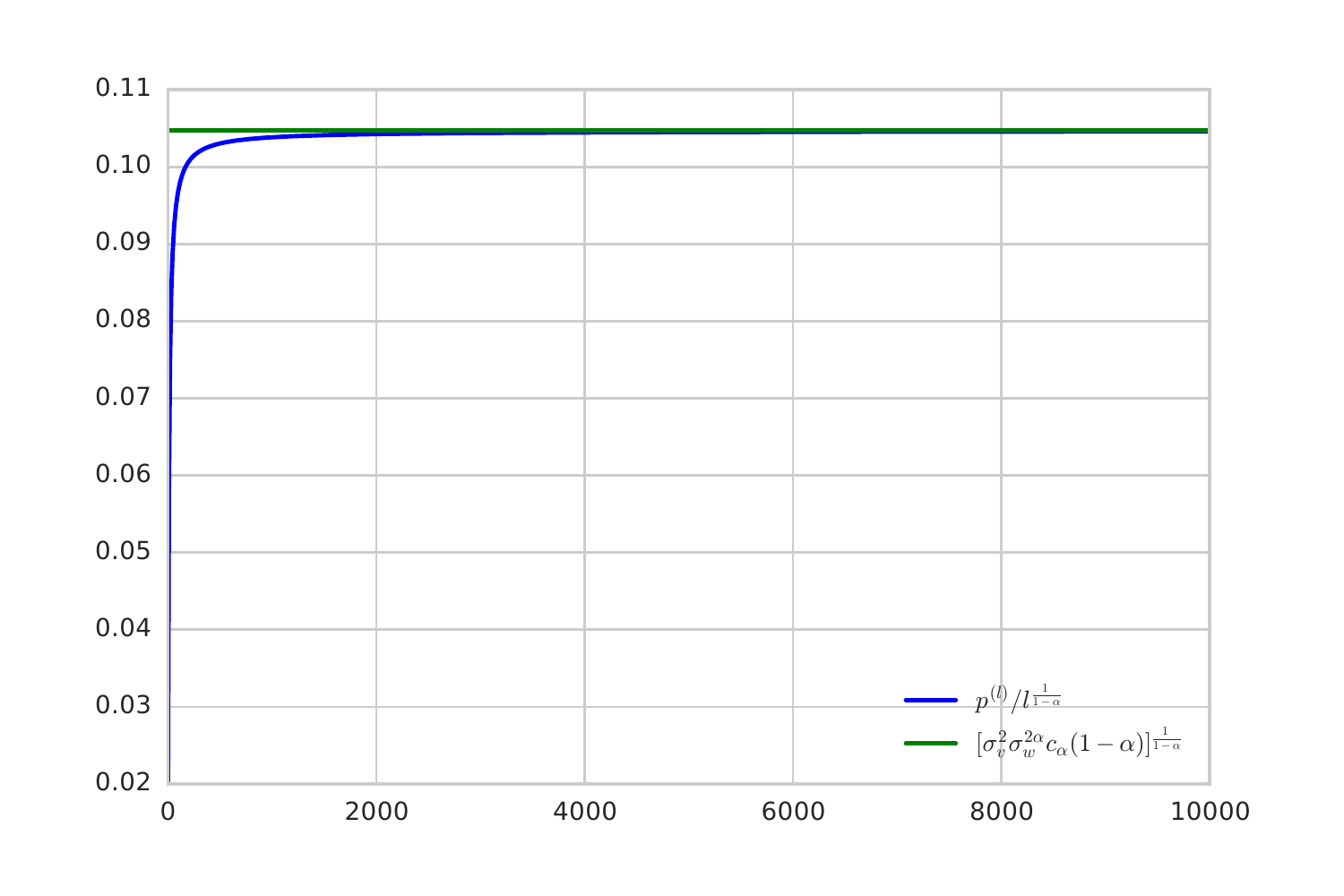}
	\caption{Verification of leading term of \cref{thm:pDynamic1ReLU} for $\alpha = 0.55$.}
	\label{fig:55reluverifyleadingcoeffp}
\end{figure}

\cref{thm:pDynamicLT1ReLU} combined with \cref{thm:pDynamic1ReLU} gives the following result.
\pDynamicAlphaReLU*

By \cite{cho_kernel_2009}, we know that $\Wt\psi_\alpha( q, q c) = \Vt\psi_\alpha( q) \JJ_\alpha(c)$, where $\JJ_\alpha(c) = J_\alpha(\arccos c)$ and
\begin{equation}
J_\alpha(\theta) := \f 1 {2\pi \cV_\alpha} (\sin \theta)^{2\alpha + 1} \Gamma(\alpha + 1)\int_0^{\pi / 2} \f{\dd \eta \cos^\alpha \eta}{(1 - \cos \theta \cos \eta)^{1 + \alpha}}.\tag{$\triangle$}\label{eqn:JalphaIntegralFormula}
\end{equation}

Note that $\JJ_\alpha(c) \in (-\infty, \infty)$ for $\alpha \in (-1, \infty)$ and any $c \in (0, 1)$, even though $\Vt \psi_\alpha$ is only defined for $\alpha > -1/2$.
\begin{figure}
\centering
\includegraphics[width=.3\textwidth]{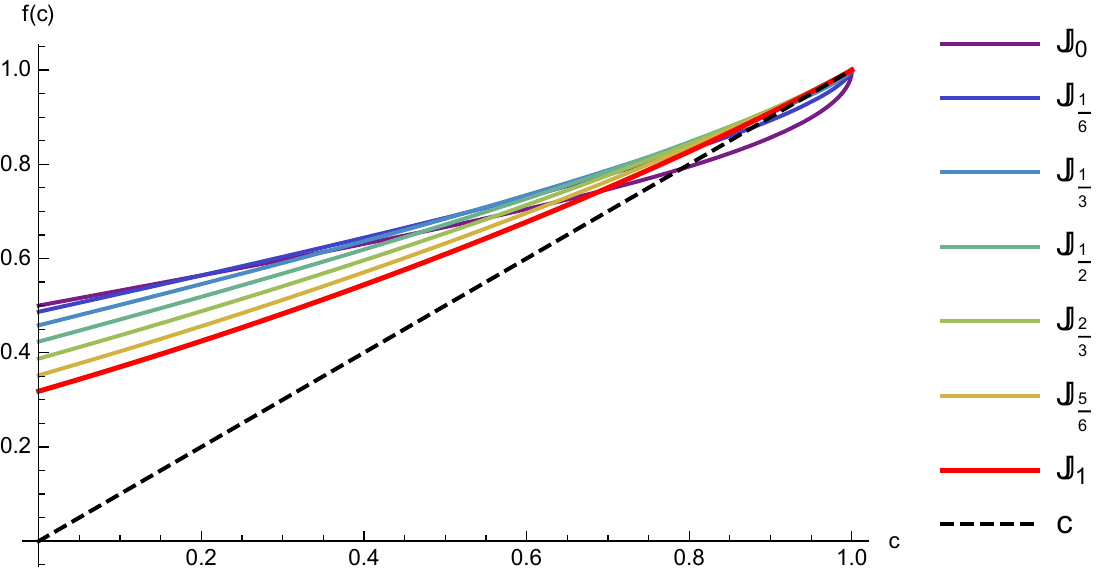}
~
\includegraphics[width=.3\textwidth]{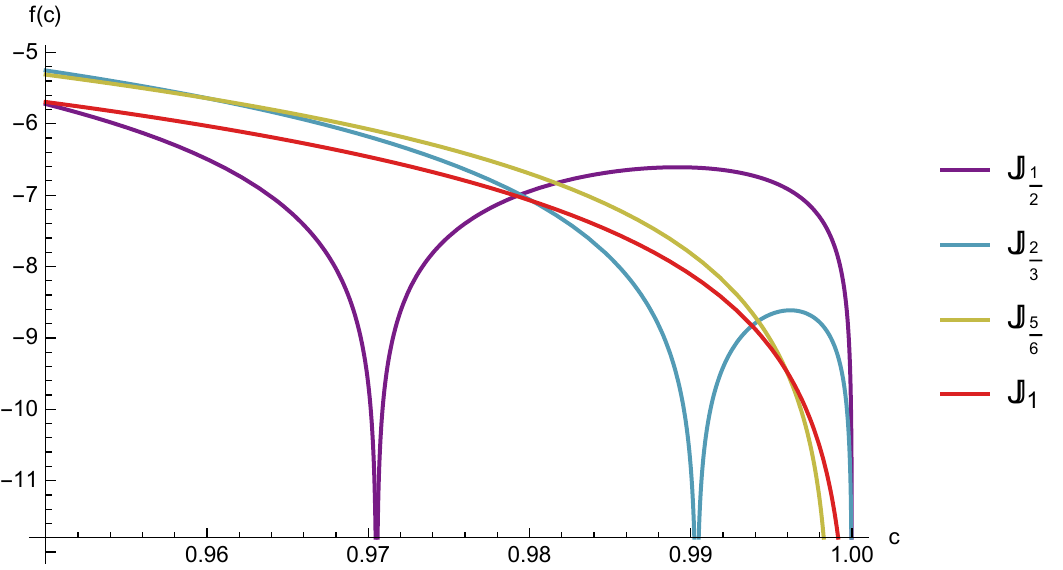}
~
\includegraphics[width=.3\textwidth]{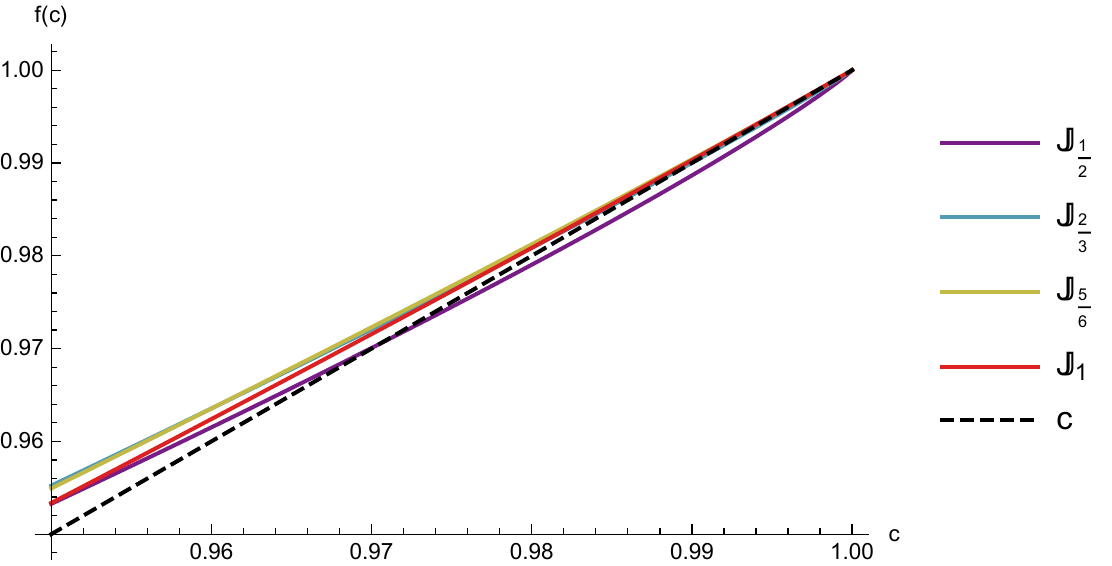}
\caption{(a) $\JJ_\alpha$ for different $\alpha$s and the identity function.
From this plot, it looks like $\JJ_\alpha(c) \ge c$ and $\dot\JJ_\alpha(c) \le 1$ for all $\alpha \in (\f 1 2, 1]$ with equality iff $c = 1$, but this is misleading.
(b) shows $|\JJ_\alpha(c) - c|$ in log scale.
Where the curves dip below the x-axis indicate points where $\JJ_\alpha(c) = c$.
We see that in fact every $\JJ_\alpha$ has a solution $\JJ_\alpha(c) = c$ for a $c < 1$, when $\alpha < 1$.
(c) Furthermore, at each such $c$, $\dot\JJ_\alpha < 1$.}
\label{fig:jjj_vs_id}
\end{figure}

\cref{fig:jjj_vs_id} shows a comparison of $\JJ_\alpha$ for different $\alpha$s along with the identity function.
By \cite[Lemma 11]{daniely_toward_2016}, $\JJ_\alpha$ is an increasing and convex function as long as $\psi_\alpha^2$ is Gaussian-integrable, which is precisely when $\alpha > -1/2$.
We can compute $\JJ_\alpha(1) = \Wt \psi_\alpha(q, q)/\Vt\psi_\alpha(q) = 1$, and $\JJ_\alpha(0) = \Wt \psi_\alpha(q, 0)/\Vt\psi_\alpha(q) = \Vt \psi_{\alpha/2}(q)^2/\Vt\psi_\alpha(q) = \cV_{\alpha/2}^2/\cV_\alpha = \f 1{2\sqrt \pi} \f{\Gamma(\f \alpha 2 + \f 1 2)^2}{\Gamma(\alpha + \f 1 2)}$.
We record these observations as a lemma.
\begin{lemma}\label{lemma:basicJalpha}
	$\JJ_\alpha(c)$ is an increasing and convex function for each $\alpha > -1/2$ on $c \in [0, 1]$.
	$\JJ_\alpha(1) = 1$ and $\JJ_\alpha(0) = \f 1{2\sqrt \pi} \f{\Gamma(\f \alpha 2 + \f 1 2)^2}{\Gamma(\alpha + \f 1 2)}$.
\end{lemma}

\begin{figure}
\centering
\includegraphics[width=.3\textwidth]{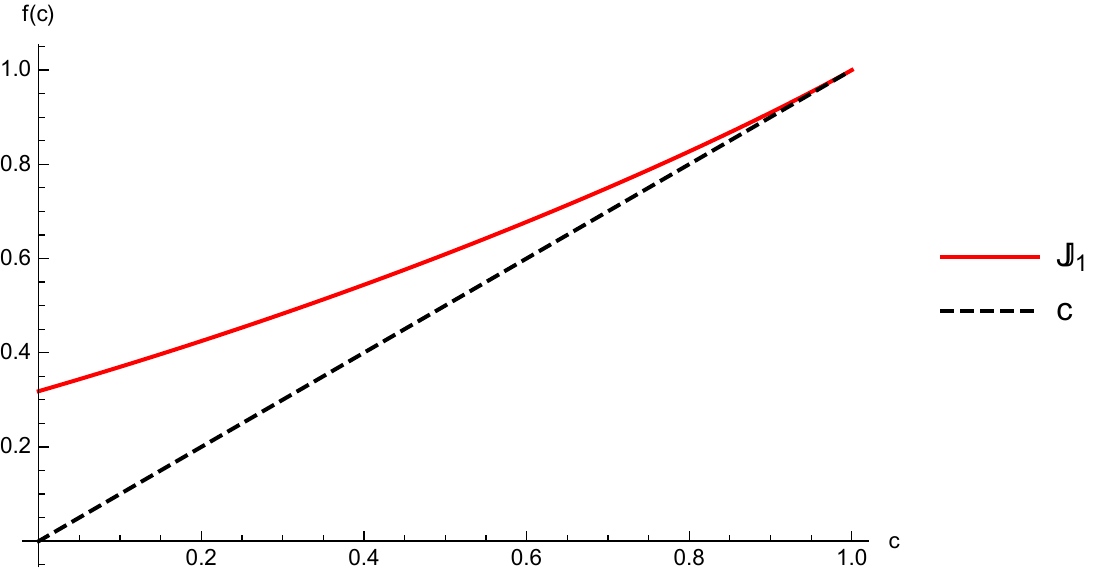}
\caption{$\JJ_1$ vs identity}
\label{fig:jjj1_vs_id}
\end{figure}
For $\alpha = 1$, \citet{cho_kernel_2009} computed
$$\JJ_1(c) = \f 1 \pi (\sqrt{1 - c^2} + (\pi - \arccos(c))c).$$
\cref{fig:jjj1_vs_id} shows a plot of $\JJ_1$ vs identity.
It has derivative $\dot\JJ_1(c) = 1 - \f 1 \pi \arccos c$, which shows that $\dot\JJ_1(c) < 1$ with equality iff $c = 1$, and consequently $\JJ_1(c) \ge c$ with equality iff $c = 1$.
At the same time, $\dot \JJ_1(c) \ge 0$ with equality iff $c = -1$, so $\JJ_1$ is increasing on $[-1, 1]$.
It has an asymptotic expansion $\JJ_1(1 - \varepsilon) = 1 - \varepsilon + \f {2\sqrt 2}{3\pi} \eps^{3/2} + \Theta(\eps^{5/2})$ at 1.

The zeroth Bessel function of the second kind is defined by $\bessel(z) = \int_1^\infty e^{-z x} (x^2 - 1)^{-1/2} \dd x$.
It is one of the fundamental solutions to the homogeneous differential equation $x^2 \dot y + x \dot y - x^2 y = 0$.
The following lemma shows that $J_\alpha$ can be expressed in terms of $\bessel$.
\begin{lemma}\label{lemma:JalphaBessel}
	For any $\alpha > -1$,
	$J_\alpha(\theta) = \f 1 {2\pi \cV_\alpha}\sin^{2\alpha+1} \theta \int_0^\infty \dd x \bessel(x) e^{x\cos \theta} x^\alpha$
\end{lemma}
\begin{proof}[]
	\citet{cho_kernel_2009} gave the expression
	$$2\pi \cV_\alpha J_\alpha(\theta) = \csc \theta \int_0^\infty \dd u \int_0^\infty \dd v e^{-(u^2+v^2 - 2uv \cos \theta)/2\sin^2\theta} u^\alpha v^\alpha.$$
	Note that the integrand is symmetric in $u$ and $v$.
	\newcommand{\sector}{{\mathsf V}}
	Thus, if $\sector = \{(u, v): u, v \ge 0 \And v \ge u\}$, then
	$$2\pi \cV_\alpha J_\alpha(\theta) = 2\csc \theta \int_\sector \dd u \dd v e^{-(u^2+v^2 - 2uv \cos \theta)/2\sin^2\theta} u^\alpha v^\alpha.$$
	\renewcommand{\pp}{\mathbbm{p}}
	\renewcommand{\qq}{\mathbbm{q}}
	Now make the change of variables from $\sector$ to $\{(\pp, \qq): \qq \ge 2 \sqrt{\pp}\}$:
	\begin{align*}
	\pp &= uv	&	\dd \pp &= v \dd u + u \dd v\\
	\qq &= u + v&	\dd \qq &= \dd u + \dd v\\
	\dd \pp \dd \qq &= (v - u) \dd u \dd v	&	\dd u \dd v &= (\qq^2 - 4 \pp)^{-1/2} \dd \pp \dd \qq
	\end{align*}
	so that we have
	$$2 \pi \cV_\alpha J_\alpha(\theta) = 2\csc \theta \int_0^\infty \dd \pp e^{\pp(1 + \cos \theta)\csc^2\theta} \pp^\alpha \int_{2\sqrt{\pp}}^\infty \dd \qq e^{-\qq^2\csc^2 \theta}(\qq^2 - 4 \pp)^{-1/2}.$$
	The inner integral in $\qq$ can be expressed in terms of $\bessel$ by a change of variable $x = \qq^2/2\sqrt \pp$:
	\begin{align*}
	2 \pi \cV_\alpha J_\alpha(\theta) &= 2\csc \theta \int_0^\infty \dd \pp e^{\pp(1 + \cos \theta)\csc^2\theta} \pp^\alpha \f 1 2 e^{-\pp\csc^2 \theta} \bessel(\pp\csc^2 \theta)\\
		&= \csc\theta \int_0^\infty \dd \pp \bessel(\pp\csc^2\theta) e^{\pp \cos \theta \csc^2\theta} \pp^\alpha\\
		&= \sin^{2\alpha + 1} \theta \int_0^\infty \dd x \bessel(x) e^{x \cos \theta} x^\alpha
	\end{align*}
\end{proof}

Define $L_\alpha(\theta) = 2 \pi \cV_\alpha J_\alpha(\theta)\csc^{2\alpha+1}\theta = \int_0^\infty \dd x \bessel(x) e^{x \cos \theta} x^\alpha$.
\begin{lemma}\label{lemma:LAlphaRec}
	If $\alpha > 1$, then
	$$L_\alpha(\theta) = \csc^2 \theta [(2\alpha-1) \cos \theta L_{\alpha-1}(\theta) + (\alpha-1)^2 L_{\alpha - 2}(\theta)].$$
\end{lemma}
\begin{proof}[]
	We will prove this claim for $\theta < 1$, and by continuity this also proves the case $\theta = 1$.
	As remarked above, $\bessel(z) = \ddot \bessel(z) + \inv z \dot \bessel(z).$
	Thus
	\begin{align*}
	L_\alpha(\theta) &= \int_0^\infty \dd x (\ddot \bessel(x) + \inv x \dot \bessel(x)) e^{x \cos\theta} x^\alpha\\
		&=\dot \bessel e^{x\cos\theta} x^\alpha \rvert_0^\infty + \bessel e^{x\cos \theta} x^{\alpha - 1} \rvert_0^\infty\\
			&\phantom{={}} - \int \dd x[\cos \theta e^{x \cos \theta} x^\alpha + \alpha e^{x \cos \theta} x^{\alpha - 1}] \dot \bessel\\
			&\phantom{={}} - \int \dd x[\cos \theta e^{x \cos \theta} x^{\alpha - 1} + (\alpha-1) e^{x\cos\theta} x^{\alpha - 2}] \bessel
	\end{align*}
	Asymptotically, $\bessel(z) \sim \sqrt{\f{\pi}{2z}} e^{-z}$ as $z \to \infty$ and $\bessel(z) \sim -\ln(z)$ as $z \searrow 0$, and $\dot\bessel(z) \sim -\sqrt{\f{\pi}{2z}} e^{-z}$ as $z \to \infty$ and $\dot\bessel(z) \sim -\inv z$ as $z \searrow 0$.
	Thus, as $\alpha > 1$,
	\begin{align*}
	\dot \bessel e^{x\cos\theta} x^\alpha \rvert_0^\infty = -\lim_{x \to \infty} \sqrt{\pi/2} e^{-x(1 - \cos \theta)}x^{\alpha - 1} + \lim_{x\searrow 0} e^{x\cos\theta} x^{\alpha-1}  = 0\\
	\bessel e^{x\cos\theta} x^{\alpha-1} \rvert_0^\infty = -\lim_{x \to \infty} \sqrt{\pi/2} e^{-x(1 - \cos \theta)}x^{\alpha - 2} + \lim_{x\searrow 0} e^{x\cos\theta} x^{\alpha-1} \ln x  = 0
	\end{align*}
	So
	\begin{align*}
	L_\alpha(\theta) &= -\cos \theta L_{\alpha-1}(\theta) - (\alpha-1)L_{\alpha-2}(\theta) - \int \dd x[\cos \theta e^{x \cos \theta} x^\alpha + \alpha e^{x \cos \theta} x^{\alpha - 1}] \dot \bessel
	\end{align*}
	Via another integration by parts, the integral on the right is
	\begin{align*}
	&\phantom{={}} \cos \theta e^{x \cos \theta} x^\alpha \bessel \rvert_0^\infty + \alpha e^{x\cos \theta} x^{\alpha-1} \bessel \rvert_0^\infty\\
	&\phantom{={}} - \int \dd x[\cos^2 \theta e^{x \cos \theta} x^\alpha + 2 \alpha \cos \theta e^{x \cos \theta} x^{\alpha - 1} + \alpha (\alpha - 1) e^{x \cos \theta} x^{\alpha - 2}] \bessel\\
	&= - [\cos^2 \theta L_\alpha(\theta) + 2 \alpha \cos \theta L_{\alpha-1}(\theta) + \alpha(\alpha-1) L_{\alpha-2}(\theta)]
	\end{align*}
	where the evaluation terms vanish just like before.
	Altogether, we have
	\begin{align*}
	L_\alpha(\theta) &= \cos^2 \theta L_\alpha(\theta) + (2\alpha-1)\cos\theta L_{\alpha-1}(\theta) + (\alpha-1)^2 L_{\alpha-2}(\theta)\\
				 &= \csc^2 \theta[(2\alpha-1) \cos \theta L_{\alpha-1}(\theta) + (\alpha-1)^2 L_{\alpha-2}(\theta)]
	\end{align*}
\end{proof}
As a corollary we get
\begin{lemma}\label{lemma:JalphaRec}
	Suppose $\alpha > 1$.
	Then
	\begin{align*}
	J_\alpha(\theta) &= \cos \theta J_{\alpha-1}(\theta) + (\alpha-1)^2 (2\alpha-1)^{-1}(2\alpha-3)^{-1} \sin^2 \theta J_{\alpha-2}(\theta)\\
	\JJ_\alpha(c) &=  c \JJ_{\alpha-1}(c) + (\alpha-1)^2 (2\alpha-1)^{-1}(2\alpha-3)^{-1}(1-c^2) \JJ_{\alpha-2}(c)
	\end{align*}
\end{lemma}

The derivative of $J_\alpha(\theta)$ turns out to be quite simple.
\begin{lemma}\label{lemma:JalphaGrad}
	Suppose $\alpha > 0$. Then
	\begin{align*}
	\dot J_\alpha(\theta) &= -\alpha^2(2 \alpha-1)^{-1} J_{\alpha-1}(\theta) \sin \theta\\
	\dot \JJ_\alpha(c) &= \alpha^2 (2 \alpha-1)^{-1} \JJ_{\alpha-1}(c)
	\end{align*}
\end{lemma}
\begin{proof}[]
	We will prove the first formula.
	The second follows from chain rule.
	By \cref{lemma:JalphaBessel},
	\begin{align*}
	J_\alpha(\theta) &= \f 1 {2\pi \cV_\alpha}\sin^{2\alpha+1} \theta \int \dd x \bessel(x) e^{x\cos \theta} x^\alpha\\
	\dot J_\alpha(\theta) &= \f 1 {2\pi \cV_\alpha}[(2\alpha+1)\sin^{2\alpha} \theta \cos\theta \int \dd x \bessel(x) e^{x\cos \theta} x^\alpha\\
						&\phantom{={}} -\sin^{2\alpha+2}\theta \int \dd x \bessel(x) e^{x\cos \theta} x^{\alpha+1}]\\
				&= (2\alpha+1) \cot \theta J_\alpha(\theta) - \f{c_{\alpha+1}}{c_\alpha}\csc\theta J_{\alpha+1}(\theta)\\
				&= (2\alpha+1) \csc \theta [\cos \theta J_\alpha(\theta) - J_{\alpha+1}(\theta)].\\
	\end{align*}
	As $\alpha + 1 > 1$, by \cref{lemma:JalphaRec}, this is
	\begin{align*}
	&\phantom{={}}-(2\alpha+1)\csc \theta [(\alpha-1)^2 (2\alpha+1)^{-1}(2\alpha-1)^{-1} \sin^2 \theta J_{\alpha-1}(\theta)]\\
	&= -(\alpha-1)^2(2\alpha-1)^{-1} \sin \theta J_{\alpha-1}(\theta).
	\end{align*}
\end{proof}

Thus $\dot \JJ_\alpha(1) = \alpha^2(2\alpha-1)^{-1} \JJ_{\alpha-1}(1) = \alpha^2(2\alpha-1)^{-1}$ for any $\alpha > 0$ by \cref{lemma:basicJalpha}.
For $1/2 < \alpha \le 1$, $\dot \JJ_\alpha(1) \ge 1$ with equality iff $\alpha = 1$, and for $\alpha = 1/2$, $\dot \JJ_\alpha(1) =\infty > 1$ by continuity of $\dot \JJ_\alpha(c)$ in $\alpha$.
Because for $\alpha > -1/2$, $\JJ_\alpha$ is increasing and convex on $[0, 1]$ and $\JJ_\alpha(0) > 0$ by \cref{lemma:basicJalpha}, $\JJ_\alpha$ intersects identity at a unique point away from 1 when $\alpha \in [1/2, 1)$.
We record this as a theorem.

\begin{thm}\label{thm:stableFixedPointsJJ}
	For $\alpha \in [1/2, 1)$, $\JJ_\alpha(c) = c$ has two solutions: an unstable solution at 1 ("unstable" meaning $\dot\JJ_\alpha(1) > 1$) and a stable solution in $\ee^* \in (0, 1)$ ("stable" meaning $\dot \JJ_\alpha(\ee^*) < 1$).
\end{thm}

\begin{figure}
	\centering
	\includegraphics[height=.1\textheight]{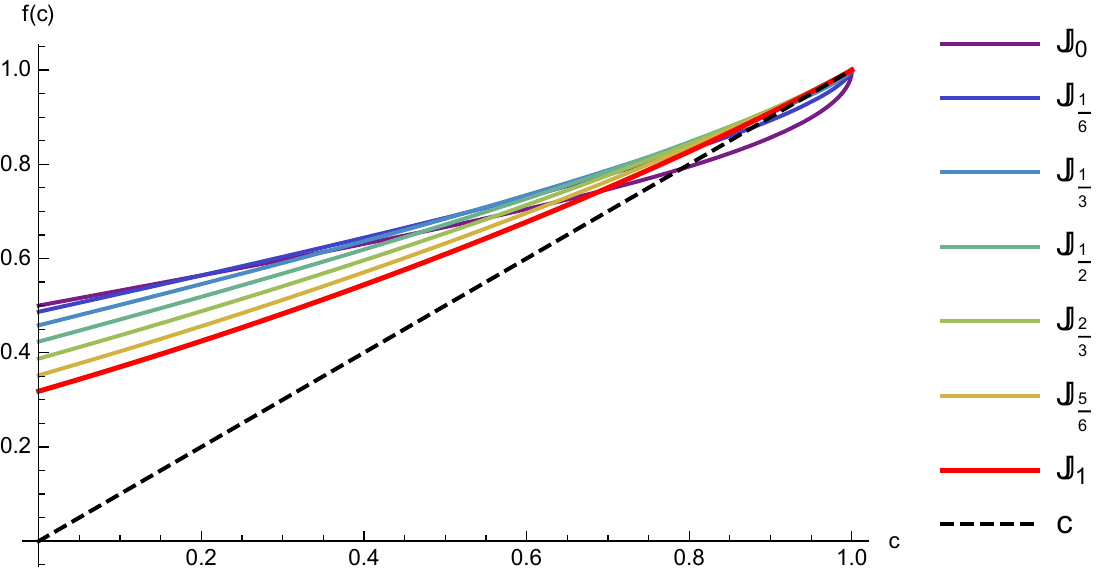}
	~
	\includegraphics[height=.1\textheight]{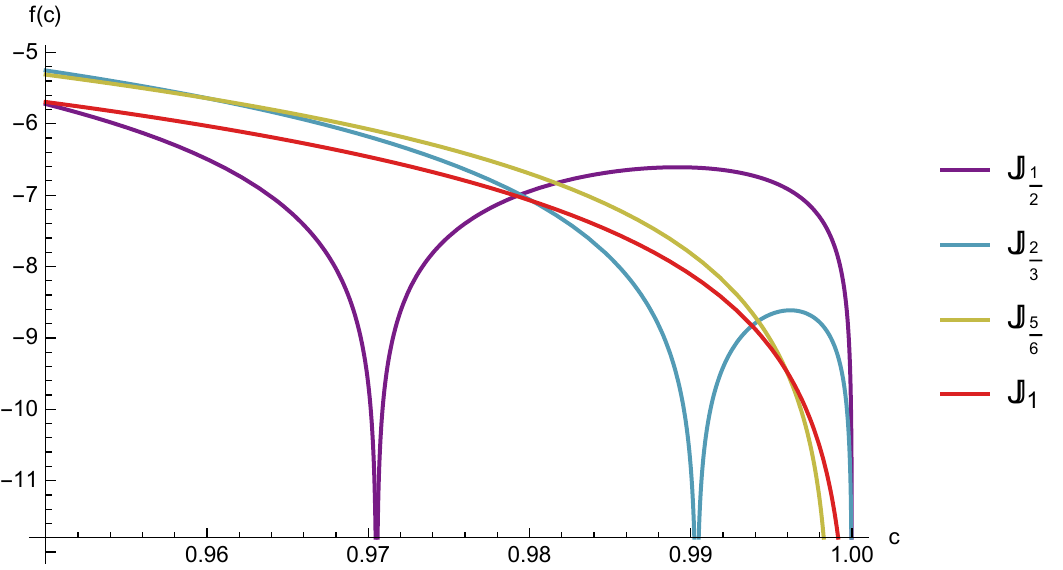}
	~
	\includegraphics[height=.1\textheight]{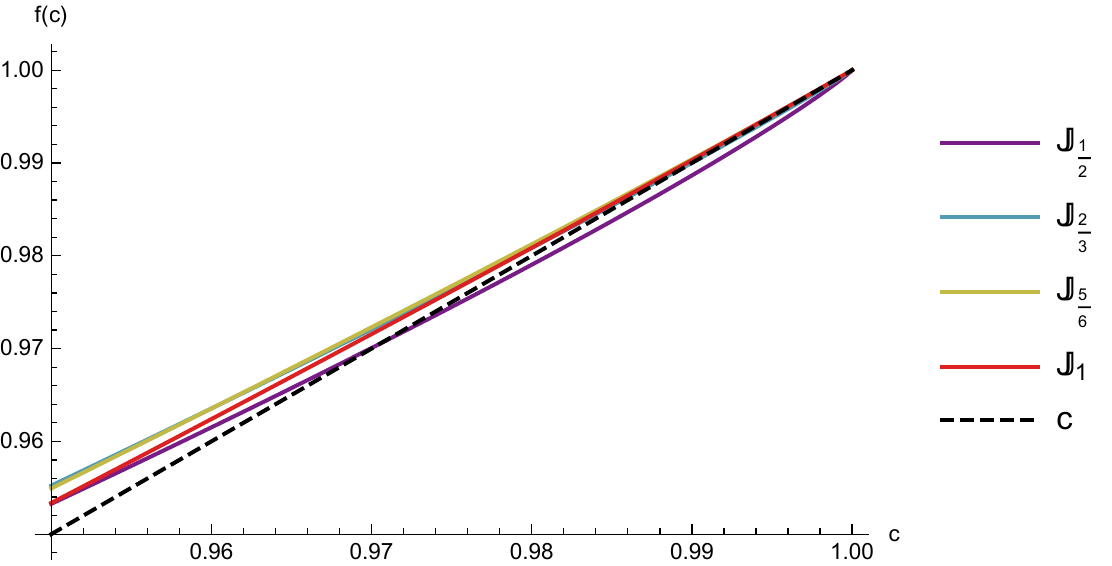}
	\caption{Left-to-right:
		\textbf{(a)} $\JJ_\alpha$ for different $\alpha$s and the identity function (black, dashed line).
		$\JJ_1$ is highlighted in red.
		From this plot, it looks like $\JJ_\alpha(c) \ge c$ and $\dot\JJ_\alpha(c) \le 1$ for all $\alpha \in (\f 1 2, 1]$ with equality iff $c = 1$, but this is misleading.
		\textbf{(b)} shows $|\JJ_\alpha(c) - c|$ in log scale.
		Where the curves dip below the x-axis indicate points where $\JJ_\alpha(c) = c$.
		We see that in fact every $\JJ_\alpha$ has a solution $\JJ_\alpha(c) = c$ for a $c < 1$, when $\alpha < 1$.
		\textbf{(c)} Furthermore, at each such $c$, $\dot\JJ_\alpha < 1$.
		%(d) A plot of $\JJ_1(c)$ vs identity $c$.
		(b) and (c) demonstrate the existence of stable fixed points away from 1 for $\JJ_\alpha, \alpha \in (1/2, 1)$, which is confirmed rigorously by \cref{thm:stableFixedPointsJJ}.
	}
	\label{fig:jjj_vs_id_main}
\end{figure}

This result confirms that pictures presented in \cref{fig:jjj_vs_id_main}b,c are qualitatively correct, that there are indeed stable fixed points of $\JJ_\alpha$ away from 1.
% TODO: characterize the asymptotics of the nonunit fixed point of \JJ_\alpha
\ReLUSquaredConvergence*
\begin{proof}
If $\prv \ee < 1$, then
\begin{align*} 
\cc = \f{\sigma_w^2 \prv \gamma + \sigma_b^2}{\sigma_w^2 \prv \pp + \sigma_b^2} &\ge \prv \ee\\
\JJ_1(\cc) &\ge \JJ_1(\prv\ee)\\
\ee = \f{\sigma_v^2 \cV_\alpha \qq^\alpha\JJ_1(\cc) + \sigma_b^2}{\sigma_v^2 \cV_\alpha \qq^\alpha + \sigma_b^2} &\ge \JJ_1(\prv \ee) 
\end{align*}
but $\ee \ge \JJ_1(\prv \ee) > \prv \ee$ as noted above.
Thus by monotone convergence $\ee$ converges, and $\ee^* = 1$ is the only possible fixed point.

% Write $B := 1 + \sigma_v^2 \sigma_w^2/2$ and $U := \f 5{6\pi \sqrt 2}$.
By \cref{lemma:cExpansion}, $\cc = \prv \ee(1 + \Theta(\prv\eps \prv \pp^{-1})) = 1 - \prv \eps + \Theta(\prv \eps \inv \pp) = 1 - u\prv\eps$ where $u := 1 - \Theta(\prv \pp^{-1})$.
Using the asymptotic expansion $\JJ_1(1 - \eps) = 1 - \eps+ U \eps^{3/2} + \Theta(\eps^{5/2})$, we have
\begin{align*}
(1 - \eps)\pp &= \sigma_v^2 \f \qq 2 \JJ_1(1 - u\prv \eps) + \sigma_a^2 + (1 - \prv \eps) \prv \pp\\
-\eps \pp &= \sigma_v^2 \f \qq 2 (\JJ_1(1 - u\prv \eps) - 1) - \prv \eps \prv \pp\\
	&= \sigma_v^2 \f \qq 2 [- u\prv\eps  + U u^{3/2}\prv\eps^{3/2} + \Theta(u^{5/2}\prv\eps^{5/2})] - \prv \eps \prv \pp\\
\eps &= \prv \eps \f 1 \pp [\prv \pp + \sigma_v^2 \f \qq 2 (u - U u^{3/2} \prv \eps^{1/2} + \Theta(u^{5/2}\prv \eps^{3/2}))]\\
	&=  \prv \eps \f 1 \pp [\pp - \sigma_a^2 + \sigma_v^2 \f \qq 2 (\Theta(\inv {\prv \pp}) - U u^{3/2} \prv \eps^{1/2} + \Theta(u^{5/2}\prv \eps^{3/2}))]\\
	&=  \prv \eps [1 + \f{-\sigma_a^2 + \sigma_v^2 \f \qq 2 (\Theta(\inv {\prv \pp}) - U u^{3/2} \prv \eps^{1/2} + \Theta(u^{5/2}\prv \eps^{3/2}))}{\pp}]\\
	&=  \prv \eps [1 + \f{-\sigma_a^2\inv \qq + \f 1 2 \sigma_v^2 (\Theta(\inv {\prv \pp}) - U u^{3/2} \prv \eps^{1/2} + \Theta(u^{5/2}\prv \eps^{3/2}))}{\pp\inv \qq}]
%\eps &= \prv \eps(\inv B + O(\inv \pp)) + \sigma_v^2 \sigma_w^2(\f 1 2 \inv B + O(\inv \pp))(\prv\eps - \Theta(\prv\eps \prv \pp^{-1}) - U \prv\eps^{3/2} + \Theta(\prv\eps^{5/2}))\\
%	&= \prv \eps(1 + O(\inv \pp)) -(\f 1 2 \sigma_v^2 \sigma_w^2 \inv B U - \Theta(\prv\eps) + O(\inv \pp))\prv\eps^{3/2}
\end{align*}
Let the content of the bracket on the RHS be $\aleph$.
We have $\pp \inv \qq = (1+ o(1))B/\sigma_w^2$.
If $\eps = O(\prv \pp^{-1})$, then $\aleph = 1 - O(\inv \pp)$, but because $\pp$ is exponentially decreasing, this means $\eps = \Theta(1)$ and does not converge to 0 --- this is a contradiction.
Therefore, $\prv \eps = \omega(\prv \pp^{-1})$, and
\begin{align*}
\eps &= \prv \eps [1 - \f 1 2 B^{-1}\sigma_v^2 \sigma_w^2 U\prv \eps^{1/2}(1 + o(1))]\\
\eps - \prv \eps &= - \f 1 2 B^{-1}\sigma_v^2 \sigma_w^2 U\prv \eps^{3/2}(1 + o(1))
\end{align*}
Using \cref{lemma:polyDynamicsPosAlpha} to upper and lower bound our dynamics, we get that $\p \eps l \sim [\f 1 4 \sigma_v^2 \sigma_w^2 \inv B U l]^{-2}$.
\end{proof}

\newcommand{\KK}{\mathbb{K}}
\begin{lemma}\label{lemma:separableDynamics}
Let $\phi$ be any nonlinearity.
Suppose $\Wt \phi( r, rd) = \Vt \phi( r)\KK(d)$ for some twice differentiable function $\KK(d)$ independent of $\qq$, where $\KK(1) = 1$ naturally.
Suppose further that 
\begin{itemize}
    \item $\KK(d) = d$ has a solution $d = \ee^* > 0$ where $\dot\KK(\ee^*) = \delta < 1$;
    \item $\KK(d) > d$ for all $d < \ee^*$ and $\KK(d) < d$ for all $1 > d > \ee^*$; and
    \item $\KK$ is nondecreasing. 
\end{itemize}
Let $\p \eps l := \p \ee l - \ee^*$ and suppose $\p \ee 0 < 1$.
If $\p \gamma l \to \infty$ and $\Vt \phi( \p \qq l) \to \infty$, then $\p \eps l \to 0$ and satisfies
$$ \eps = \prv \eps\left(1 - \f{\sigma_a^2 + (1 - \delta + O(\prv \eps))\sigma_v^2 \Vt \phi( \qq)}{\pp}\right) + \Vt\phi( \qq)\Theta(\inv \gamma\inv \pp).$$
\end{lemma}

\begin{proof}
	First we note that because $\ee^*$ is the only stable fixed point of the dynamics $x \mapsto \KK(x)$, with the basin of attraction $[0, 1)$, we can show $\p \ee l \to \ee^*$ as in the proof of \cref{thm:eDynamicsFullResTanh} (using \cref{lemma:timeDependentConvergence}).
	
Write $\p V l := \Vt \phi( \p \qq l)$.
We first show that $\p \ee l \to \ee^*$.
When $l$ is large,
% If $\prv \ee < \ee^*$, then
\begin{align*} 
\cc = \f{\sigma_w^2 \prv \gamma + \sigma_b^2}{\sigma_w^2 \prv \pp + \sigma_b^2} & = \prv \ee ( 1 + O(\inv \gamma))\\
\ee = \f{\sigma_v^2V\KK(\cc) + \sigma_a^2}{\sigma_v^2 V + \sigma_a^2} &= \KK(\cc)(1 + O(\inv V \KK(\cc)^{-1})).
\end{align*}
If $\p \gamma l$ is bounded for all $l$, then $\ee \to 0$ because $\p \pp l \to \infty$.
Since $\KK(\cc) > 0$ for $\cc \in [0, 1]$ and $\p V l\to \infty$, we have that in the limit $l\to \infty$, $\lim_{l \to \infty} \ee = 0 = \KK(\lim_{l \to \infty}\ee) = \KK(0)$ (by the continuity of $\KK$), which is impossible by our assumptions.
Thus $\p \gamma l \to \infty$, and we have $\lim_{l \to \infty} \ee = \KK(\lim_{l \to \infty} \ee)$.
By our assumptions, $\ee^*$ is the only stable fixed point of $\KK$ with basin of attraction $[0, 1)$, so this shows that $\ee \to \ee^*$ as desired.

Now we derive the equation in question.
Note that $\cc = \prv \ee(1 + \Theta(\inv \gamma))$ because $\ee^* < 1$.
We use the Taylor expansion $\KK(\ee^* + \eps) = \ee^* + \delta \eps + O(\eps^2)$.
\renewcommand{\sv}{\sigma_v^2}
\newcommand{\sa}{\sigma_a^2}
\begin{align*}
(\ee^* + \eps)\pp &= \sv V \KK\lp(\ee^* + \prv \eps)(1 + \Theta(\inv \gamma))\rp + \sa + (\ee^* + \prv\eps) \prv \pp\\
	&= \sv V(\ee^* + \delta(\prv \eps + \Theta(\inv \gamma)) + O(\prv \eps^2)) + \sa + (\ee^* + \prv\eps) \prv \pp\\
\eps \pp&= \sv V(\delta(\prv \eps + \Theta(\inv \gamma)) + O(\prv \eps^2)) + \prv \eps \prv \pp\\
\eps &= \prv \eps (1 - \f{\sa + (1 - \delta + O(\prv \eps))\sv V}\pp) + \Theta(V \inv\gamma \inv \pp)
\end{align*}

\end{proof}

\alphaReLUeConvergence*
\begin{proof}
We apply \cref{lemma:separableDynamics}.
We first check the conditions of the lemma, with $\KK = \JJ_\alpha$.
The following conditions were already verified.
\begin{itemize}
    \item $\JJ_\alpha$ has a fixed point $\ee^*$ less than but very close to 1, where its slope is $\upsilon := \dot\JJ_\alpha(\ee^*) < 1$. (\cref{thm:stableFixedPointsJJ})
    
    \item $\JJ_\alpha(d) > d$ for all $d < \ee^*$ and $\JJ_\alpha(d) < d$ for all $d > \ee^*$. (By the convexity shown in \cref{lemma:basicJalpha})
    \item $\JJ_\alpha$ is nondecreasing (\cref{lemma:basicJalpha}).
    Furthermore, from its integral formula (\cref{eqn:JalphaIntegralFormula}), we see easily that $\JJ_\alpha$ is smooth at $\ee^* < 1$.
\end{itemize}
We also proved the following
\begin{itemize}
    \item $\p \pp l \sim [\sigma_v^2 \sigma_w^{2\alpha} \cV_\alpha (1 - \alpha)]^{\f 1 {1-\alpha}}l^{\f 1 {1-\alpha}}$ (\cref{thm:pDynamicLT1ReLU}) and $\p \gamma l$ is asymptotically a constant fraction of $\p \pp l$ (\cref{lemma:separableDynamics}), so both go to $\infty$. 
    \item $\Vt \psi_\alpha( \qq) = \cV_\alpha \qq^\alpha = \cV_\alpha (\sigma_w^2 \pp + \sigma_b^2)^\alpha = \Theta(l^{\alpha/(1-\alpha)})$, so goes to $\infty$. (\cref{lemma:VtPsiAlpha})
\end{itemize}
Thus, for $\upsilon = \dot \JJ(\ee^*)$,
\begin{align*}
\f{\sigma_a^2 + (1 - \upsilon + O(\prv\eps))\sigma_v^2 \Vt \phi( \qq)}{\pp} &\sim \f{(1 - \upsilon)\sigma_v^2\sigma_w^{2\alpha}\cV_\alpha}{\pp^{1-\alpha}}\\
	&= \inv l (1-\upsilon)/(1 - \alpha).
\end{align*}

Now, $\Vt\phi( \qq)\inv \gamma \inv \pp = \Theta(l^{-\f1{1-\alpha} - 1})$.
By using the dynamics of \cref{lemma:alphaDeltaDynamics} to upper and lower bound our dynamics, we have $\p \eps l = \Omega(l^{-\mu - \eps}), O(l^{-\mu + \eps})$ for any $\eps > 0$, where $\mu = \min((1-\upsilon)/(1-\alpha), 1/(1-\alpha)) = (1-\upsilon)/(1-\alpha).$

\end{proof}

\subsubsection{Backward Dynamics}

\begin{lemma}\label{lemma:infVarAlphaLe75}
	Suppose random variable $X \sim \Gaus(0, \sigma^2)$, and $Y = \psi_{-\beta}(X)$ for some $\beta > 0 $, where $\psi_\alpha$ is $\alpha$-ReLU.
	Then for $\xi > 0$, $Y$ has density
	$$\Pr[Y \in [\xi, \xi+\dd \xi]] = \f 1 {\beta\sqrt{2\pi \sigma^2}} \xi^{-\f 1 \beta - 1} e^{-\xi^{-2/\beta}/2\sigma^2}.$$
	At $\xi = 0$, $Y$ has density given by a Dirac delta of mass $\f 1 2$.
	
	Furthermore, $Y$ has finite second moment iff $\beta < \f 1 2$.
\end{lemma}
\begin{proof}
	We have
	\begin{align*}
	\Pr[Y \in [\xi, \infty)] &= \Pr[X \in [0, \xi^{-1/\beta}]]\\
	&= \f 1 {\sqrt{2\pi \sigma^2}}\int_0^{\xi^{-1/\beta}} e^{-x^2/2\sigma^2}\dd x.
	\end{align*}
	Differentiating the RHS against $\xi$ using Leibniz's rule, we get
	\begin{align*}
	d\Pr[Y \in [\xi, \infty)]/d\xi &= \f 1 {\sqrt{2\pi \sigma^2}}  e^{-\xi^{-2/\beta}/2\sigma^2} \f d {d\xi}\xi^{-1/\beta}\\
	&= \f {-1} {\beta\sqrt{2\pi \sigma^2}} \xi^{-\f 1 \beta - 1} e^{-\xi^{-2/\beta}/2\sigma^2}.
	\end{align*}
	Negating both sides gives the density $f_Y$ of $Y$ for $\xi > 0$.
	For $\xi = 0$, observe that $\lim_{\xi \to 0} f_Y(\xi) = 0$ because, while $\xi^{-\f 1 \beta -1}$ blows up polynomially, $e^{-\xi^{-2/\beta}/2\sigma^2}$ blows up exponentially.
	Thus the contribution of $Y$'s mass at $Y = 0$ from $X > 0$ is 0.
	On the other hand, all $X < 0$ gets mapped to $Y = 0$, so $f_Y(0) = \f 1 2 \delta_0$, where $\delta_0$ is the Dirac delta.
	
	For the second assertion, observe that
	\begin{align*}
	f_Y(\xi) \sim \f 1 {\beta \sqrt{2\pi \sigma^2}} \xi^{-\f 1 \beta - 1} & \text{as ${\xi \to \infty}$}.
	\end{align*}
	Thus, $\xi^2 f_Y(\xi)$ is integrable iff $2 -\f 1 \beta - 1 < -1 \iff \beta < \f 1 2$.
\end{proof}
\dalethInfVarAlphaReLU*
\begin{proof}
	Note that $\dot \psi_\alpha \propto \psi_{\alpha - 1}$, so it suffices to show that $\Var(\psi_{\alpha-1}(\zeta)^2) = \Var(\psi_{2\alpha - 2}(\zeta))$ is infinite for $\zeta \sim \Gaus(0, \sigma^2)$.
	By \cref{lemma:infVarAlphaLe75} with $\beta = 2 - 2\alpha$, $\psi_{2\alpha - 2}(\zeta)$ has finite variance iff $\beta < \f 1 2 \iff \alpha > \f 3 4$.
\end{proof}

\dalethDynamicsAlphaReLU*
\begin{proof}
If $\alpha = 1$, then
$$\prv\daleth = \daleth (1 + \f 1 2\sigma_v^2 \sigma_w^2).$$
So $\p \daleth {l-m} / \p \daleth l = \Theta(1) B^m$ for $B = 1 + \f 1 2\sigma_v^2 \sigma_w^2$.

If $\f 1 2 < \alpha < 1$, then $\prv \daleth/\daleth - 1$ is
\begin{align*}
 &\phantom{=} \sigma_v^2 \sigma_w^2 \Vt \dot \phi( \qq)\\
	&= \sigma_v^2 \sigma_w^2 \alpha^2 \cV_{\alpha - 1} \qq^{\alpha - 1}\\
	&= \sigma_v^2 \sigma_w^2 \alpha^2 \cV_{\alpha - 1} (\sigma_w^2 \pp)^{\alpha - 1} + \Theta(\pp^{\alpha - 2})\\
	&= \sigma_v^2 \sigma_w^{2\alpha} \alpha^2 \cV_{\alpha - 1}(K_1 l^\oalpha - K_2 l^\aalpha \log l + o(l^\aalpha \log l))^{\alpha - 1} + \Theta(l^{\f{\alpha - 2}{1 - \alpha}}) & \text{by \cref{thm:pDynamicLT1ReLU}}\\
	&= \sigma_v^2 \sigma_w^{2\alpha} \alpha^2 \cV_{\alpha - 1} [K_1^{\alpha - 1} \inv l + \Theta(l^{-2} \log l)] + O(l^{-3})\\
	&= \sigma_v^2 \sigma_w^{2\alpha} \alpha^2 \cV_{\alpha - 1} K_1^{\alpha - 1} \inv l + \Theta(l^{-2} \log l)\\
	&= R \inv l + \Theta(l^{-2} \log l)
\end{align*}
where $R = \sigma_v^2 \sigma_w^{2\alpha} \alpha^2 \cV_{\alpha - 1} K_1^{\alpha - 1} = \f{\alpha^2}{(1-\alpha)(2 \alpha - 1)} $ and $K_1 = [\sigma_v^2 \sigma_w^{2\alpha} \cV_\alpha (1 - \alpha)]^{\f 1 {1-\alpha}}$.
So
\begin{align*}
\prv \daleth &= \daleth \exp(R \inv l + \Theta(l^{-2} \log l))\\
\p \daleth {l -m} &= \Theta(1)\p \daleth {l} \lp\f l {l-m} \rp^R
\end{align*}
as desired.
\end{proof}

\alphaReLUAllGradients*
\begin{proof}
	The proof is similar to that of \cref{thm:dalethExpSqrtTanhAllGrad}.
\end{proof}
\end{document}